\documentclass{article}

\usepackage[main,final,nonatbib]{neurips_2026}

\usepackage[utf8]{inputenc} 
\usepackage[T1]{fontenc}    
\usepackage{booktabs}       
\usepackage{amsfonts}       
\usepackage{nicefrac}       
\usepackage{microtype}      
\usepackage{xcolor}         
\usepackage{amsmath,amssymb,amsthm,graphicx,url}
\usepackage{dsfont}       

\usepackage{bm}

\def\1{\bm{1}}

\def\rmk{{\mathbf{k}}}

\def\rmB{{\mathbf{B}}}

\def\rmD{{\mathbf{D}}}

\def\rmI{{\mathbf{I}}}

\def\rmK{{\mathbf{K}}}

\def\rmW{{\mathbf{W}}}

\def\rmPsi{{\bm{\Psi}}}

\DeclareMathAlphabet{\mathsfit}{\encodingdefault}{\sfdefault}{m}{sl}
\SetMathAlphabet{\mathsfit}{bold}{\encodingdefault}{\sfdefault}{bx}{n}

\def\gB{{\mathcal{B}}}
\def\gC{{\mathcal{C}}}
\def\gD{{\mathcal{D}}}

\def\gF{{\mathcal{F}}}
\def\gG{{\mathcal{G}}}
\def\gH{{\mathcal{H}}}

\def\gK{{\mathcal{K}}}
\def\gL{{\mathcal{L}}}
\def\gM{{\mathcal{M}}}
\def\gN{{\mathcal{N}}}

\def\gQ{{\mathcal{Q}}}
\def\gR{{\mathcal{R}}}
\def\gS{{\mathcal{S}}}
\def\gT{{\mathcal{T}}}
\def\gU{{\mathcal{U}}}
\def\gV{{\mathcal{V}}}
\def\gW{{\mathcal{W}}}
\def\gX{{\mathcal{X}}}
\def\gY{{\mathcal{Y}}}

\def\fp{{\mathfrak{p}}}

\newcommand{\E}{\mathbb{E}}

\newcommand{\R}{\mathbb{R}}

\newcommand{\lan}{\langle}
\newcommand{\ran}{\rangle}
\newcommand{\abs}[1]{\left\lvert #1 \right\rvert}
\newcommand{\inprod}[2]{\langle #1, #2 \rangle}
\newcommand{\norm}[1]{\lVert#1\rVert}

\DeclareMathOperator*{\argmax}{arg\,max}
\DeclareMathOperator*{\argmin}{arg\,min}

\DeclareMathOperator{\Tr}{Tr}

\let\originalleft\left
\let\originalright\right
\renewcommand{\left}{\mathopen{}\mathclose\bgroup\originalleft}
\renewcommand{\right}{\aftergroup\egroup\originalright}

\usepackage{mathtools}
\usepackage{algorithm}
\usepackage{algorithmic}
\usepackage{wrapfig}
\usepackage{subcaption}
\usepackage[shortlabels]{enumitem}
\usepackage{hyperref}[links]
\hypersetup{
  colorlinks,
  linkcolor={red!72!black},
  citecolor={blue!72!black},
  urlcolor={magenta!72!black}
}
\usepackage[most]{tcolorbox}

\newtheorem{innercustomgeneric}{\customgenericname}
\providecommand{\customgenericname}{}
\newcommand{\newcustomtheorem}[2]{%
  \newenvironment{#1}[1]
  {%
   \renewcommand\customgenericname{#2}%
   \renewcommand\theinnercustomgeneric{##1}%
   \innercustomgeneric
  }
  {\endinnercustomgeneric}
}

\newtheorem{theorem}{Theorem}
\newtheorem{lemma}{Lemma}
\newtheorem{proposition}{Proposition}
\newtheorem{corollary}{Corollary}
\newtheorem{definition}{Definition}
\newtheorem{assumption}{Assumption}

\newtheorem{remark}{Remark}

\newcustomtheorem{customassumption}{Assumption}

\makeatletter
\long\def\@makecaption#1#2{
        \vskip 0.8ex
        \setbox\@tempboxa\hbox{\small {\bf #1:} #2}
        \parindent 1.5em  
        \dimen0=\hsize
        \advance\dimen0 by -3em
        \ifdim \wd\@tempboxa >\dimen0
                \hbox to \hsize{
                        \parindent 0em
                        \hfil 
                        \parbox{\dimen0}{\def\baselinestretch{0.96}\small
                                {\bf #1.} #2
                                } 
                        \hfil}
        \else \hbox to \hsize{\hfil \box\@tempboxa \hfil}
        \fi
        }
\makeatother

\title{Conditional Counterfactual Mean Embeddings: \\ Doubly Robust Estimation and Learning Rates}

\author{%
  Thatchanon Anancharoenkij \\
  Department of Statistics\\
  Chiang Mai University\\
  \texttt{thatchanon\_ananch@cmu.ac.th} \\
  \And
  Donlapark Ponnoprat \\
  Department of Statistics\\
  Chiang Mai University\\
  \texttt{donlapark.p@cmu.ac.th} \\
}

\begin{document}

\maketitle

\begin{abstract}%
  A complete understanding of heterogeneous treatment effects requires fully capturing the conditional distribution of the counterfactual outcomes. To this end, we study the \emph{Conditional Counterfactual Mean Embeddings} (CCME), a framework that embeds conditional distributions of counterfactual outcomes into a reproducing kernel Hilbert space (RKHS). Under this framework, we develop a two-stage meta-estimator for CCME that accommodates any RKHS-valued regression in each stage. Based on this meta-estimator, we develop three practical CCME estimators: (1) \emph{Ridge Regression} estimator, (2) \emph{Deep Feature} estimator that parameterizes the feature map by a neural network, and (3) \emph{Neural-Kernel} estimator that performs RKHS-valued regression, with the coefficients parameterized by a neural network. We provide finite-sample convergence rates for all estimators, establishing that they possess the double robustness property. Our experiments demonstrate that our estimators accurately recover distributional features including multimodal structure of conditional counterfactual distributions. Our code is available at \href{https://github.com/donlap/Conditional-Counterfactual-Mean-Embeddings}{this link}.
\end{abstract}

\section{Introduction}

Causal inference is fundamental to scientific inquiry, policy design, and development of reliable artificial intelligence, where the primary goal is to evaluate the impact of interventions on outcomes of interest. Consider a scenario where a decision maker wants to understand the causal effect of a binary treatment $A \in \{0,1\}$ on an outcome $Y$, with particular interest in how this effect varies across subpopulations characterized by covariates $V$. The potential outcomes framework \cite{Rubin1974, Holland1986} provides a principled approach to this setting: it posits that each unit has two \emph{counterfactuals} $Y^0$ and $Y^1$. This framework highlights two fundamental challenges: (i) for each unit, the outcome under the alternative treatment remains unobserved, and (ii) a broader set of covariates $X$ may confound the relationship between treatment and outcome. Given these challenges, the literature has traditionally focused on estimating scalar summaries, such as the conditional average treatment effect $\mathbb{E}[Y^1 - Y^0 \mid V = v]$.

While these scalar summaries provide valuable high-level insights, they often obscure critical nuances in the underlying distribution: different counterfactual distributions can share identical means yet exhibit vastly different shapes \cite{Kennedy2023}. A multimodal counterfactual distribution often indicates the presence of distinct underlying subgroups with different responses to the treatment. Capturing these behaviors calls for the estimation of the \emph{conditional counterfactual distributions}, that is, the distributions of $Y^0$ and $Y^1$ conditional on $V=v$.

To study the conditional counterfactual distributions, we leverage the framework of \emph{Kernel Mean Embeddings} (KME) \cite{Muandet2017}, which allows us to manipulate probability distributions as vectors in a \emph{reproducing kernel Hilbert space} (RKHS). This framework enables a wide range of statistical tasks, such as density estimation \cite{Song2008, Sriperumbudur2011,Kanagawa2014}, hypothesis testing \cite{Gretton2012,Chwialkowski2016} and sampling \cite{Welling2009,Chen2010}. For this reason, KME has seen considerable success as a tool for studying counterfactual distributions \cite{Muandet2021,Park2021,Singh2023}.

These successes motivate us to study the \emph{Conditional Counterfactual Mean Embeddings} (CCME), the embeddings of conditional counterfactual distributions in the RKHS, allowing us to extend the analysis of counterfactual distributions when the distributions vary across subpopulations. Inheriting the representation properties of the KME, the CCME enables a wide range of statistical tasks on conditional counterfactual distributions, such as density estimation, conditional two-sample testing, and sampling.

To estimate the CCME, we utilize \emph{doubly robust} methods \cite{Robins1994, Scharfstein1999}. The classical doubly robust methods debias the outcome variable by incorporating two models that describe the relationships between (i) the outcome and the covariates, and (ii) the treatment assignment and the covariates. These methods yield estimators that are consistent even if only one of the two models is correctly specified. The key insight of our approach is to extend these methods to the RKHS setting by instead \emph{debiasing the feature map of the outcome}. As a result, we obtain CCME estimators in the functional space that inherit the double robustness property.

\textbf{Contribution.} We study the \emph{Conditional Counterfactual Mean Embeddings} (CCME) as a way to represent conditional counterfactual distributions in an RKHS and develop doubly robust estimators through the double machine learning framework. Our specific contributions are three-fold:
\begin{enumerate}
    \item \textbf{A Meta-Estimator for CCME:} We propose a two-stage meta-estimator for CCME that accommodates arbitrary propensity score models and conditional mean embedding (CME) estimators in the first stage, and any CME estimator in the second stage.

    \item \textbf{Three Specific CCME Estimators:} We develop three practical instances of the meta-estimator, namely the \emph{Ridge Regression} estimator, the \emph{Deep Feature} estimator, and the \emph{Neural-Kernel} estimator, each of which offers distinct advantages in terms of simplicity, flexibility and scalability.

    \item \textbf{Finite-Sample Convergence Rates:} We establish finite-sample error bounds for all three estimators that explicitly track how convergence rates depend on the dimensions of the outcome, the covariates, and the conditioning variables. Crucially, we prove that our estimators achieve \emph{double robustness in rates}, i.e., the first-stage error is controlled by the minimum of the rates achieved by the propensity score estimator and the CME estimator.
\end{enumerate}

\section{Background and Notations}

We briefly describe the framework to study conditional counterfactual distributions. As a complement, we provide full definitions of positive definite kernels and Bochner integrals in Appendix~\ref{app:sobolev}.

\subsection{Problem Setup}
We observe an i.i.d.\ sample $Z_1,\ldots,Z_{2n}$ where $Z_i \coloneqq (X_i, A_i, Y_i) \sim P$. Here, $X \in \mathcal{X} \subset \mathbb{R}^{d_x}$ is a vector of covariates, $A \in \{0,1\}$ is a binary treatment assignment, and $Y \in \mathcal{Y} \subset \mathbb{R}^{d_y}$ is the observed outcome. We adopt the potential outcomes framework \cite{Rubin1974, Holland1986}, under which $Y = A Y^1 + (1-A) Y^0$, where $Y^1$ and $Y^0$ are referred to as the \emph{potential outcomes} or \emph{counterfactuals} under treatment ($A=1$) and control ($A=0$), respectively. Our focus is on the conditional distribution of the counterfactual $Y^1$ given covariates $V \coloneqq \eta(X) \in \mathcal{V} \subset \mathbb{R}^{d_v}$, where $\eta : \gX \to \gV$ is a known function (e.g., a projection onto a subset of covariates).

\subsection{Reproducing Kernel Hilbert Spaces}
Let $\gH_{\gY}$ denote an RKHS on the outcome space $\gY$ with inner product $\inprod{\cdot}{\cdot}_{\gH_{\gY}}$ and reproducing kernel $k_\gY: \gY \times \gY \to \mathbb{R}$. The kernel is \emph{reproducing} in the sense that $\inprod{g}{k_\gY(\cdot, y)}_{\gH_{\gY}} = g(y)$ for any $g \in \gH_{\gY}$ and $y \in \gY$. The \emph{canonical feature map} $\phi: \gY \to \gH_{\gY}$ is defined as $\phi(y) \coloneqq k_\gY(\cdot, y)$, which satisfies $k_\gY(y, y') = \inprod{\phi(y)}{\phi(y')}_{\gH_{\gY}}$.

The \emph{mean embedding} of a random variable $Y \in \gY$ is defined as $\mu_Y \coloneqq \E_{Y}[\phi(Y)] \in \gH_{\gY}$, where the expectation is the Bochner integral. When $k_\gY$ is a \emph{characteristic kernel}, the mapping from the space of probability distributions to the RKHS is injective \cite{Fukumizu2007,Sriperumbudur2010}. This property ensures that $\mu_Y$ encodes all information about the distribution of $Y$.

The \emph{conditional mean embedding} (CME) $\mu_{Y|X}: \gX \to \gH_\gY$ is defined as $\mu_{Y|X}(x) \coloneqq \E[\phi(Y) \mid X=x]$, which satisfies $\E[g(Y) \mid X=x] = \inprod{g}{\mu_{Y|X}(x)}_{\gH_{\gY}}$ for any $g \in \gH_{\gY}$.

\subsection{The Target Estimand}\label{sec:id}

Our goal is to estimate the conditional distribution of the counterfactual outcome $Y^1$ given a set of covariates $V = \eta(X)$. In the RKHS framework, we represent the conditional distribution by the \emph{Conditional Counterfactual Mean Embeddings} (CCME):
\begin{align*}\label{eq:ccmedef}
    \mu_{Y^1\vert V}(v) \coloneqq \E[\phi(Y^{1}) \mid V=v] \in \gH_{\gY}.
\end{align*}
This embedding fully characterizes the conditional distribution of $Y^1$ given $V=v$ when the kernel $k_\gY$ is \emph{characteristic} \cite{Fukumizu2007,Sriperumbudur2010} (e.g., Gaussian, Mat\'ern, and Laplace kernels). As the conditional distribution of $Y^1$ can differ from that of $Y$, the embedding $\mu_{Y^1\vert V}(v)$ might not be the same as $\mu_{Y\vert V}(v) \coloneqq \E[\phi(Y) \mid V=v]$.

\section{Doubly Robust Identification of CCME}

To identify the CCME from observed data $(X, A, Y)$, we adopt a standard technique for treatment effect estimation from observational data \cite{Rosenbaum1983,Rosenbaum1987} and extend it to RKHS settings.

First, we introduce two \emph{nuisance functions} that model the conditional distributions of $A$ and $Y$ given covariates $X$: the \emph{propensity score} $\pi(x) \coloneqq P(A=1|X=x)$ and the CME of the treated outcome conditional on $X$, $\mu_{0}(x) \coloneqq \E[\phi(Y)|X=x, A=1]$. Using these nuisance functions, we construct an RKHS-valued pseudo-outcome $\xi(Z) \in \gH_{\gY}$ as:
\begin{equation}\label{eq:defxi}
    \xi(Z) \coloneqq \frac{A}{\pi(X)}\left( \phi(Y) - \mu_{0}(X) \right)
    + \mu_{0}(X).
\end{equation}
For identification of CCME, we require the following two standard assumptions:
\begin{assumption}\label{ass:unconfound}
(i) \textit{Conditional ignorability:} $Y^a \perp A \mid X$ for each $a\in\{0,1\}$. (ii) \textit{Strong positivity:} There exists $\varepsilon > 0$ such that $\varepsilon \le \pi(x) \le 1-\varepsilon$ for all $x \in \gX$.
\end{assumption}
\begin{assumption}\label{ass:boundk}
The kernel $k_\gY$ is continuous and $\sup_{y \in \gY}k_\gY(y,y) \le B_k < \infty$ for some $B_k > 0$.
\end{assumption}
The first assumption is common in theoretical analysis of inverse propensity weighting and doubly robust estimators \cite{Rosenbaum1983,Kennedy2023}. Specifically, it ensures that treatment assignment is independent of potential outcomes given covariates $X$, i.e., there are no unmeasured confounders. Strong positivity requires that the propensity score is uniformly bounded away from 0 and 1, ensuring that both treatment groups are represented across all covariate values.

The boundedness of kernels is a mild assumption that is satisfied by commonly used kernels, e.g., Gaussian, Matérn, and Laplace kernels, but not by polynomial and linear kernels.

Under these assumptions, the CCME is identified via $\xi(Z)$. The proof is provided in Appendix~\ref{sec:iden_proof}.
\begin{proposition}\label{prop:identification}
Under Assumptions~\ref{ass:unconfound} and \ref{ass:boundk}, the CCME satisfies the following identification result:
\begin{equation}\label{eq:identification}
\mu_{Y^1\vert V}(v) = \E[\xi(Z) \mid V=v].
\end{equation}
\end{proposition}

\section{The Meta-Estimator}\label{sec:meta}

The identification result from the previous section suggests a two-stage estimation strategy that has proven successful for estimating conditional treatment effects \cite{Chernozhukov2018,Semenova2020,Kennedy2023b}: first, estimate the nuisance functions $\pi$ and $\mu_{0}$; second, construct pseudo-outcomes $\widehat{\xi}(Z)$ and regress them on $V$ to obtain the CCME estimator. Formally, we randomly split the full dataset $\gD$ of size $2n$ into two disjoint folds, $\mathcal{D}_{0}$ and $\mathcal{D}_{1}$, each of size $n$. The two-stage estimation then proceeds as follows:

\textbf{Stage 1: Nuisance Estimation.} Using $\mathcal{D}_{0}$, we estimate two \emph{nuisance functions}: (1) a propensity score estimator $\widehat{\pi}$ of $\pi$, and (2) a CME estimator $\widehat\mu_{0}$ of $\mu_0$. The propensity score can be estimated using any classification model that provides probability predictions. For the CME, several estimators have been proposed \cite{Song2009,Xu21,Shimizu24}, which we explore in the next section.

\textbf{Stage 2: Pseudo-Outcome Regression.} Using $\mathcal{D}_{1}$, we construct pseudo-outcomes from the estimated nuisance functions. For each $z = (x, a, y) \in \gX \times \{0,1\} \times \gY$, the pseudo-outcome $\widehat{\xi}(z) \in \gH_\gY$ is constructed as:
\begin{equation}\label{eq:pseudo_outcome}
    \widehat{\xi}(z) \coloneqq \frac{a}{\widehat{\pi}(x)}\left(\phi(y) - \widehat\mu_{0}(x)\right) + \widehat\mu_{0}(x).
\end{equation}%
Given the observations $Z_{1i} \coloneqq (X_{1i}, A_{1i}, Y_{1i}) \in \gD_1$ with the covariates $V_{1i} \coloneqq \eta(X_{1i})$, we estimate the CCME by regressing the pseudo-outcomes on the covariates. Specifically, we solve a regularized empirical risk minimization problem over a hypothesis class $\gF$ of $\gH_\gY$-valued functions:
\begin{equation}\label{eq:generic_loss}
    \widehat{\mu}_{Y^1|V} \coloneqq \argmin_{\mu \in \gF} \frac{1}{n} \sum_{i=1}^n
    \norm{\mu(V_{1i}) - \widehat{\xi}(Z_{1i})}_{\gH_\gY}^2 + \Lambda(\mu),
\end{equation}
where $\Lambda: \gF \to \mathbb{R}_+$ is a regularization functional that controls model complexity. The choice of $\gF$ and $\Lambda$ determines the specific estimator. In the next section, we present three specific choices of $\gF$ and $\Lambda$ that yield practical estimators for the CCME.

\section{Three Practical Estimators}\label{sec:three_estimators}
In this section, we propose three choices of hypothesis class $\mathcal{F}$ for both first- and second-stage CME regression, yielding three computationally tractable estimators for the CCME with different computational-statistical trade-offs.

Throughout, we work with the two splits of data: $\mathcal{D}_0 \coloneqq \{Z_{0i}\}_{i=1}^n$ and $\mathcal{D}_1 \coloneqq \{Z_{1i}\}_{i=1}^n$, where $Z_{0i}\coloneqq (X_{0i}, A_{0i}, Y_{0i})$ and $Z_{1i}=(X_{1i}, A_{1i}, Y_{1i})$.

Let $\mathbf{I}_n$ be the $n \times n$ identity matrix and $\mathrm{Id}$ the identity operator on $\gH_\gY$. Define an $\gH_\gY$-valued vector $\boldsymbol{\widehat\Xi} \coloneqq (\widehat{\xi}(Z_{11}), \ldots, \widehat{\xi}(Z_{1n}))^\top$, which maps any $\mathbf{u} \in \R^n$ to $\gH_\gY$ as follows: $\boldsymbol{\widehat\Xi}\mathbf{u} = \sum_{i=1}^n u_i \widehat\xi(Z_{1i})$.

\subsection{Ridge Regression Estimator}

We use the kernel ridge regression (KRR) \cite{Song2009} with $V_{1i}$ as the input and $\widehat\xi(Z_{1i})$ as the target. To formalize this, we let $k_{\mathcal{V}}: \mathcal{V} \times \mathcal{V} \to \mathbb{R}$ be a kernel on $\mathcal{V}$ with associated RKHS $\mathcal{H}_{\mathcal{V}}$. Denote by $\gB\gL(\mathcal{H}_\mathcal{Y};\mathcal{H}_\mathcal{Y})$ the set of bounded linear operators from $\gH_\gY$ to itself. We define the operator-valued kernel $\Gamma_\gV: \mathcal{V} \times \mathcal{V} \to \gB\gL(\mathcal{H}_\mathcal{Y}; \mathcal{H}_\mathcal{Y})$ as $\Gamma_\gV(v, v') \coloneqq k_{\mathcal{V}}(v, v')\,\mathrm{Id}$, and let $\mathcal{H}_{\Gamma_\gV}$ be the associated vector-valued RKHS. Setting $\mathcal{F} = \mathcal{H}_{\Gamma_\gV}$ in the least squares problem \eqref{eq:generic_loss} with regularizer $\Lambda = \lambda_1 \|\cdot\|^2_{\mathcal{H}_{\Gamma_\gV}}$ for some $\lambda_1>0$ yields the following problem:
\begin{equation*}\label{eq:emp_ccme_loss}
    \widehat{\mu}_{\mathrm{RR}} = \argmin_{\mu \in \mathcal{H}_{\Gamma_\gV}} \frac{1}{n} \sum_{i=1}^n
    \|\mu(V_{1i}) - \widehat{\xi}(Z_{1i})\|_{\mathcal{H}_\mathcal{Y}}^2
    + \lambda_1 \|\mu\|^2_{\mathcal{H}_{\Gamma_\gV}}.
\end{equation*}
The Ridge Regression estimator $\widehat{\mu}_{\mathrm{RR}}$ can be written explicitly as \cite{Song2009,Grunewalder2012,Park2020}:
\begin{equation}\label{eq:emp_ccme_form}
    \widehat{\mu}_{\mathrm{RR}}(v) = \boldsymbol{\widehat\Xi}
    (\rmK_{V} + n \lambda_1 \mathbf{I}_n)^{-1}\rmk_{V}(v),
\end{equation}
where $\rmK_V \coloneqq (k_{\mathcal{V}}(V_{1i}, V_{1j}))_{i,j=1}^n \in \R^{n\times n}$ and $\rmk_{V}(v)\coloneqq (k_{\mathcal{V}}(v, V_{11}),\ldots,k_{\mathcal{V}}(v,V_{1n}))^{\top}$.

\subsection{Deep Feature Estimator}

The empirical CCME estimator in \eqref{eq:emp_ccme_form} requires $O(n^3)$ time for matrix inversion, making it computationally expensive for large sample sizes. As an alternative, we adapt from~\cite{Xu21} the \emph{Deep Feature} estimator, which takes the form $\mu_{\mathrm{DF}}(v) = C \psi_\theta(v)$, where $\psi_\theta : \gV \to \R^M$ is a neural network, parameterized by $\theta$, that learns a feature map of $M \ll n$ dimensions, and $C: \R^M \to \gH_\gY$ is a linear operator. This reduces the matrix inversion to $O(M^3)$ time.

To formalize the Deep Feature estimator, we denote by $\Theta$ a parameter space for neural networks from $\gV$ to $\R^M$ and $\gB\gL(\R^{M};\gH_\gY)$ the set of bounded linear operators from $\R^{M}$ to $\gH_\gY$. We then solve the following variant of the least squares problem \eqref{eq:generic_loss}:
\begin{equation}\label{eq:first_stage_DF_loss}
    \begin{split}
    \argmin_{\substack{\theta \in \Theta, \\ C \in \gB\gL(\R^{M};\gH_\gY)}}
    \biggl(\frac{1}{n}\sum_{i=1}^{n} \lVert C\psi_\theta(V_{1i}) &- \widehat\xi(Z_{1i}) \rVert_{\gH_\gY}^{2} + \lambda_1\lVert C \rVert_{\text{HS}}^{2}\biggr),
    \end{split}
\end{equation}
where $\|\cdot\|_{\text{HS}}$ is the Hilbert-Schmidt norm and $\lambda_1 > 0$ is the regularization parameter. For each fixed $\psi_\theta$, solving \eqref{eq:first_stage_DF_loss} with respect to $C$ yields the following exact solution:
\begin{equation}\label{eq:C_hat}
    \widehat{C}_{\psi_\theta} = \boldsymbol{\widehat\Xi} \rmPsi_\theta
    (\rmPsi_\theta^{\top}\rmPsi_\theta + n\lambda_1\mathbf{I}_{M})^{-1},
\end{equation}
where $\rmPsi_\theta \coloneqq(\psi_\theta(V_{11}),\ldots,\psi_\theta(V_{1n}))^\top \in \R^{n \times M}$. Substituting $C = \widehat{C}_{\psi_\theta}$ into \eqref{eq:first_stage_DF_loss} and dropping the term that does not depend on $\theta$ yield the following loss for $\psi_\theta$:
\begin{equation}\label{eq:first_stage_DF_psi}
     \widehat\gL_{\mathrm{DF}}(\theta) = -\Tr\left( \rmK_{\widehat\xi}
    \rmPsi_\theta(\rmPsi_\theta^{\top}\rmPsi_\theta + n\lambda_1\mathbf{I}_{M})^{-1}\rmPsi_\theta^\top \right),
\end{equation}
where $\rmK_{\widehat{\xi}} \coloneqq (\langle \widehat{\xi}(Z_{1i}), \widehat{\xi}(Z_{1j}) \rangle_{\gH_\gY})_{i,j=1}^n \in \mathbb{R}^{n \times n}$.

Optimizing the loss with gradient descent, we obtain a learned feature map $\psi_{\widehat\theta}$. Our estimator is obtained by applying the optimal linear operator $\widehat{C}_{\psi_{\widehat\theta}}$ defined in \eqref{eq:C_hat}:
\begin{equation*}\label{eq:DF_final}
    \widehat{\mu}_{\mathrm{DF}}(v) \coloneqq \widehat{C}_{\psi_{\widehat\theta}} \psi_{\widehat\theta}(v)
    = \boldsymbol{\widehat{\Xi}} \rmPsi_{\widehat\theta}
    (\rmPsi^{\top}_{\widehat\theta}\rmPsi_{\widehat\theta} + n\lambda_1\mathbf{I}_{M})^{-1} \psi_{\widehat\theta}(v).
\end{equation*}
\subsection{Neural-Kernel Estimator}

The Neural-Kernel estimator \cite{Shimizu24} offers a third approach for CME estimation that completely eliminates the matrix inversion and the regularization parameter $\lambda_1$.

To formalize this estimator, we let $\widetilde\gY_M \coloneqq \{\widetilde{y}_{j}\}_{j=1}^{M}$ be a set of $M$ grid points in $\gY$, which we specify before estimation.

Let $f_\theta:\gV \to \R^M$ be a neural network parameterized by $\theta \in \Theta$ with $M$ outputs, and $f_{\theta}(v)_j$ be the $j$-th component of $f_\theta(v)$. The Neural-Kernel estimator takes the form $\mu_{\mathrm{NK}}(v) = \sum_{j=1}^M f_\theta(v)_j\phi(\widetilde{y}_j)$. We then solve the following variant of the least squares problem \eqref{eq:generic_loss}:
\begin{equation*}
     \argmin_{\theta \in \Theta} \frac{1}{n}\sum_{i=1}^n
    \left\lVert \sum_{j=1}^M f_{\theta}(V_{1i})_j\phi(\widetilde{y}_j) - \widehat\xi(Z_{1i})\right\rVert^2_{\gH_{\gY}}.
\end{equation*}
Expanding the squared norm using the reproducing property yields the following loss:
\begin{align}\label{eq:NK_loss}
    \widehat\gL_{\mathrm{NK}}(\theta) = \frac{1}{n} \sum_{i=1}^n \left[ f_\theta(V_{1i})^\top \rmK_{M} f_\theta(V_{1i}) - 2 f_\theta(V_{1i})^\top \mathbf{b}_i \right],
\end{align}
where $\rmK_{M} \coloneqq (k_\gY(\widetilde{y}_j, \widetilde{y}_l))_{j,l=1}^M$ and $\mathbf{b}_i \coloneqq (\langle \phi(\widetilde{y}_1), \widehat{\xi}(Z_{1i}) \rangle_{\gH_\gY},\ldots, \langle \phi(\widetilde{y}_M), \widehat{\xi}(Z_{1i}) \rangle_{\gH_\gY})^\top$.

Now suppose that the first-stage CME estimator $\widehat\mu_{0}$ is also a Neural-Kernel estimator \emph{with the same grid points}. In other words, $\widehat\mu_{0}(x) = \sum_{j=1}^M g_{\widehat\theta_0}(x)_j\phi(\widetilde{y}_j)$ where $g_{\widehat\theta_0}$ is the learned network from the first stage. Plugging the pseudo-outcome formula from \eqref{eq:pseudo_outcome} into the definition of $\mathbf{b}_i$, we obtain:
\begin{align*}
    \mathbf{b}_i &=
    \frac{A_{1i}}{\widehat{\pi}(X_{1i})} \rmk_i + \left(1 - \frac{A_{1i}}{\widehat{\pi}(X_{1i})}\right) \rmK_{M} g_{\widehat\theta_0}(X_{1i}),
\end{align*}
where $\rmk_i \coloneqq (k_\gY(\widetilde{y}_1, Y_{1i}), \ldots, k_\gY(\widetilde{y}_M, Y_{1i}))^\top$. Substituting this expression into \eqref{eq:NK_loss} yields:
\begin{align*}
    \widehat\gL_{\mathrm{NK}}(\theta) &= \frac{1}{n} \sum_{i=1}^n \Bigg[ f_\theta(V_{1i})^\top \rmK_{M} f_\theta(V_{1i}) - 2\frac{A_{1i}}{\widehat{\pi}(X_{1i})} f_\theta(V_{1i})^\top \rmk_i - 2\left(1 - \frac{A_{1i}}{\widehat{\pi}(X_{1i})}\right) f_\theta(V_{1i})^\top \rmK_{M} g_{\widehat\theta_0}(X_{1i}) \Bigg].
\end{align*}
Hence, we can compute the loss in \eqref{eq:NK_loss} solely in terms of the Gram matrix of the grid points and the neural network. Let $f_{\widehat\theta}$ be a neural network obtained by optimizing the loss with gradient descent. The final estimator is given by:
\begin{equation*}\label{eq:NK_final}
    \widehat{\mu}_{\mathrm{NK}}(v) \coloneqq \sum_{j=1}^{M} f_{\widehat\theta}(v)_j \phi(\widetilde{y}_{j}).
\end{equation*}
\textbf{Choosing an estimator.} The choice of estimator is mainly determined by sample size. The Ridge Regression estimator becomes computationally intractable for large $n$ due to $O(n^3)$ matrix inversion costs, making the Deep Feature and Neural-Kernel estimators preferable in such regimes. Nonetheless, the Ridge Regression estimator has the advantages of having only two tuning parameters and generally smaller variance compared to the neural estimators.

Between the two neural estimators, the Deep Feature estimator typically exhibits lower variance due to the regularization, while the Neural-Kernel estimator can achieve comparable variance reduction through weight regularization or early stopping. For both methods, we recommend using cross-validation to select the strength of the regularization. We discuss related work in Appendix~\ref{sec:related}.

\section{Theoretical Analysis}

We establish the convergence rates of our meta- and specific estimators. In preparation, we define the risks of the propensity score and the CME estimators as:
\begin{align*}
    \gR^{2}_{\pi}(\widehat\pi) \coloneqq \E\left[(\widehat{\pi}(X) - \pi(X))^{2}\right], \qquad \gR^{2}_{\mu_{0}}(\widehat\mu_{0}) \coloneqq \E\left[\norm{\widehat\mu_{0}(X) - \mu_{0}(X)}_{\gH_\gY}^{2}\right].
\end{align*}
For the second stage, we define the risk of the $\gH_\gY$-valued regression of the pseudo-outcome:
\begin{equation*}
    \gR^{2}_{\widehat{\xi}}(\widehat\mu_{Y^1\vert V}) \coloneqq
    \E\left[ \norm{\widehat\mu_{Y^1|V}(V) - \E[\widehat\xi(Z)|V, \gD_0]}^{2}_{\gH_\gY}\right].
\end{equation*}
For all our bounds, we require the following assumption on the propensity score estimator analogous to Assumption~\ref{ass:unconfound}, which can be achieved by truncating the propensity score estimator to $[\varepsilon, 1-\varepsilon]$:
\begin{assumption}\label{ass:first_stage}
The propensity score estimator satisfies $\varepsilon \le \widehat\pi(x) \le 1-\varepsilon$ for all $x \in \gX$.
\end{assumption}

In the following result, we use $A \lesssim B$ and $A \asymp B$ when $A \le cB$ and $c_1B \le A \le c_2B$, respectively, for some constants $c,c_1,c_2>0$ independent of $n$.

\subsection{Meta-Estimator}
We start with the convergence rate for the meta-estimator:
\begin{theorem}[Meta-Estimator's Rate]\label{thm:ECME}
    Suppose that Assumptions~\ref{ass:unconfound}, \ref{ass:boundk}, and~\ref{ass:first_stage} hold. In addition, assume that $\sup_{x\in\gX}\norm{\widehat\mu_{0}(x)}_{\gH_{\gY}} \le B_{\mu_0}$ for some constant $B_{\mu_0} > 0$. Then the meta-estimator $\widehat\mu_{Y^1\vert V}$ from \eqref{eq:generic_loss} satisfies:
    \begin{align}\label{eq:meta_bound}
        &\E\left[ \norm{\widehat\mu_{Y^1\vert V}(V) - \mu_{Y^1\vert V}(V)}^{2}_{\gH_\gY}\right] \lesssim \gR^{2}_{\widehat\xi}(\widehat\mu_{Y^1\vert V})
        + \min\left\{\gR^{2}_{\pi}(\widehat\pi) , \gR^{2}_{\mu_{0}}(\widehat\mu_{0})\right\}.
    \end{align}
\end{theorem}
Here, the assumption $\sup_{x\in\gX}\norm{\widehat\mu_{0}(x)}_{\gH_{\gY}} \le B_{\mu_0}$ is mild, as we can control the norm of the CME estimator via regularization or truncation of the estimator's parameters.

The most important aspect of \eqref{eq:meta_bound} is that the convergence rate of the meta-estimator is controlled by the rate of the second-stage estimator, as well as the \emph{faster rate} of the two nuisance estimators, establishing the \emph{double robustness} property of the meta-estimator.

\begin{remark}\label{rem:product}
The last term on the right-hand side of \eqref{eq:meta_bound} can also be replaced by a product of the two nuisance errors:
\begin{equation}\label{eq:l4bound}
\E\left[(\widehat\pi(X) - \pi(X))^{4}\right]^{1/2}\E\left[\norm{\widehat\mu_{0}(X) - \mu_{0}(X)}^{4}_{\gH_\gY}\right]^{1/2},
\end{equation}
as we show in Appendix~\ref{sec:misc_proofs}. We state \eqref{eq:meta_bound} with the minimum because rates for nuisance estimators are usually given in $L^2$-norms, while passing from $L^4$- to $L^2$-norms requires the moment comparison condition of~\cite[Assumption~8]{Foster2023}. The two bounds are not comparable in general. Under the moment comparison condition, the product is of smaller order than the minimum when both nuisance estimators converge, and of the same order when one of them does not. Without this condition, the $L^4$-norms can be much larger than the $L^2$-norms. In that case, the minimum can be the smaller bound.
\end{remark}

\subsection{Three Practical Estimators}

In the following results, $W^{k,p}(\Omega)$ for any Euclidean set $\Omega$ denotes the Sobolev space of $L^p(\Omega)$ functions of order $k$ (cf. Definition~\ref{def:sobolev} in the Appendix), $\widehat{f}$ is the Fourier transform of a function $f$, $a \lesssim_{\log n} b$ if $a \le c_1(\log n)^{c_2} b$ for some constants $c_1,c_2>0$, and $\norm{\cdot}$ is the standard Euclidean norm.

Here are the assumptions required for our theoretical analysis:

\begin{assumption} \label{ass:domain}
   $Y^1$ and $V$ are continuous random variables. The marginal distribution of $V$ is strictly positive on $\gV$ and has a density $p_V$ with $\sup_{v\in\gV}p_V(v) < \infty$. The domains $\gY \subset \R^{d_y}$, $\gX \subset \R^{d_x}$ and $\gV \subset \R^{d_v}$ are non-empty, bounded sets with Lipschitz boundaries and satisfy a uniform interior cone condition (see Definition \ref{def:cone_condition}).
\end{assumption}
The assumption on $Y^1$ and $V$ ensures the existence of the conditional density $\fp^1(y\vert v)$. The assumption on the domains is often required for analysis of Sobolev spaces (see, e.g., \cite[Chapter 5]{Evans2010}). It is satisfied by any ``well-behaved'' domain that does not have any sharp cusps, such as balls or cubes.

\begin{assumption} \label{ass:kernel}
   The kernel $k_\gY$ is translation invariant. In other words, $k_\gY(y, y') = \varphi(y-y')$ for some function $\varphi:\R^{d_y} \to \R$. Its Fourier transform $\widehat{\varphi}(\omega)$ satisfies $\widehat{\varphi}(\omega) \lesssim (1 + \norm{\omega}^{2})^{-\tau}$, for some $\tau > d_y/2$. This implies that $\gH_\gY$ is continuously embedded in the Sobolev space $W^{\tau, 2}(\gY)$.
\end{assumption}
The translation invariance of the kernel allows us to analyze the smoothness of the kernel and the functions in the RKHS via their Fourier transforms (cf. Proposition \ref{prop:kernelinv}). The Gaussian kernel, for example, satisfies this assumption for all $\tau > 0$. The Mat\'ern kernel on $\R^{d_y}$ with smoothness parameter $\alpha$ satisfies the assumption with $\tau = \alpha + d_y/2$.

The next assumption concerns the smoothness of the conditional density function of $Y^1$ given $V=v$, denoted by $\fp^1(y\vert v)$:

\begin{assumption} \label{ass:smooth_density}
   For every outcome $y \in \gY$, the function $v \mapsto \fp^1(y|v)$ belongs to a Sobolev space $W^{r,2}(\gV)$ for some $r \in \mathbb{N}_0$ satisfying $r > d_v/2$. Moreover, $ \int_\gY \norm{\fp^1(y|\cdot)}_{W^{r,2}(\gV)}^{2} \, \mathrm{d}y < \infty$.
\end{assumption}
Assumption~\ref{ass:smooth_density} quantifies the smoothness of the conditional density $\fp^1$ in $v$ via the Sobolev index $r$. As will be shown in our main results, $r$ appears explicitly in our convergence rate bounds, indicating that densities that are smoother in $v$ lead to faster rates.

\subsubsection{Ridge Regression Estimator}

As the second stage of the Ridge Regression estimator is the KRR estimator, we can use one of the existing results, e.g., \cite{Fischer2020,Li2022,Li2024}, to establish the rate. In particular, as the estimator \eqref{eq:emp_ccme_form} is calculated based on the kernel function $k_\gV$ on $\gV$, we impose an assumption on this function instead of on $k_\gY$:

\begin{customassumption}{5$\mathbf{'}$} \label{ass:kernel1}
   The kernel of the conditioning variable $k_\gV(v, v') = \varphi(v-v')$ is translation invariant. Its Fourier transform $\widehat{\varphi}(\omega)$ satisfies $ \widehat{\varphi}(\omega) \asymp (1 + \norm{\omega}^{2})^{-\tau}$ for some $\tau > d_v/2$. This implies that $\gH_\gV$ is norm-equivalent to the Sobolev space $W^{\tau, 2}(\gV)$.
\end{customassumption}

We now state our bound for the convergence rate of the Ridge Regression estimator:
\begin{theorem}\label{thm:plugin_rate}
Let the assumptions in Theorem~\ref{thm:ECME} and Assumptions \ref{ass:domain}, \ref{ass:kernel1}, and \ref{ass:smooth_density} hold with $\tau \ge r$. In addition, assume that $\inf_{v\in\gV}p_V(v) > 0$.

For the second-stage estimator $\widehat{\mu}_{\mathrm{RR}}$, by choosing $\lambda_1 \asymp n^{-\frac{\tau}{r + d_v/2}}$, the estimator achieves the bound:
\begin{align*}
   &\E\left[ \|\widehat{\mu}_{\mathrm{RR}}(V) - \mu_{Y^1\vert V}(V)\|_{\gH_\gY}^2 \right] \lesssim n^{-\frac{2r}{2r + d_v}}
   + \min\left\{\gR^{2}_{\pi}(\widehat\pi), \gR^{2}_{\mu_{0}}(\widehat\mu_{0})\right\}.
\end{align*}
\end{theorem}

\subsubsection{Deep Feature Estimator}

For both the Deep Feature and Neural-Kernel estimators, we still assume the translation invariance of $k_\gY$, with an additional assumption on the Fourier transform $\widehat\varphi$:
\begin{assumption} \label{ass:kernel_nk}
   The kernel $k_\gY$ is translation invariant, $k_\gY(y,y')=\varphi(y-y')$, and, for any constant $C_y>0$ such that $\gY \subset [-C_y, C_y]^{d_y}$, $\inf_{\norm{\omega} \le 26d_y/C_y} \widehat{\varphi}(\omega) > 0$.
\end{assumption}
As the RKHS-norm of the translation invariant kernel can be written as the $L^2$-norm weighted by the inverse of $\widehat{\varphi}$ (cf. Proposition \ref{prop:kernelinv}), Assumption~\ref{ass:kernel_nk} ensures that $k_\gY$ is strictly positive definite, so that the Gram matrix of distinct points is invertible. This condition is satisfied by commonly used kernels, e.g., Gaussian, Mat\'ern, and Laplace kernels, but not by polynomial or linear kernels.

Both estimators use fully connected ReLU networks with $M$ outputs,
\begin{equation}\label{eq:nn_class}
\psi_\theta(v) \coloneqq \rmW_{L}\,\sigma_{\mathrm{relu}}(\rmW_{L-1} \cdots \sigma_{\mathrm{relu}}(\rmW_{1}v + \mathfrak{b}_1) \cdots + \mathfrak{b}_{L-1}) + \mathfrak{b}_L, \qquad \rmW_i \in \R^{d_i \times d_{i-1}},\ \mathfrak{b}_i \in \R^{d_i},
\end{equation}
where $\theta \coloneqq (\rmW_i, \mathfrak{b}_i)_{i=1}^L \in \prod_{i=1}^L (\R^{d_i \times d_{i-1}} \times \R^{d_i})$ with $d_0 = d_v$ and $d_L = M$, and $\sigma_{\mathrm{relu}}(x) \coloneqq \max\{x, 0\}$ applied coordinatewise. Here, $L$ is referred to as the \emph{depth} and $d_1, \ldots, d_{L-1}$ the \emph{widths} of the network. We write $W \coloneqq \sum_{i=1}^L (d_i d_{i-1} + d_i)$ for the number of weights and biases.

With these notations, we present the convergence rates of the Deep Feature estimator, distinguishing between two cases based on the kernel type.
\begin{theorem}\label{thm:DF_upper_bdd}
Let the assumptions in Theorem~\ref{thm:ECME} and Assumptions \ref{ass:domain}, \ref{ass:smooth_density}, and \ref{ass:kernel_nk} hold. Let $\psi_\theta$ be a neural network in~\eqref{eq:nn_class} with a total number of parameters $W$ and depth $L$. In addition, assume that $\min_{1 \le i \le L-1} d_i \asymp \max_{1 \le i \le L-1} d_i$. Then, there exist constants $\bar B > 0$ and $L_0$, which only depend on $r$ and $d_v$, such that, for any $B_\gF^2 \ge \bar B$ and any $L \ge \max\{L_0, 3\}$, the Deep Feature estimator $\widehat\mu_{\mathrm{DF}} = \widehat{C}_{\psi_{\widehat\theta}}\psi_{\widehat\theta}$ with the minimizing pair $(\widehat\theta, \widehat{C}_{\psi_{\widehat\theta}})$ of~\eqref{eq:first_stage_DF_loss} over the set $\Theta \coloneqq \{ \theta : \sup_{v\in\gV}\norm{\psi_\theta(v)} \le B_\gF\}$ satisfies the following statements with $b \coloneqq 2r/d_{v}$:
\begin{enumerate}
\item If Assumption \ref{ass:kernel} holds, define $a \coloneqq 2\tau/d_{y}$. By choosing $M \asymp n^{\frac{b}{4ab+4b+a}}$, $WL \asymp n^{\frac{a+2b}{4ab+4b+a}}$ and $\lambda_1 \asymp n^{-\frac{b(a+1)}{4ab+4b+a}}$, the estimator achieves
\begin{align*}
   &\E[\norm{\widehat\mu_{\mathrm{DF}}(V) - \mu_{Y^1\vert V}(V)}^{2}_{\gH_{\gY}}] \lesssim_{\log n} n^{-\frac{ab}{4ab+4b+a}} + \min\left\{\gR^{2}_{\pi}(\widehat\pi), \gR^{2}_{\mu_{0}}(\widehat\mu_{0})\right\}.
\end{align*}
\item If $k_\gY$ is a Gaussian kernel, by choosing $M \asymp (\log n)^{2d_y}$, $WL \asymp n^{\frac{1}{4b+1}}$ and $\lambda_1 \asymp n^{-\frac{b}{4b+1}}$, we have
\begin{align*}
   &\E[\norm{\widehat\mu_{\mathrm{DF}}(V) - \mu_{Y^1\vert V}(V)}^{2}_{\gH_{\gY}}] \lesssim_{\log n} n^{-\frac{b}{4b+1}} + \min\left\{\gR^{2}_{\pi}(\widehat\pi), \gR^{2}_{\mu_{0}}(\widehat\mu_{0})\right\}.
\end{align*}
\end{enumerate}
\end{theorem}

\subsubsection{Neural-Kernel Estimator}

For the Neural-Kernel estimator, we rely on an additional assumption:
\begin{assumption} \label{ass:grid}
   The points $\{\widetilde{y}_{j}\}_{j=1}^{M}\subset \gY$ form a quasi-uniform grid with fill distance $h \asymp M^{-1/d_y}$, where the fill distance is defined as $h \coloneqq \sup_{y \in \gY} \min_{1 \le j \le M} \norm{y - \widetilde{y}_{j}}$.
\end{assumption}
Assumption~\ref{ass:grid} requires the grid points to be roughly uniformly spaced in $\gY$. Here is our result:
\begin{theorem}\label{thm:NK_upper_bdd}
Let the assumptions in Theorem~\ref{thm:ECME} and Assumptions \ref{ass:domain}, \ref{ass:smooth_density}, \ref{ass:kernel_nk}, and \ref{ass:grid} hold. Let $f_\theta$ be a neural network of the form~\eqref{eq:nn_class} with a total number of parameters $W$ and depth $L$. In addition, assume that $\min_{1 \le i \le L-1} d_i \asymp \max_{1 \le i \le L-1} d_i$. Then, there exist constants $\bar B > 0$ and $L_0$, which only depend on $r$ and $d_v$, such that, for any $B_\gF^2 \ge \bar B$ and any $L \ge \max\{L_0, 3\}$, the Neural-Kernel estimator $\widehat\mu_{\mathrm{NK}}$ with the minimizer $\widehat\theta$ of~\eqref{eq:NK_loss} over the set $\Theta \coloneqq \{\theta : \sup_{v\in\gV}\norm{\sum_{j=1}^M f_\theta(v)_j\phi(\widetilde{y}_j)}_{\gH_\gY} \le B_\gF\}$ satisfies the following statements with $a$ and $b$ as in Theorem~\ref{thm:DF_upper_bdd}.
\begin{enumerate}
\item If Assumption \ref{ass:kernel} holds, by choosing $M \asymp n^{\frac{b}{ab+a+2b}}$ and $WL \asymp n^{\frac{a+2b}{ab+a+2b}}$, we have
\begin{align*}
   &\E[\norm{\widehat\mu_{\mathrm{NK}}(V) - \mu_{Y^1\vert V}(V)}^{2}_{\gH_{\gY}}] \lesssim_{\log n} n^{-\frac{ab}{ab+a+2b}} + \min\left\{\gR^{2}_{\pi}(\widehat\pi), \gR^{2}_{\mu_{0}}(\widehat\mu_{0})\right\}.
\end{align*}
\item If $k_\gY$ is a Gaussian kernel, by choosing $WL \asymp n^{\frac{1}{b+1}}$ and $M \asymp (\log n)^{2d_y}$, we have
\begin{align*}
   &\E[\norm{\widehat\mu_{\mathrm{NK}}(V) - \mu_{Y^1\vert V}(V)}^{2}_{\gH_{\gY}}] \lesssim_{\log n} n^{-\frac{2r}{2r+d_v}} + \min\left\{\gR^{2}_{\pi}(\widehat\pi), \gR^{2}_{\mu_{0}}(\widehat\mu_{0})\right\}.
\end{align*}
\end{enumerate}
\end{theorem}

\textbf{Discussion of the rates.} For all estimators, the bounds improve with the smoothness $r$ of $v \mapsto \fp^1(y\vert v)$ but deteriorate with the data dimensions. For the Deep Feature and Neural-Kernel estimators, the bounds also improve with the kernel smoothness $\tau$. The bound of Theorem~\ref{thm:DF_upper_bdd} decreases more slowly than the bound of Theorem~\ref{thm:NK_upper_bdd}. This difference arises because the Neural-Kernel estimator fits coefficients in a subspace of $\gH_\gY$ that is fixed before the data are seen, whereas the Deep Feature estimator also learns the map $C$ from the data, which contributes to the variance of the estimator.

The term $\min\left\{\gR^2_\pi(\widehat\pi), \gR^2_{\mu_{0}}(\widehat\mu_{0})\right\}$ indicates that the first-stage error is governed by the nuisance with the smaller error. If $d_x > d_v$ and the first-stage CME estimator is of the same type as the second-stage estimator, the first-stage error will converge more slowly and become the bottleneck.
\begin{remark}
The Deep Feature estimator and Neural-Kernel estimator were originally introduced in \cite{Xu21} and \cite{Shimizu24} as CME estimators without finite-sample convergence rates. For the Deep Feature estimator,~\cite[Appendix~B]{Xu21} provides a bound in terms of the Rademacher complexity of the deep feature class, for a regression with values in a finite-dimensional feature space. In contrast, we evaluate the complexity of the network class for the $\gH_\gY$-valued regression, which, combined with the universal approximation bound of~\cite{Yang2025}, yields a convergence rate in terms of the sample size. Consequently, the above theorems with $\widehat{\pi} = \pi$ and $\widehat\mu_{0} = \mu_{0}$ establish the first known finite-sample convergence rates for these methods.
\end{remark}

\section{Experiments}

To benchmark $\widehat{\mu}_{\bullet}(v)$ where $\bullet \in \{\mathrm{RR},\mathrm{DF},\mathrm{NK}\}$, we choose the task of estimation of conditional density $\fp^1(y|v)$ using $\widehat{\fp}^1_{\bullet}(y|v) \coloneqq \langle \widehat{\mu}_{\bullet}(v), \phi(y) \rangle_{\gH_\gY}$ with the Gaussian kernel $k_\gY(y, y') = (\sqrt{2\pi}\sigma)^{-d_y}\exp(-\norm{y - y'}^2/(2\sigma^2))$. We note that $\widehat{\fp}^1_{\bullet}(y|v)$ is a biased estimate of $\fp^1(y|v)$, where the bias is $O(\sigma^2)$ (cf. Corollary~\ref{cor:density} in the Appendix). The corresponding density estimators are provided in Appendix~\ref{sec:ccd_estimation}. For each estimator, we compare four pseudo-outcome constructions: (i) the proposed \emph{Doubly Robust} (DR) pseudo-outcome $\widehat{\xi}(Z)$ from \eqref{eq:pseudo_outcome}; (ii) \emph{Inverse Propensity Weighting} (IPW), $\widehat{\xi}_{\text{IPW}}(Z) \coloneqq \frac{A}{\widehat{\pi}(X)}\phi(Y)$; (iii) the \emph{Plug-in} (PI), $\widehat{\xi}_{\text{PI}}(Z) \coloneqq \widehat\mu_{0}(X)$; and (iv) the \emph{One-Step} estimator that directly uses the corresponding CME estimator with $V$ as an input to estimate the CCME~\cite{Park2021}; this is a conditional version of the plug-in estimator of~\cite{Muandet2021}. Full experiment details are provided in Appendix~\ref{sec:exp_details}. Further applications of our DR estimator are demonstrated in Appendix~\ref{sec:downstream}.

\subsection{Analysis of Double Robustness}\label{sec:exp1}

\paragraph{Setup.} We generate covariates $X \sim \mathcal{N}(\mathbf{1}_{10}, \mathbf{I}_{10}) \in \mathbb{R}^{10}$ with the conditioning variable $V = (X_1,X_2,X_3,X_4,X_5)$. The propensity score is $\pi(X) = 0.1 + 0.8 \cdot \mathbb{I}(X_1 \in [0, 2] \text{ and } X_6 \geq 1.5)$. Our estimator for $\pi$ is a random forest. We generate the outcomes according to $Y = 1 + X^\top \beta + A \cdot (2 + X^\top \gamma) + \epsilon$ where $\beta$ and $\gamma$ are fixed coefficients and $\epsilon$ is a heteroscedastic bimodal noise.

We consider three misspecification scenarios: (a) both nuisance models correctly specified, (b) $\pi$ misspecified by replacing the random forest with a logistic regression, and (c) $\mu_{0}$ misspecified by excluding $X_6$ from the outcome model. We vary $n \in \{100, 250, 500, 1000, 2500, 5000, 10000\}$ across 10 seeds and measure the MSE against the analytically computed ground truth $\fp^1(y|v)$.

\begin{figure}[!t]
    \centering

    \begin{subfigure}[b]{0.52\textwidth}
        \centering
        \includegraphics[width=\textwidth]{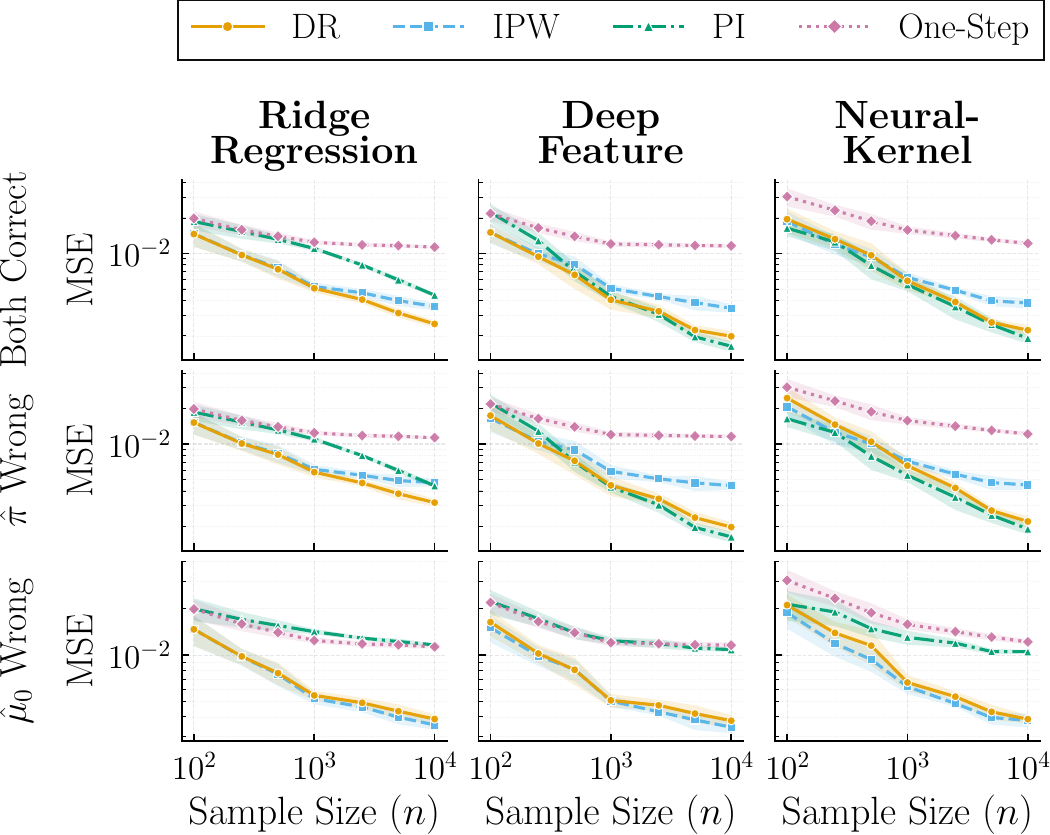}
        \caption{}
        \label{fig:mses}
    \end{subfigure}
    \hfill
    \begin{minipage}[b]{0.46\textwidth}

        \begin{subfigure}[b]{\textwidth}
            \centering
            \includegraphics[width=\textwidth]{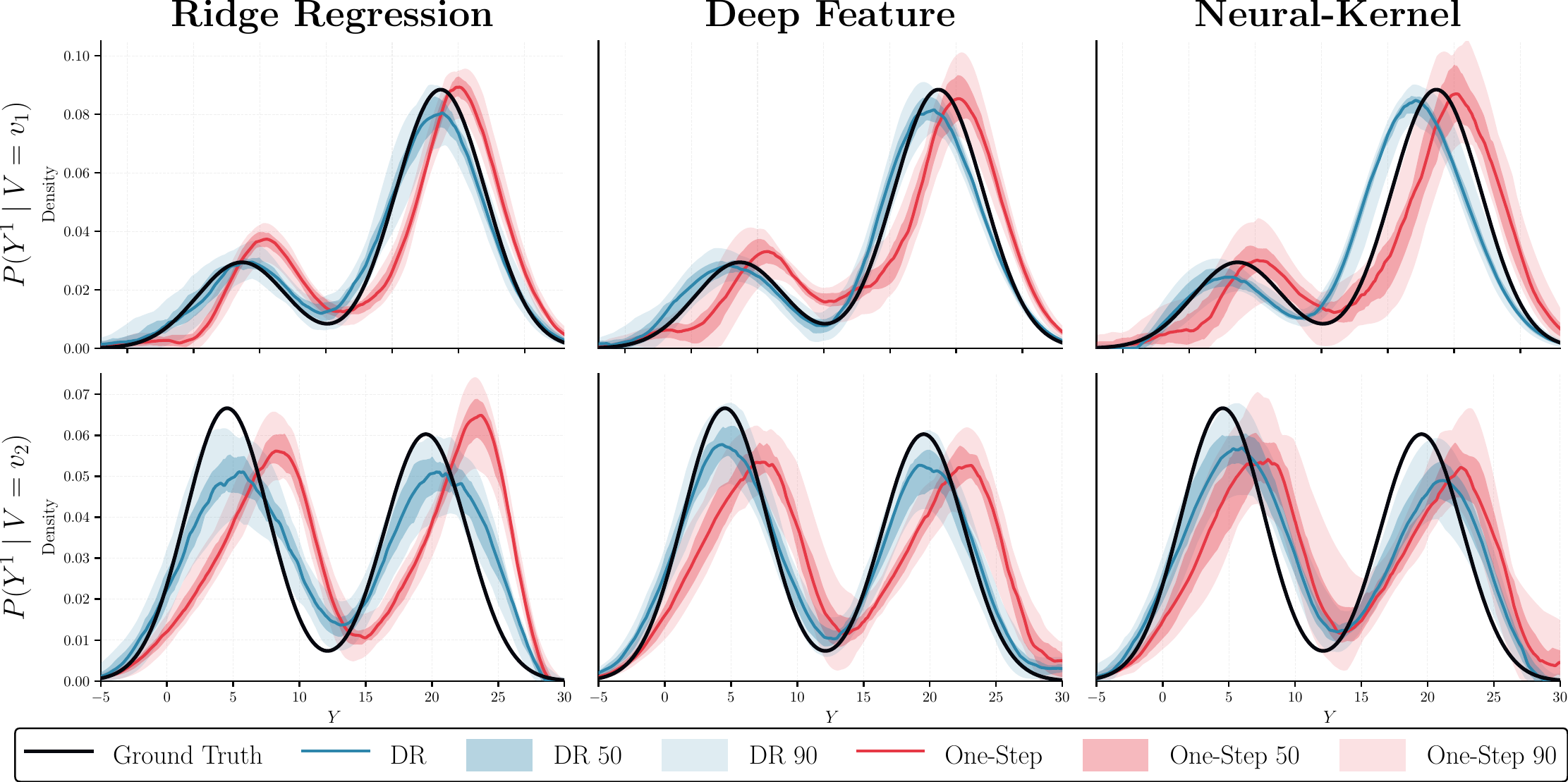}
            \caption{}
            \label{fig:density}
        \end{subfigure}

        \vspace{1em}

        \begin{subfigure}[b]{\textwidth}
            \centering
            \includegraphics[width=\textwidth]{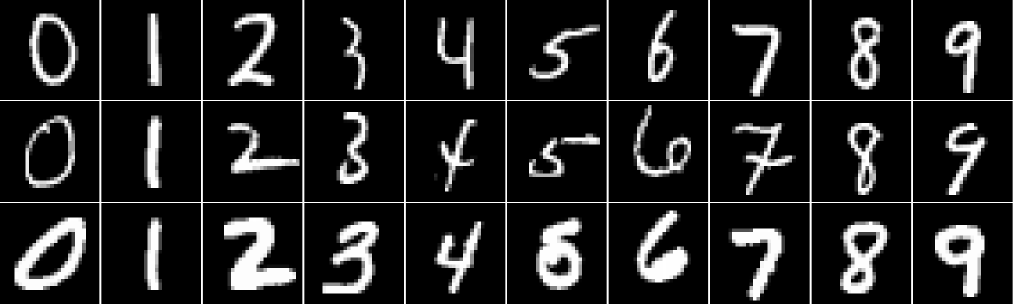}
            \caption{}
            \label{fig:mnist}
        \end{subfigure}

    \end{minipage}

    \caption{(a) MSE vs.\ sample size for three methods (columns) under three misspecification scenarios (rows). Shaded regions indicate standard errors over 10 seeds; (b) Estimated counterfactual densities for profiles $v_1$ (top) and $v_2$ (bottom) defined in Appendix~\ref{sec:exp_details} over 30 runs at $n=10000$. Blue: DR; Red: One-Step. Solid lines are medians; dark and light regions show pointwise 50th and 90th percentile intervals; (c) Empirical mode for each class under three density estimates. Row~1: Oracle estimates on clean data. Row~2: DR estimates (proposed). Row~3: One-Step estimates.}
    \label{fig:main_figure}
\end{figure}

\paragraph{Results.} Figure~\ref{fig:mses} shows the convergence behavior. When both models are correct (top row), DR consistently achieves lower MSE than One-Step across all sample sizes and all three methods. This gap persists even at large sample sizes, indicating that One-Step suffers from selection bias. The DR, IPW, and PI variants of the estimators achieve comparable MSE in this setting; however, for the Ridge Regression estimator, PI exhibits slightly higher MSE than DR and IPW due to selection bias. For the Deep Feature and Neural-Kernel estimators, IPW performs worse than DR and PI due to the first-stage CME estimator.

When $\pi$ is misspecified (middle row), IPW variants of the Deep Feature and Neural-Kernel estimators converge more slowly than DR and PI, as the latter two leverage the correct outcome model. Conversely, when $\mu_0$ is misspecified (bottom row), PI variants of all estimators fail to converge while DR and IPW remain consistent via the correct propensity score. In all cases, DR matches the superior performance of either PI or IPW, confirming the double robustness property of our estimators.

To demonstrate our estimators' capacity for bias correction, we compare their density estimates for $Y^1$ against those from the corresponding One-Step estimators. The plots of conditional densities with two different values of $V$ are shown in Figure~\ref{fig:density}. We observe that the One-Step estimates show noticeable bias, especially in the Ridge Regression and Deep Feature estimators. The DR estimates, on the other hand, are closer to the ground truth, suggesting that the DR estimator effectively mitigates selection bias.

\subsection{Semi-Synthetic MNIST Experiment}\label{sec:exp3}

\paragraph{Setup.} To evaluate scalability to high-dimensional outcomes, we consider a task of estimation of class-conditional densities from noisy MNIST images $\gD_{\mathrm{MNIST}} \coloneqq \{Y_i\}_{i=1}^{60000}$, where $Y_i \in [0,1]^{784}$. One of our covariates is the image intensities $I \coloneqq \{I_i\}_{i=1}^{60000}$, where $I_i \coloneqq \frac{1}{784}\sum_{j=1}^{784} Y_{ij}$.  Due to budget constraints, the denoiser prefers to denoise easy images that have high intensities. Thus, we can assign a treatment variable $A_i$ to each image, with $A_i=1$ if the image is clean and $A_i=0$ if it is noisy. To simulate this task, we assume that the probability of an image being denoised is $\pi(I_i) = \mathsf{sigmoid}\left(3 (I_i - q_{0.9}(I))/\text{std}(I)\right)$, where $\mathsf{sigmoid}(x) \coloneqq 1/(1+e^{-x})$ and $q_{0.9}(I)$ is the $0.9$ quantile of $I$, so that high-intensity images are more likely to be denoised. We let $Y^1_i = Y_i$ and $\widetilde{Y}_i \coloneqq Y_i + \epsilon_i$ where $\epsilon_i \sim \mathcal{N}(0, 0.16 \mathbf{I}_{784})$. Let $V_i$ be the digit label and $X_i = (I_i, V_i)$ the covariate. Our goal is to recover the conditional densities of $Y^1$: $\fp^1(y|v)$ for each $v \in \{0, \ldots, 9\}$. We use the Deep Feature estimator in both stages, with a ReLU network of two hidden layers of 100 units and a Gaussian kernel with bandwidth $1$. The remaining configuration is given in Table~\ref{tab:exp3_config}.

\paragraph{Results.} Figure~\ref{fig:mnist} shows the retrieved empirical modes $\widehat{y}_{\mathrm{mode}}(v) \coloneqq\argmax_{y \in \gD_{\mathrm{MNIST}}} \widehat{\fp}^1(y|v)$ across all 10 classes. The DR estimator produces similar images to the oracle estimator, while the One-Step estimator retrieves the correct digits but with bolder strokes due to the selection bias toward high-intensity images. The DR estimator addresses this bias through the propensity score weighting.

\section{Conclusion}

We study the Conditional Counterfactual Mean Embeddings, a framework for representing conditional counterfactual distributions in an RKHS. We develop a two-stage meta-estimator based on doubly robust pseudo-outcomes and instantiate it with three estimators. Our analysis establishes convergence rates that explicitly track the dependence on the data dimensions and the smoothness of the conditional density, and shows that all three estimators achieve double robustness in rates. Experiments on synthetic and semi-synthetic data confirm that the DR estimators recover multimodal conditional counterfactual densities and correct for selection bias. A limitation is the lack of a lower bound for CCME estimation, which we plan to explore in future work. Moreover, the conditional distribution of $Y^1 - Y^0$ given $V$ depends on the joint distribution of $(Y^1, Y^0)$, which is identified only under additional assumptions such as rank invariance~\cite{Heckman1997}. Our estimators instead target the conditional distributions of $Y^1$ and $Y^0$ given $V$. These estimates can be used for functionals of the two distributions, such as quantile and conditional value-at-risk (CVaR) treatment effects (Appendix~\ref{sec:downstream}).

\section*{Acknowledgements}
TA and DP were financially supported by the Faculty of Science, Chiang Mai University, Thailand. TA also received  the TA/RA Scholarship from Chiang Mai University.

\bibliographystyle{alpha}
{\small{
\bibliography{ccme}
}}

\appendix

\section{Related Work}\label{sec:related}

The estimation of counterfactual distributions has evolved through several approaches. Early work focused on estimating cumulative distribution functions under linear parametric assumptions \cite{Abadie2002,Chernozhukov2013,Diaz2017}, while another line of work focused on estimating probability density functions using either kernel smoothing combined with propensity score weighting \cite{DiNardo1996} or doubly robust methods \cite{Bickel2001,Kim2024}. More recently, \cite{Kennedy2023} proposed an alternative approach using cosine series approximation within a doubly robust framework. Other recent advances in density estimation include doubly robust methods via kernel Stein discrepancies \cite{Liu2016,Chwialkowski2016,Gorham2017} developed by \cite{Martinez-Taboada24}, and via normalizing flows \cite{Tabak2010,Rezende2015} developed by \cite{Melnychuk23}.

A significant development occurred with the introduction of kernel mean embeddings (KME) into causal inference. \cite{Muandet2021} proposed Counterfactual Mean Embeddings, which represent \emph{marginal} counterfactual distributions in an RKHS. They proposed to estimate these embeddings using the plug-in estimator, that is, the kernel ridge regression \cite{Song2009, Grunewalder2012} on each treatment group. These embeddings enabled applications in treatment effect measurement and hypothesis testing of distributional effects; for example, doubly robust estimators for Counterfactual Mean Embeddings were employed to perform hypothesis testing of causal effects \cite{Martinez-Taboada23,Fawkes2024}.

Extensions to more general settings have also been explored. \cite{Luedtke2024} established a comprehensive framework for one-step estimation of differentiable functional-valued parameters, proving consistency and asymptotic linearity of their estimator. \cite{Zenati2025} studied RKHS embeddings of counterfactual distributions under stochastic discrete interventions, which include binary treatments as a special case. They leveraged the results of \cite{Luedtke2024} to propose a doubly robust estimator. However, a key limitation of both \cite{Luedtke2024} and \cite{Zenati2025} is their focus on \emph{marginal} counterfactual distributions rather than conditional counterfactual distributions.

The literature on \emph{conditional} counterfactual distributions (CCD) remains relatively sparse. \cite{Park2021} introduced CCME for hypothesis testing of CCD, though their estimator is a one-step plug-in estimator, which may suffer from selection bias. In contrast, our doubly robust estimators correct this bias through propensity score weighting. \cite{Singh2023} developed kernel estimators of dose, heterogeneous, and incremental response functions, which include embeddings of conditional counterfactual distributions in continuous treatment regimes, whereas our work focuses on binary treatments. In a complementary direction, \cite{Kallus2023} developed doubly robust methods for estimating treatment effects on various statistics of CCD, including quantiles and conditional value-at-risk. However, our CCME framework offers a distinct advantage: by preserving the full distributional information as a functional object in an RKHS, it enables a broader range of downstream tasks such as density estimation \cite{Kanagawa2014}, conditional independence testing \cite{Fukumizu2007,Zhang2011,Huang2022,Scetbon2022}, and sampling \cite{Chen2010}.

\section{Algorithms}\label{sec:algorithms}

The algorithms for the Ridge Regression estimator, Deep Feature estimator, and Neural-Kernel estimator are detailed in Algorithm~\ref{alg:empirical_ccme}, Algorithm~\ref{alg:deep_feature}, and Algorithm~\ref{alg:neural_kernel}, respectively. In each algorithm, $m$ denotes the number of treated units in $\gD_0$.

\begin{algorithm}[!htpb]\small
\caption{Ridge Regression Estimator}
\label{alg:empirical_ccme}
\begin{algorithmic}[1] 
\REQUIRE Dataset $\mathcal{D} \coloneqq \{Z_i\}_{i=1}^{2n}$ where $Z_i = (X_i, A_i, Y_i)$
\REQUIRE Kernels $k_{\mathcal{X}}, k_{\mathcal{V}}, k_{\mathcal{Y}}$ on $\mathcal{X}, \mathcal{V}, \mathcal{Y}$ 
\REQUIRE Propensity score estimation algorithm $\mathsf{Alg}_{\pi}$
\REQUIRE Regularization parameters $\lambda_0, \lambda_1 > 0$

\vspace{0.1cm}

\STATE Randomly split $\mathcal{D}$ into $\gD_0=\{Z_{0i}\}_{i=1}^n$ and $\gD_1=\{Z_{1i}\}_{i=1}^n$
\STATE $\{Z_{0(i)}\}_{i=1}^m \leftarrow \{ Z_{0i} \in \gD_0 \mid A_{0i} = 1\}$
\vspace{0.1cm}
\STATE \textbf{Stage 1: Nuisance Estimation on $\mathcal{D}_0$}
\STATE \quad $\widehat{\pi} \leftarrow \mathsf{Alg}_{\pi}(\gD_0)$
\STATE \quad Compute Gram matrix $\rmK_{X_0} \in \mathbb{R}^{m \times m}$ with $(\rmK_{X_0})_{ij} \coloneqq k_{\mathcal{X}}(X_{0(i)}, X_{0(j)})$
\STATE \quad $\rmk_{X_0}(x) \leftarrow (k_{\mathcal{X}}(x, X_{0(1)}), \ldots, k_{\mathcal{X}}(x, X_{0(m)}))^\top$
\STATE \quad $\widehat\mu_{0}(x) \leftarrow \sum_{j=1}^{m} \alpha_j(x) \phi(Y_{0(j)})$ with $\bm{\alpha}(x) \coloneqq (\rmK_{X_0} + m\lambda_0 \mathbf{I}_m)^{-1}\rmk_{X_0}(x)$

\vspace{0.1cm}

\STATE \textbf{Stage 2: Pseudo-Outcome Regression on $\mathcal{D}_1$}
\STATE \quad \textbf{for} $i = 1$ to $n$ \textbf{do}
\STATE \quad \quad $\widehat{\xi}(Z_{1i}) \leftarrow \frac{A_{1i}}{\widehat{\pi}(X_{1i})}\left(\phi(Y_{1i}) - \widehat\mu_{0}(X_{1i})\right) + \widehat\mu_{0}(X_{1i})$
\STATE \quad \textbf{end for}
\STATE \quad Compute Gram matrix $\rmK_V \in \mathbb{R}^{n \times n}$ with $(\rmK_V)_{ij} = k_{\mathcal{V}}(V_{1i}, V_{1j})$
\STATE \quad $\bm{\beta}(v) \leftarrow (\rmK_V + n\lambda_1 \mathbf{I}_n)^{-1}\rmk_V(v)$
\STATE  \quad $\widehat{\mu}_{Y^1|V}(v) \leftarrow \sum_{i=1}^{n} \beta_i(v) \widehat{\xi}(Z_{1i})$

\vspace{0.1cm}

\RETURN $\widehat{\mu}_{Y^1|V}$
\end{algorithmic}
\end{algorithm}

\begin{algorithm}[!htpb] \small
\caption{Deep Feature Estimator}
\label{alg:deep_feature}
\begin{algorithmic}[1]
\REQUIRE Dataset $\mathcal{D} = \{Z_i\}_{i=1}^{2n}$ where $Z_i = (X_i, A_i, Y_i)$
\REQUIRE Kernel $k_{\mathcal{Y}}$ on $\mathcal{Y}$
\REQUIRE Propensity score estimation algorithm $\mathsf{Alg}_{\pi}$
\REQUIRE Neural network parameter spaces $\Theta_0, \Theta_1$
\REQUIRE Feature dimension $M$
\REQUIRE Regularization parameters $\lambda_0, \lambda_1 > 0$

\vspace{0.1cm}

\STATE Randomly split $\mathcal{D}$ into $\gD_0=\{Z_{0i}\}_{i=1}^n$ and $\gD_1=\{Z_{1i}\}_{i=1}^n$
\STATE $\{Z_{0(i)}\}_{i=1}^m \leftarrow \{ Z_{0i} \in \gD_0 \mid A_{0i} = 1\}$
\vspace{0.1cm}
\STATE \textbf{Stage 1: Nuisance Estimation on $\mathcal{D}_0$}
\STATE \quad $\widehat{\pi} \leftarrow \mathsf{Alg}_{\pi}(\gD_0)$
\STATE \quad Initialize neural network $\psi_{0,\theta}: \mathcal{X} \to \mathbb{R}^M$ with $\theta \in \Theta_0$
\STATE \quad $\rmPsi_{0,\theta} \leftarrow (\psi_{0,\theta}(X_{0(1)}), \ldots, \psi_{0,\theta}(X_{0(m)}))^\top $
\STATE \quad Compute Gram matrix $\rmK_{Y_0} \in \mathbb{R}^{m \times m}$ with  $(\rmK_{Y_0})_{ij} \coloneqq k_{\mathcal{Y}}(Y_{0(i)}, Y_{0(j)})$
\STATE \quad $\gL_0(\theta) \leftarrow {-\text{Tr}\left( \rmK_{Y_0} \rmPsi_{0,\theta}(\rmPsi_{0,\theta}^\top\rmPsi_{0,\theta} + m\lambda_0\mathbf{I}_{M})^{-1}\rmPsi_{0,\theta}^\top \right)}$
\STATE \quad $\widehat{\theta}_0 \leftarrow \argmin_{\theta \in \Theta_0} \gL_0(\theta)$
\STATE \quad $\widehat{\rmPsi}_0 \leftarrow (\psi_{0,\widehat{\theta}_0}(X_{0(1)}), \ldots, \psi_{0,\widehat{\theta}_0}(X_{0(m)}))^\top$

\vspace{0.1cm}

\STATE \textbf{Stage 2: Pseudo-Outcome Regression on $\mathcal{D}_1$}
\STATE \quad Initialize neural network $\psi_{\theta}: \mathcal{V} \to \mathbb{R}^M$ with $\theta \in \Theta_1$
\STATE \quad $\rmPsi_{\theta} \leftarrow (\psi_{\theta}(V_{11}), \ldots, \psi_{\theta}(V_{1n}))^\top$
\STATE \quad $\widehat{\rmPsi}_{01} \leftarrow (\psi_{0,\widehat{\theta}_0}(X_{11}), \ldots, \psi_{0,\widehat{\theta}_0}(X_{1n}))$
\STATE \quad Compute Gram matrix $\rmK_{Y_1} \in \mathbb{R}^{n \times n}$ with $(\rmK_{Y_1})_{ij} \coloneqq k_{\mathcal{Y}}(Y_{1i}, Y_{1j})$
\STATE \quad Compute Gram matrix $\rmK_{Y_0Y_1} \in \mathbb{R}^{m \times n}$ with $(\rmK_{Y_0Y_1})_{ij} \coloneqq k_{\mathcal{Y}}(Y_{0(i)}, Y_{1j})$
\STATE \quad $\mathbf{C} \leftarrow \widehat{\rmPsi}_0 (\widehat{\rmPsi}_0^\top \widehat{\rmPsi}_0 + m\lambda_0 \mathbf{I}_{M})^{-1} \widehat{\rmPsi}_{01}$
\STATE \quad $\mathbf{D}_\omega \leftarrow \text{diag}(A_{11}/\widehat{\pi}(X_{11}), \ldots, A_{1n}/\widehat{\pi}(X_{1n}))$
\STATE \quad $\mathbf{D}_{1-\omega} \leftarrow \mathbf{I}_n - \mathbf{D}_\omega$
\STATE \quad $\rmK_{\widehat{\xi}} \leftarrow \mathbf{D}_\omega \rmK_{Y_1} \mathbf{D}_\omega + \mathbf{D}_\omega \rmK_{Y_0 Y_1}^\top \mathbf{C} \mathbf{D}_{1-\omega} + \mathbf{D}_{1-\omega} \mathbf{C}^\top \rmK_{Y_0 Y_1} \mathbf{D}_\omega + \mathbf{D}_{1-\omega} \mathbf{C}^\top \rmK_{Y_0} \mathbf{C} \mathbf{D}_{1-\omega}$
\STATE \quad $\gL_1(\theta) \leftarrow {-\text{Tr}\left( \rmK_{\widehat{\xi}} \rmPsi_{\theta}(\rmPsi_{\theta}^\top\rmPsi_{\theta} + n\lambda_1\mathbf{I}_{M})^{-1}\rmPsi_{\theta}^\top \right)}$
\STATE \quad $\widehat{\theta}_1 \leftarrow \argmin_{\theta \in \Theta_1} \gL_1(\theta)$ 
\STATE \quad $\widehat{\rmPsi}_1 \leftarrow (\psi_{\widehat{\theta}_1}(V_{11}), \ldots, \psi_{\widehat{\theta}_1}(V_{1n}))^\top$
\STATE \quad $\widehat{\mu}_{Y^1|V}(v) \leftarrow \sum_{i=1}^n c_i(v) \widehat{\xi}(Z_{1i})$ with $\mathbf{c}(v) \coloneqq \widehat{\rmPsi}_1 (\widehat{\rmPsi}_1^\top\widehat{\rmPsi}_1 + n\lambda_1\mathbf{I}_{M})^{-1} \psi_{\widehat{\theta}_1}(v)$

\vspace{0.1cm}

\RETURN $\widehat{\mu}_{Y^1|V}$
\end{algorithmic}
\end{algorithm}

\begin{algorithm}[!ht]\small
\caption{Neural-Kernel Estimator}
\label{alg:neural_kernel}
\begin{algorithmic}[1]
\REQUIRE Dataset $\mathcal{D} = \{Z_i\}_{i=1}^{2n}$ where $Z_i = (X_i, A_i, Y_i)$
\REQUIRE Kernel $k_{\mathcal{Y}}$ on $\mathcal{Y}$
\REQUIRE Propensity score estimation algorithm $\mathsf{Alg}_{\pi}$
\REQUIRE Neural network parameter spaces $\Theta_0, \Theta_1$
\REQUIRE Grid points $\widetilde{\mathcal{Y}}_M = \{\widetilde{y}_1, \ldots, \widetilde{y}_M\} \subset \mathcal{Y}$

\vspace{0.1cm}

\STATE Randomly split $\mathcal{D}$ into two disjoint sets: $\mathcal{D}_0$ and $\mathcal{D}_1$, each of size $n$
\STATE $\{Z_{0(i)}\}_{i=1}^m \leftarrow \{ Z_{0i} \in \gD_0 \mid A_{0i} = 1\}$

\vspace{0.1cm}
\STATE \textbf{Stage 1: Nuisance Estimation on $\mathcal{D}_0$}
\STATE \quad $\widehat{\pi} \leftarrow \mathsf{Alg}_{\pi}(\gD_0)$
\STATE \quad Compute grid Gram matrix $\rmK_M \in \mathbb{R}^{M \times M}$ with $(\rmK_M)_{jl} = k_{\mathcal{Y}}(\widetilde{y}_j, \widetilde{y}_l)$
\STATE \quad Initialize neural network $g_{\theta}: \mathcal{X} \to \mathbb{R}^M$ with $\theta \in \Theta_0$
\STATE \quad $\gL_0(\theta) \leftarrow \frac{1}{m}\sum_{i=1}^m \left[ g_{\theta}(X_{0(i)})^\top \rmK_M g_{\theta}(X_{0(i)}) - 2 \sum_{j=1}^M g_{\theta}(X_{0(i)})_j k_{\mathcal{Y}}(\widetilde{y}_j, Y_{0(i)}) \right]$
\STATE \quad $\widehat{\theta}_0 \leftarrow \argmin_{\theta \in \Theta_0} \gL_0(\theta)$

\vspace{0.1cm}

\STATE \textbf{Stage 2: Pseudo-Outcome Regression on $\mathcal{D}_1$}
\STATE \quad Initialize neural network $f_{\theta}: \mathcal{V} \to \mathbb{R}^M$ with $\theta \in \Theta_1$
\STATE \quad \textbf{for} $i = 1$ to $n$ \textbf{do}
\STATE \quad \quad $\rmk_i \leftarrow (k_\gY(\widetilde{y}_1, Y_{1i}), \ldots, k_\gY(\widetilde{y}_M, Y_{1i}))^\top$
\STATE \quad \quad $\mathbf{b}_i \leftarrow \frac{A_{1i}}{\widehat{\pi}(X_{1i})} \rmk_i + \left(1 - \frac{A_{1i}}{\widehat{\pi}(X_{1i})}\right) \rmK_{M} g_{\widehat\theta_0}(X_{1i})$
\STATE \quad \quad $\gL_{1,i}(\theta) \leftarrow f_{\theta}(V_{1i})^\top \rmK_M f_{\theta}(V_{1i}) - 2 f_{\theta}(V_{1i})^\top \mathbf{b}_i$
\STATE \quad \textbf{end for}
\STATE \quad $\gL_1(\theta) \leftarrow \frac{1}{n} \sum_{i=1}^n \gL_{1,i}(\theta)$
\STATE \quad $\widehat{\theta}_1 \leftarrow \argmin_{\theta \in \Theta_1} \gL_1(\theta)$ \STATE \quad $\widehat{\mu}_{Y^1|V}(v) \leftarrow \sum_{j=1}^M f_{\widehat{\theta}_1}(v)_j \phi(\widetilde{y}_j)$

\vspace{0.1cm}

\RETURN $\widehat{\mu}_{Y^1|V}$
\end{algorithmic}
\end{algorithm}

\section{Application to Conditional Counterfactual Density Estimation}\label{sec:ccd_estimation}

The CCME can be used to recover the full conditional counterfactual density (CCD) $\fp^1(y|v)$. If $k_\gY$ is a translation-invariant kernel that integrates to one (e.g., the Gaussian kernel $k_\gY(y, y') = (\sqrt{2\pi}\sigma)^{-d_y} \exp(-\norm{y-y'}^2/(2\sigma^2))$), the CCME satisfies:
 \[ \langle \mu_{Y^1|V}(v), \phi(y) \rangle_{\gH_\gY}
 = \E_{Y^1\sim \fp^1(\cdot |v)}[k_\gY(Y^1, y)], \]
 which is the kernel-smoothed density $\bar\fp^1(y|v)$. Thus, we obtain the estimated density $\widehat{\fp}^1(y|v)$ at any query point $y \in \gY$ by computing the inner product:
\begin{equation*}\label{eq:density_from_ccme}
    \widehat{\fp}^1(y|v) = \langle \widehat{\mu}_{Y^1|V}(v), \phi(y) \rangle_{\gH_\gY}.
\end{equation*}
Below, we derive the explicit formulae for the density estimators based on all three proposed methods. For the formulae to be explicit, we denote the set of the treated units in $\gD_0$ by $\gD_0^{A=1} \coloneqq \{Z_{0(1)},\ldots,Z_{0(m)}\} = \{Z_{0i} \in \gD_0 \mid A_{0i}=1\}$. We denote $\boldsymbol{\Phi}_0 \coloneqq (\phi(Y_{0(1)}), \ldots, \phi(Y_{0(m)}))^\top$. Define $\rmk_{Y_1}(y) \coloneqq (k_\gY(y, Y_{11}),\ldots,k_\gY(y, Y_{1n}))^\top$ and $\rmk_{Y_0}(y) \coloneqq (k_\gY(y, Y_{0(1)}),\ldots,k_\gY(y, Y_{0(m)}))^\top$. For brevity, we also denote $\omega_i \coloneqq A_{1i}/\widehat{\pi}(X_{1i})$, $\mathbf{D}_\omega \coloneqq \text{diag}(\omega_1, \dots, \omega_n)$, and $\mathbf{D}_{1-\omega} \coloneqq \mathbf{I}_n - \mathbf{D}_\omega$.

\subsection{Inference with Ridge Regression Estimator}

We recall from \eqref{eq:emp_ccme_form} the Ridge Regression estimator $\widehat{\mu}_{\mathrm{RR}}(v) = \boldsymbol{\widehat\Xi} (\rmK_{V} + n \lambda_1 \mathbf{I}_n)^{-1}\rmk_{V}(v)$. We assume the first-stage CME estimator is also the kernel ridge regression fitted on the treatment group in $\mathcal{D}_0$, denoted by $\widehat{\mu}_{0,\mathrm{RR}}(x) \coloneqq \boldsymbol{\Phi}_0(\rmK_{X_0} + m\lambda_0 \rmI_m)^{-1}\rmk_{X_0}(x)$ where $\rmK_{X_0} = (k_{\mathcal{X}}(X_{0(i)}, X_{0(j)}))_{i,j=1}^m$ and $\rmk_{X_0}(x)= (k_{\mathcal{X}}(x, X_{0(1)}),\ldots,k_{\mathcal{X}}(x, X_{0(m)}))^{\top}$. Constructing the pseudo-outcome $\widehat\xi(Z_{1i}) = \omega_i(\phi(Y_{1i}) - \widehat{\mu}_{0,\mathrm{RR}}(X_{1i}))+\widehat{\mu}_{0,\mathrm{RR}}(X_{1i})$, we obtain the CCD estimator:
\begin{align}\label{eq:emp_density}
    \widehat{\fp}^1_{\mathrm{RR}}(y|v) &= \left\langle \boldsymbol{\widehat\Xi}
    (\rmK_{V} + n \lambda_1 \mathbf{I}_n)^{-1}\rmk_{V}(v), \phi(y) \right\rangle_{\gH_\gY} \nonumber \\
     &= (\langle \widehat\xi(Z_{11}), \phi(y) \rangle_{\gH_\gY}, \ldots, \langle \widehat\xi(Z_{1n}), \phi(y) \rangle_{\gH_\gY})^\top(\rmK_{V} + n \lambda_1 \mathbf{I}_n)^{-1}\rmk_{V}(v) \nonumber \\
    &= \left[ \rmk_{Y_1}(y)^\top\rmD_\omega + \rmk_{Y_0}(y)^\top (\rmK_{X_0} + m\lambda_0 \rmI_m)^{-1}\rmK_{X_0X_1}\rmD_{1-\omega} \right](\rmK_{V} + n \lambda_1 \mathbf{I}_n)^{-1}\rmk_{V}(v),
\end{align}
where $\rmK_{X_0X_1} \coloneqq (k_\gX(X_{0(i)}, X_{1j}))_{1 \le i \le m,1 \le j \le n}$ is a cross-kernel matrix.

\subsection{Inference with Deep Feature Estimator}\label{sec:df_density}

Recall from \eqref{eq:C_hat} the Deep Feature estimator $\widehat{\mu}_{\mathrm{DF}}(v) = \boldsymbol{\widehat\Xi} \rmPsi_{\widehat\theta} (\rmPsi_{\widehat\theta}^{\top}\rmPsi_{\widehat\theta} + n\lambda_1\mathbf{I}_{M})^{-1} \psi_{\widehat\theta}(v)$. Analogously, with a learned feature map $\psi_{0,\widehat\theta_0}$, denoting $\boldsymbol{\widehat\Psi}_{0} \coloneqq(\psi_{0,\widehat\theta_0}(X_{0(1)}),\ldots,\psi_{0,\widehat\theta_0}(X_{0(m)}))^\top$, the first-stage Deep Feature estimator of the CME is given by:
\begin{equation*}\label{eq:df_cross_pred}
    \widehat{\mu}_{0,\mathrm{DF}}(x) \coloneqq \boldsymbol{\Phi}_0\boldsymbol{\widehat\Psi}_{0} (\boldsymbol{\widehat\Psi}_{0}^\top \boldsymbol{\widehat\Psi}_{0} + m\lambda_0 \mathbf{I}_{M})^{-1} \psi_{0,\widehat\theta_0}(x).
\end{equation*}
Constructing the pseudo-outcome $\widehat\xi(Z_{1i})$ based on $\widehat{\mu}_{0,\mathrm{DF}}(X_{1i})$, we obtain the CCD estimator:
\begin{align}\label{eq:df_density}
    &\widehat{\fp}^1_{\mathrm{DF}}(y|v) \nonumber \\
    &\quad  = \left\langle \boldsymbol{\widehat\Xi}
    \rmPsi_{\widehat\theta} (\rmPsi_{\widehat\theta}^{\top}\rmPsi_{\widehat\theta} + n\lambda_1\mathbf{I}_{M})^{-1} \psi_{\widehat\theta}(v), \phi(y) \right\rangle_{\gH_\gY} \nonumber \\
    &\quad  = \left[ \rmk_{Y_1}(y)^\top\rmD_{\omega} + \rmk_{Y_0}(y)^\top \boldsymbol{\widehat\Psi}_{0}(\boldsymbol{\widehat\Psi}_{0}^\top \boldsymbol{\widehat\Psi}_{0} + m\lambda_0 \mathbf{I}_{M})^{-1} \boldsymbol{\widehat\Psi}_{01} \rmD_{1-\omega}  \right] \nonumber \\
    &\qquad \rmPsi_{\widehat\theta} (\rmPsi_{\widehat\theta}^{\top}\rmPsi_{\widehat\theta} + n\lambda_1\mathbf{I}_{M})^{-1} \psi_{\widehat\theta}(v),
\end{align}
where $\boldsymbol{\widehat\Psi}_{01} \coloneqq (\psi_{0,\widehat\theta_0}(X_{11}),\ldots,\psi_{0,\widehat\theta_0}(X_{1n}))$ is the first-stage network evaluated on the second-stage covariates.

\subsection{Inference with Neural-Kernel Estimator}\label{sec:nk_density}

Since the Neural-Kernel estimator $\widehat{\mu}_{\mathrm{NK}}(v) = \sum_{j=1}^M f_{\widehat{\theta}}(v)_j \phi(\widetilde{y}_j)$ is expressed as a linear combination of feature maps, the CCD estimator can be easily computed:
\begin{equation}\label{eq:nk_density}
    \widehat{\fp}^1_{\mathrm{NK}}(y|v) =  \left\langle \sum_{j=1}^M f_{\widehat{\theta}}(v)_j \phi(\widetilde{y}_j), \phi(y) \right\rangle_{\gH_\gY} = \sum_{j=1}^M f_{\widehat{\theta}}(v)_j k_\gY(\widetilde{y}_j, y).
\end{equation} 
This makes inference with the Neural-Kernel estimator significantly faster compared to the previous two estimators, since there is no need for calculation or inversion of large matrices.

\section{Baseline Estimators for Conditional Counterfactual Densities}\label{sec:baseline}

In this section, we derive the CCD estimators for the two ablation baselines used in the experiments: the Inverse Propensity Weighting (IPW) estimator and the Plug-in (PI) estimator. These estimators use only one component of the doubly robust pseudo-outcome $\widehat{\xi}(Z)$ from \eqref{eq:pseudo_outcome}. Recall that the doubly robust pseudo-outcome can be written as:
\begin{equation*}
    \widehat{\xi}(Z_{1i}) = \omega_i \phi(Y_{1i}) - \omega_i\widehat\mu_{0}(X_{1i}) + \widehat\mu_{0}(X_{1i}),
\end{equation*}
where $\omega_i = A_{1i}/\widehat{\pi}(X_{1i})$. The IPW and PI estimators correspond to retaining only the first or last term on the right-hand side, respectively. For each baseline, we derive the CCD estimators for all three methods by specializing the doubly robust density formulae from Appendix~\ref{sec:ccd_estimation}.

\subsection{Inverse Propensity Weighting Estimator}

The Inverse Propensity Weighting (IPW) estimator relies solely on the propensity score to correct for the covariate shift between the observational and counterfactual distributions. The pseudo-outcome is:
\begin{equation*}\label{eq:pseudo_ipw}
    \widehat{\xi}_{\mathrm{IPW}}(Z_{1i}) = \frac{A_{1i}}{\widehat{\pi}(X_{1i})} \phi(Y_{1i}) = \omega_i \phi(Y_{1i}).
\end{equation*}
Since $\widehat{\xi}_{\mathrm{IPW}}$ does not involve the outcome model $\widehat\mu_{0}$, the inner product with $\phi(y)$ reduces to:
\begin{equation}\label{eq:ipw_inner}
    \langle \widehat{\xi}_{\mathrm{IPW}}(Z_{1i}), \phi(y) \rangle_{\gH_\gY} = \omega_i k_\gY(Y_{1i}, y),
\end{equation}
yielding the vector $\rmD_\omega \rmk_{Y_1}(y)$ when collected over $i = 1, \ldots, n$.

The first-stage training remains the same for all methods. We describe the second stage for each method below.

\textbf{Ridge Regression.} Replacing the bracket term in \eqref{eq:emp_density} by $\rmD_\omega \rmk_{Y_1}(y)$, we obtain:
\begin{equation*}\label{eq:ipw_rr_density}
    \widehat{\fp}^1_{\mathrm{IPW,RR}}(y|v) = \rmk_{Y_1}(y)^\top \rmD_\omega (\rmK_{V} + n\lambda_1 \mathbf{I}_n)^{-1} \rmk_{V}(v).
\end{equation*}
\textbf{Deep Feature.} To train the second-stage network $\psi_\theta$, we substitute the IPW pseudo-outcome Gram matrix into the loss \eqref{eq:first_stage_DF_psi}. The entries of this matrix are:
\begin{equation*}\label{eq:ipw_df_gram}
    (\rmK_{\widehat{\xi}}^{\mathrm{IPW}})_{ij} = \langle \widehat{\xi}_{\mathrm{IPW}}(Z_{1i}), \widehat{\xi}_{\mathrm{IPW}}(Z_{1j}) \rangle_{\gH_\gY} = \omega_i \omega_j \, k_\gY(Y_{1i}, Y_{1j}),
\end{equation*}
which in matrix form reads $\rmK_{\widehat{\xi}}^{\mathrm{IPW}} = \rmD_\omega \rmK_{Y_1} \rmD_\omega$. Let $\psi_{\widehat\theta}$ be the learned feature map. Replacing the bracket term in \eqref{eq:df_density} by $\rmD_\omega \rmk_{Y_1}(y)$ gives:
\begin{equation*}\label{eq:ipw_df_density}
    \widehat{\fp}^1_{\mathrm{IPW,DF}}(y|v) = \rmk_{Y_1}(y)^\top \rmD_\omega \rmPsi_{\widehat\theta} (\rmPsi_{\widehat\theta}^{\top}\rmPsi_{\widehat\theta} + n\lambda_1\mathbf{I}_{M})^{-1} \psi_{\widehat\theta}(v).
\end{equation*}
\textbf{Neural-Kernel.} The vector $\mathbf{b}_i \in \R^M$ for the Neural-Kernel loss \eqref{eq:NK_loss} simplifies to:
\begin{equation*}\label{eq:ipw_nk_target}
    (\mathbf{b}_i)_j = \left\langle \phi(\widetilde{y}_j), \widehat{\xi}_{\mathrm{IPW}}(Z_{1i}) \right\rangle_{\gH_\gY} = \frac{A_{1i}}{\widehat{\pi}(X_{1i})} k_\gY(\widetilde{y}_j, Y_{1i}),
\end{equation*}
that is, $\mathbf{b}_i = \omega_i \rmk_i$. Substituting into \eqref{eq:NK_loss} yields the training loss:
\begin{equation*}\label{eq:ipw_nk_loss}
    \widehat\gL_{\mathrm{IPW,NK}}(\theta) = \frac{1}{n}\sum_{i=1}^n \left[ f_\theta(V_{1i})^\top \rmK_{M} f_\theta(V_{1i}) - 2\omega_i \, f_\theta(V_{1i})^\top \rmk_i \right].
\end{equation*}
Let $\widehat\theta$ be the parameter obtained from training the model with this loss. The density estimate follows from \eqref{eq:nk_density}: $\widehat{\fp}^1_{\mathrm{IPW,NK}}(y|v) = \sum_{j=1}^M f_{\widehat\theta}(v)_j k_\gY(\widetilde{y}_j, y)$.

\subsection{Plug-in Density Estimator}

The Plug-in (PI) estimator relies solely on the conditional mean embedding of the outcome, ignoring the propensity scores. The pseudo-outcome is the prediction of the first-stage model:
\begin{equation*}\label{eq:pseudo_pi}
    \widehat{\xi}_{\mathrm{PI}}(Z_{1i}) = \widehat\mu_{0}(X_{1i}).
\end{equation*}
The inner product with $\phi(y)$ is:
\begin{equation}\label{eq:pi_inner}
    \langle \widehat{\xi}_{\mathrm{PI}}(Z_{1i}), \phi(y) \rangle_{\gH_\gY} = \langle \widehat\mu_{0}(X_{1i}), \phi(y) \rangle_{\gH_\gY} = \widehat\mu_{0}(X_{1i})(y).
\end{equation}
The first-stage training remains the same for all methods. We describe the second stage for each method below.

\textbf{Ridge Regression.} Assume the first-stage estimator is the Ridge Regression estimator fitted on the treatment group, $\widehat{\mu}_{0,\mathrm{RR}}(x) = \boldsymbol{\Phi}_0 (\rmK_{X_0} + m\lambda_0 \rmI_m)^{-1} \rmk_{X_0}(x)$. We then calculate \eqref{eq:pi_inner}:
\begin{equation*}
    \langle \widehat{\mu}_{0,\mathrm{RR}}(X_{1i}), \phi(y) \rangle_{\gH_\gY} = \rmk_{Y_0}(y)^\top (\rmK_{X_0} + m\lambda_0 \rmI_m)^{-1} \rmk_{X_0}(X_{1i}).
\end{equation*}
Collecting over $i = 1, \ldots, n$ yields the vector $\rmK_{X_0X_1}^\top (\rmK_{X_0} + m\lambda_0 \rmI_m)^{-1} \rmk_{Y_0}(y)$. Replacing the bracket term in \eqref{eq:emp_density} by this vector, we obtain:
\begin{equation*}\label{eq:pi_rr_density}
    \widehat{\fp}^1_{\mathrm{PI,RR}}(y|v) = \rmk_{Y_0}(y)^\top (\rmK_{X_0} + m\lambda_0 \rmI_m)^{-1} \rmK_{X_0X_1} (\rmK_{V} + n\lambda_1 \mathbf{I}_n)^{-1} \rmk_{V}(v).
\end{equation*}
\textbf{Deep Feature.} Assume the first-stage estimator is the Deep Feature estimator with the learned feature map $\psi_{0,\widehat\theta_0}$. To train the second-stage network, we substitute the Gram matrix $\rmK_{\widehat{\xi}}^{\mathrm{PI}}\coloneqq(\langle \widehat\mu_{0}(X_{1i}), \widehat\mu_{0}(X_{1j}) \rangle_{\gH_\gY})_{i,j=1}^n$ into the loss \eqref{eq:first_stage_DF_psi}. Since $\widehat\mu_{0}(X_{1i}) = \boldsymbol{\Phi}_0\boldsymbol{\widehat\Psi}_{0} (\boldsymbol{\widehat\Psi}_{0}^\top \boldsymbol{\widehat\Psi}_{0} + m\lambda_0 \mathbf{I}_{M})^{-1} \psi_{0,\widehat\theta_0}(X_{1i})$, we have:
\begin{equation*}\label{eq:pi_df_gram}
    \rmK_{\widehat{\xi}}^{\mathrm{PI}} = \boldsymbol{\widehat\Psi}_{01}^\top (\boldsymbol{\widehat\Psi}_{0}^\top \boldsymbol{\widehat\Psi}_{0} + m\lambda_0 \mathbf{I}_{M})^{-1} \boldsymbol{\widehat\Psi}_{0}^\top \rmK_{Y_0} \, \boldsymbol{\widehat\Psi}_{0} (\boldsymbol{\widehat\Psi}_{0}^\top \boldsymbol{\widehat\Psi}_{0} + m\lambda_0 \mathbf{I}_{M})^{-1} \boldsymbol{\widehat\Psi}_{01}.
\end{equation*}
After the second-stage optimization, we calculate \eqref{eq:pi_inner}:
\begin{equation*}
    \langle \widehat{\mu}_{0,\mathrm{DF}}(X_{1i}), \phi(y) \rangle_{\gH_\gY} = \rmk_{Y_0}(y)^\top \boldsymbol{\widehat\Psi}_{0} (\boldsymbol{\widehat\Psi}_{0}^\top \boldsymbol{\widehat\Psi}_{0} + m\lambda_0 \mathbf{I}_{M})^{-1} \psi_{0,\widehat\theta_0}(X_{1i}).
\end{equation*}
Collecting over $i = 1, \ldots, n$ and replacing the bracket term in \eqref{eq:df_density} by this vector, we obtain:
\begin{equation*}\label{eq:pi_df_density}
    \widehat{\fp}^1_{\mathrm{PI,DF}}(y|v) = \rmk_{Y_0}(y)^\top \boldsymbol{\widehat\Psi}_{0} (\boldsymbol{\widehat\Psi}_{0}^\top \boldsymbol{\widehat\Psi}_{0} + m\lambda_0 \mathbf{I}_{M})^{-1} \boldsymbol{\widehat\Psi}_{01} \rmPsi_{\widehat\theta} (\rmPsi_{\widehat\theta}^{\top}\rmPsi_{\widehat\theta} + n\lambda_1\mathbf{I}_{M})^{-1} \psi_{\widehat\theta}(v).
\end{equation*}
\textbf{Neural-Kernel.} Assume the first-stage estimator is a Neural-Kernel estimator $\widehat{\mu}_{0,\mathrm{NK}}(x) \coloneqq \sum_{j=1}^M g_{\widehat\theta_0}(x)_j \phi(\widetilde{y}_j)$ with the same grid points. The target vector $\mathbf{b}_i \in \R^M$ for the loss \eqref{eq:NK_loss} simplifies to:
\begin{equation*}\label{eq:pi_nk_target}
    (\mathbf{b}_i)_j = \left\langle \phi(\widetilde{y}_j), \widehat{\mu}_{0,\mathrm{NK}}(X_{1i}) \right\rangle_{\gH_\gY} = \sum_{l=1}^M g_{\widehat\theta_0}(X_{1i})_l \, k_\gY(\widetilde{y}_l, \widetilde{y}_j),
\end{equation*}
that is, $\mathbf{b}_i = \rmK_M g_{\widehat\theta_0}(X_{1i})$. Substituting into \eqref{eq:NK_loss} yields the training loss:
\begin{equation*}\label{eq:pi_nk_loss}
    \widehat\gL_{\mathrm{PI,NK}}(\theta) = \frac{1}{n}\sum_{i=1}^n \left[ f_\theta(V_{1i})^\top \rmK_{M} f_\theta(V_{1i}) - 2 f_\theta(V_{1i})^\top \rmK_{M} g_{\widehat\theta_0}(X_{1i}) \right].
\end{equation*}
Let $\widehat\theta$ be the parameter obtained from training the model with this loss. The density estimate follows from \eqref{eq:nk_density}: $\widehat{\fp}^1_{\mathrm{PI,NK}}(y|v) = \sum_{j=1}^M f_{\widehat\theta}(v)_j k_\gY(\widetilde{y}_j, y)$.

\section{\texorpdfstring{Applications of Theorem~\ref*{thm:ECME}}{Applications of Theorem 1}}\label{sec:applications}

We provide three applications of Theorem~\ref{thm:ECME} to three particular tasks, namely density estimation, hypothesis testing, and sampling of conditional counterfactual distributions. For each task, we demonstrate an application of our meta-estimator bound. Specifically, we bound the error of each task in terms of $\E\left[\norm{\widehat\mu_{Y^1\vert V}(V) - \mu_{Y^1\vert V}(V)}^2_{\gH_\gY}\right]$, which allows us to apply Theorem~\ref{thm:ECME} to bound this error by the first-stage and second-stage errors. 

\subsection{Density Estimation}
Let $k_\gY(y, y') = \sigma^{-d_y}\rho((y - y')/\sigma)$ with a bandwidth $\sigma > 0$, where $\rho \ge 0$ is a positive definite function on $\R^{d_y}$ with $\int\rho(u)\,\mathrm{d}u = 1$ and $\int\norm{u}^2\rho(u)\,\mathrm{d}u < \infty$. As in Appendix~\ref{sec:ccd_estimation}, we can use the density estimator $\widehat\fp^1(y|v) = \inprod{\widehat\mu_{Y^1\vert V}(v)}{\phi(y)}_{\gH_\gY}$ to estimate the conditional counterfactual density $\fp^1(\cdot|v)$. The following corollary provides a bound on the error of $\widehat\fp^1(\cdot | v)$ as an estimator of $\fp^1(\cdot | v)$.
\begin{corollary}[Density Estimation]\label{cor:density}
Let $k_\gY$ be as above. Assume that, for every $v \in \gV$, the function $y \mapsto \fp^1(y|v)$, extended by zero outside $\gY$, is twice differentiable on $\R^{d_y}$ with $\sup_{y}\norm{\nabla^2_y\fp^1(y|v)}_{\mathrm{op}} \le L_\fp$ for some $L_\fp > 0$, where $\norm{\cdot}_{\mathrm{op}}$ is the matrix operator norm. Then,
\[
\E\left[\sup_{y\in\gY}\abs{\widehat\fp^1(y|V) - \fp^1(y|V)}^2\right] \lesssim \frac{1}{\sigma^{d_y}}\E\left[\norm{\widehat\mu_{Y^1\vert V}(V) - \mu_{Y^1\vert V}(V)}^2_{\gH_\gY}\right] + \sigma^4.
\]
\end{corollary}
\begin{proof}
By the reproducing property and the symmetry $\rho(-u) = \rho(u)$, substituting $y' = y - \sigma u$ yields:
\[\bar\fp^1(y|v) \coloneqq \inprod{\mu_{Y^1\vert V}(v)}{\phi(y)}_{\gH_\gY} = \frac1{\sigma^{d_y}}\int \rho\left(\frac{y-y'}{\sigma}\right) \fp^1(y' |v)\,\mathrm{d}y' = \int \rho(u) \fp^1(y - \sigma u |v)\,\mathrm{d}u.\]
We then decompose the target error as $\widehat\fp^1 - \fp^1 = (\widehat\fp^1 - \bar\fp^1) + (\bar\fp^1 - \fp^1)$ and apply the triangle inequality. By the reproducing property and the Cauchy-Schwarz inequality, we have:
\begin{align}
\abs{\widehat\fp^1(y|v) - \bar\fp^1(y|v)} &= \abs{\inprod{\widehat\mu_{Y^1\vert V}(v) - \mu_{Y^1\vert V}(v)}{\phi(y)}_{\gH_\gY}} \notag \\
&\le \sqrt{k_\gY(y,y)}\,\norm{\widehat\mu_{Y^1\vert V}(v) - \mu_{Y^1\vert V}(v)}_{\gH_\gY} \notag \\
&= \sqrt{\frac{\rho(0)}{\sigma^{d_y}}}\,\norm{\widehat\mu_{Y^1\vert V}(v) - \mu_{Y^1\vert V}(v)}_{\gH_\gY}. \label{eq:p1diff}
\end{align}
Next, we bound $\bar\fp^1 - \fp^1$. Since $\int\rho(u)\,\mathrm{d}u = 1$ and $\int u\,\rho(u)\,\mathrm{d}u = 0$ by the symmetry of $\rho$, a second-order Taylor expansion of $\fp^1(y - \sigma u|v)$ around $y$ yields:
\begin{equation}
\abs{\bar\fp^1(y|v) - \fp^1(y|v)} = \abs{\int\rho(u)\big(\fp^1(y - \sigma u|v) - \fp^1(y|v)\big)\,\mathrm{d}u} \le \frac{L_\fp\sigma^2}{2}\int\norm{u}^2\rho(u)\,\mathrm{d}u. \label{eq:p2diff}
\end{equation}
The desired bound then immediately follows from \eqref{eq:p1diff} and \eqref{eq:p2diff}.
\end{proof}

\subsection{Testing}
We denote by $\mu_{Y^0\vert V}$ the CCME of the counterfactual $Y^0$: $\mu_{Y^0\vert V}(v) \coloneqq \E[\phi(Y^0)\vert V = v]$. The meta-estimator $\widehat\mu_{Y^0\vert V}$ of $\mu_{Y^0\vert V}$ is $\widehat\mu_{Y^1\vert V}$ with $A$ and $\widehat\pi$ replaced by $1 - A$ and $1 - \widehat\pi$, and with $\widehat\mu_0$ replaced by a first-stage estimator $\widehat\nu_0$ of $\nu_0(x) \coloneqq \E[\phi(Y)\vert X = x, A = 0]$. We randomly split the indices of $\gD_1$ into two sets $\mathcal{I}_1$ and $\mathcal{I}_2$ of size $n/2$. Let $\widehat\mu^{(1)}_{Y^1\vert V}$ and $\widehat\mu^{(1)}_{Y^0\vert V}$ be the meta-estimators with the second stage computed on $\mathcal{I}_1$, and $\widehat\mu^{(2)}_{Y^1\vert V}$ and $\widehat\mu^{(2)}_{Y^0\vert V}$ be those on $\mathcal{I}_2$. We then define:
\[ \widehat\Delta^{(1)}(v)\coloneqq\widehat\mu^{(1)}_{Y^1\vert V}(v)-\widehat\mu^{(1)}_{Y^0\vert V}(v), \qquad \widehat\Delta^{(2)}(v)\coloneqq\widehat\mu^{(2)}_{Y^1\vert V}(v)-\widehat\mu^{(2)}_{Y^0\vert V}(v). \]
To test the null hypothesis $H_0$ that $P(Y^1\vert V = v) = P(Y^0\vert V = v)$ for almost every $v$, we use the cross U-statistic of~\cite{KimRamdas2024}, which~\cite{Martinez-Taboada23} apply to the marginal distributions of $Y^1$ and $Y^0$:
\begin{equation}\label{eq:mmd_stat}
\widehat T(v) \coloneqq \inprod{\widehat\Delta^{(1)}(v)}{\widehat\Delta^{(2)}(v)}_{\gH_\gY}.
\end{equation}
This statistic estimates the squared \emph{maximum mean discrepancy} between $P(Y^1\vert V = v)$ and $P(Y^0\vert V = v)$:
\[ T(v) \coloneqq \mathrm{MMD}^2(P(Y^1\vert V = v),P(Y^0\vert V = v)) = \norm{\mu_{Y^1\vert V}(v) - \mu_{Y^0\vert V}(v)}^2_{\gH_\gY}.\]  For a characteristic kernel, $T(v) = 0$ if and only if $P(Y^1\vert V = v) = P(Y^0\vert V = v)$. The following corollary provides a bound on the estimation error of the statistic.
\begin{corollary}[Testing]\label{cor:test}
The statistic $\widehat T$ satisfies:
\begin{align*}
\E\left[\abs{\widehat T(V) - T(V)}\right] &\lesssim \E\left[\norm{\widehat\mu^{(1)}_{Y^1\vert V}(V) - \mu_{Y^1\vert V}(V)}^2_{\gH_\gY}\right]  + \E\left[\norm{\widehat\mu^{(1)}_{Y^0\vert V}(V) - \mu_{Y^0\vert V}(V)}^2_{\gH_\gY}\right] \\
&\quad + \E\left[T(V)\right]^{1/2}\biggl(\E\left[\norm{\widehat\mu^{(1)}_{Y^1\vert V}(V) - \mu_{Y^1\vert V}(V)}^2_{\gH_\gY}\right] \\
&\qquad \qquad \qquad \qquad + \E\left[\norm{\widehat\mu^{(1)}_{Y^0\vert V}(V) - \mu_{Y^0\vert V}(V)}^2_{\gH_\gY}\right]\biggr)^{1/2}.
\end{align*}
\end{corollary}
\begin{proof}
This follows from the following application of the Cauchy--Schwarz inequality:
\begin{align*}
\abs{\widehat T(v) - T(v)} &\le \norm{\widehat\Delta^{(1)}(v) - \mu_{Y^1\vert V}(v) + \mu_{Y^0\vert V}(v)}_{\gH_\gY}\norm{\widehat\Delta^{(2)}(v) - \mu_{Y^1\vert V}(v) + \mu_{Y^0\vert V}(v)}_{\gH_\gY} \\
&\quad + T(v)^{1/2}\norm{\widehat\Delta^{(1)}(v) - \mu_{Y^1\vert V}(v) + \mu_{Y^0\vert V}(v)}_{\gH_\gY} \\
&\quad + T(v)^{1/2}\norm{\widehat\Delta^{(2)}(v) - \mu_{Y^1\vert V}(v) + \mu_{Y^0\vert V}(v)}_{\gH_\gY}.
\end{align*}
It remains to take expectations and apply the Cauchy--Schwarz and triangle inequalities, since $(V, \widehat\Delta^{(1)}(V))$ and $(V, \widehat\Delta^{(2)}(V))$ have the same distribution.
\end{proof}

\subsection{Sampling}
Suppose that $\gY$ is compact and that $k_\gY$ is continuous. Kernel herding~\cite{Welling2009,Chen2010} with the target $\widehat\mu_{Y^1\vert V}(v)$ generates points $y_1, y_2, \ldots \in \gY$ by
\begin{equation}\label{eq:herding}
y_{t+1} \in \argmax_{y\in\gY}\left[\inprod{\widehat\mu_{Y^1\vert V}(v)}{\phi(y)}_{\gH_\gY} - \frac{1}{t+1}\sum_{s=1}^{t}k_\gY(y_s, y)\right], \qquad t = 0, 1, \ldots.
\end{equation}
When the target is the embedding of a distribution and lies in the relative interior of the convex hull $\gM \coloneqq \mathrm{conv}\{\phi(y) : y \in \gY\}$, \cite[Proposition~1]{Chen2010} shows that the error $\norm{\widehat\mu_{Y^1\vert V}(v) - \frac{1}{N}\sum_{t=1}^N\phi(y_t)}_{\gH_\gY}$ decays as $O(N^{-1})$. However, when $\gH_\gY$ is infinite-dimensional, as for the Gaussian kernel, the relative interior condition fails~\cite[Proposition~2]{Bach2012}. In fact, the target $\widehat\mu_{Y^1\vert V}(v)$ might not even lie in $\gM$; an example is our own CCME estimator, since the pseudo-outcome~\eqref{eq:pseudo_outcome} puts the negative weight $1 - 1/\widehat\pi(x)$ on $\widehat\mu_0(x)$ when $a = 1$. Here, we demonstrate that our meta-estimator bound makes it possible to bound the error of the herding points with respect to $P(Y^1\vert V = v)$ in this regime, albeit with a slower rate of $O(N^{-1/2})$.
\begin{corollary}[Sampling]\label{cor:herding}
Let $\gY$ be compact, and let $k_\gY$ be continuous with $\sup_{y\in\gY}k_\gY(y,y) \le B_k$. Then, for every $v \in \gV$ and every $N \ge 1$, the empirical distribution $\widehat P_N$ of the points $y_1, \ldots, y_N$ generated by~\eqref{eq:herding} satisfies:
\[
\mathrm{MMD}\left(\widehat P_N, P(Y^1\vert V = v)\right) \lesssim \norm{\widehat\mu_{Y^1\vert V}(v) - \mu_{Y^1\vert V}(v)}_{\gH_\gY} + \frac{\sqrt{B_k} + \norm{\widehat\mu_{Y^1\vert V}(v)}_{\gH_\gY}}{\sqrt{N}}.
\]
\end{corollary}
\begin{proof}
We fix $v$ and denote $\widehat\mu \coloneqq \widehat\mu_{Y^1\vert V}(v)$, $\mu \coloneqq \mu_{Y^1\vert V}(v)$, $\Delta \coloneqq \norm{\widehat\mu - \mu}_{\gH_\gY}$, and $R \coloneqq \sqrt{B_k} + \norm{\widehat\mu}_{\gH_\gY}$. Since $\frac{1}{N}\sum_{t=1}^N\phi(y_t)$ is the embedding of $\widehat P_N$, the triangle inequality yields:
\begin{equation}\label{eq:herd_split}
\mathrm{MMD}\big(\widehat P_N, P(Y^1\vert V = v)\big) \le \Delta + \biggl\lVert\widehat\mu - \frac{1}{N}\sum_{t=1}^N\phi(y_t)\biggr\rVert_{\gH_\gY}.
\end{equation}
To bound the second term on the right-hand side, we follow the idea from~\cite{Chen2010}. First, we define the weight vectors:
\[
w_t \coloneqq (t+1)\widehat\mu - \sum_{s=1}^t\phi(y_s), \qquad t \ge 0.
\]
Since $w_N - w_0 = N\widehat\mu - \sum_{t=1}^N\phi(y_t)$, the second term on the right-hand side of~\eqref{eq:herd_split} is at most $(\norm{w_N}_{\gH_\gY} + \norm{w_0}_{\gH_\gY})/N$. Since $w_0 = \widehat\mu$, it suffices to bound $\norm{w_N}_{\gH_\gY}$, which we do by induction on $t$ through a recursion that relates $\norm{w_{t+1}}_{\gH_\gY}$ to $\norm{w_t}_{\gH_\gY}$.

Since the objective in \eqref{eq:herding} is $\inprod{w_t}{\phi(y)}_{\gH_\gY}/(t+1)$, we have $y_{t+1} \in \argmax_{y\in\gY}\inprod{w_t}{\phi(y)}_{\gH_\gY}$. Denote $c_t \coloneqq \phi(y_{t+1}) - \widehat\mu$, which satisfies $w_{t+1} = w_t - c_t$ and $\norm{c_t}_{\gH_\gY} \le R$. Since $y_{t+1}$ maximizes $\inprod{w_t}{\phi(y)}_{\gH_\gY}$ over $\gY$ and $P(Y^1\vert V = v)$ is supported in $\gY$, we have:
\begin{align*}
\inprod{w_t}{c_t}_{\gH_\gY} &= \max_{y\in\gY}\inprod{w_t}{\phi(y)}_{\gH_\gY} - \inprod{w_t}{\widehat\mu}_{\gH_\gY} \\
&\ge \E_{Y\sim P(Y^1\vert V = v)}[\inprod{w_t}{\phi(Y)}_{\gH_\gY}] - \inprod{w_t}{\widehat\mu}_{\gH_\gY} \\
&= \inprod{w_t}{\mu - \widehat\mu}_{\gH_\gY} \ge -\Delta\norm{w_t}_{\gH_\gY}.
\end{align*}
Consequently, the norms of the weight vectors satisfy the recursion:
\[
\norm{w_{t+1}}^2_{\gH_\gY} = \norm{w_t}^2_{\gH_\gY} - 2\inprod{w_t}{c_t}_{\gH_\gY} + \norm{c_t}^2_{\gH_\gY} \le \left(\norm{w_t}_{\gH_\gY} + \Delta\right)^2 + R^2.
\]
We now show by induction that $\norm{w_t}_{\gH_\gY} \le \norm{w_0}_{\gH_\gY} + t\Delta + \sqrt{t}R$ for every $t \ge 0$. The case $t = 0$ is trivial. If the bound holds at $t$, then inserting this bound into the right-hand side of the last display yields:
\[
\norm{w_{t+1}}_{\gH_\gY} \le \left( \left(\norm{w_0}_{\gH_\gY} + (t+1)\Delta + \sqrt{t}R\right)^2 + R^2 \right)^{1/2} \le \norm{w_0}_{\gH_\gY} + (t+1)\Delta + \sqrt{t+1}R.
\]
We thus complete the induction step. In particular, with $t = N$, we obtain:
\[
\biggl\lVert\widehat\mu - \frac{1}{N}\sum_{t=1}^N\phi(y_t)\biggr\rVert_{\gH_\gY} \le \frac{\norm{w_N}_{\gH_\gY} + \norm{w_0}_{\gH_\gY}}{N} \le \Delta + \frac{2\norm{\widehat\mu}_{\gH_\gY}}{N} + \frac{R}{\sqrt{N}} \le \Delta + \frac{3R}{\sqrt{N}}.
\]
Inserting this bound into the right-hand side of~\eqref{eq:herd_split} yields the claim.
\end{proof}

\section{Experimental Details}\label{sec:exp_details}

We implement all matrix operations and automatic differentiation in \texttt{JAX} \cite{Bradbury2018}, all neural networks in \texttt{Equinox} \cite{Kidger2021}, and all propensity score models in \texttt{scikit-learn} \cite{Pedregosa2011}. 

\subsection{Computational Resources}
All experiments were conducted on an internal workstation equipped with an Intel Core i9-9900K CPU (3.60 GHz, 8 physical cores), 32 GB of RAM, and a single NVIDIA GeForce RTX 2080 Ti GPU (11 GB VRAM).

\subsection{Synthetic Data (Section~\ref{sec:exp1})}

We generate covariates $X \sim \mathcal{N}(\mathbf{1}_{10}, \mathbf{I}_{10}) \in \mathbb{R}^{10}$ with the conditioning variable $V = (X_1,X_2,X_3,X_4,X_5)$. The treatment assignment follows the propensity score: $P(A=1 \,|\, X=x) = \pi(x) = 0.1 + 0.8 \cdot \mathbb{I}(x_1 \in [0, 2] \text{ and } x_6 \geq 1.5)$. Then we generate $Y$ according to the equation:
\begin{align*}
    Y &= 1 + X^\top \beta + A \cdot (2 + X^\top \gamma) + \epsilon, \\
    \beta &\coloneqq (1, -0.5, 0.8, -0.7, 0.6, 1, 0.3, -0.2, 0.1, -0.3)^\top, \\
    \gamma &\coloneqq (0.8, 0, 0, 0.6, 0, 2.0, 0.4, 0, 0, 0.2)^\top.
\end{align*}
Here, $\epsilon$ is a heteroscedastic bimodal noise with the following generation process:
\begin{align*}
    \epsilon &= S + \epsilon', \\
    S/15.0 &\sim \text{Bernoulli}(\mathsf{sigmoid}(0.5 X_1)), \\
    \epsilon' &\sim \mathcal{N}\left(0, 0.25(1 + 0.5|X_1| + 0.3|X_5|)^2\right).
\end{align*}
The ground truth density $\fp^1(y|v)$ is computed analytically: since the covariates are independent, conditioning on $V = X_{1:5}$ leaves $X_{6:10} \sim \mathcal{N}(\mathbf{1}_5, \mathbf{I}_5)$ unchanged, so the conditional moments of $Y^1$ given $V = v$ are available in closed form.
 
For the qualitative assessment, we visualize 30 repeated runs of estimated counterfactual densities at $n = 10000$, evaluated at two covariate profiles: $v_1 \coloneqq (2.2, -0.2, 2.2, -0.2, 2.2)^\top$ and $v_2 \coloneqq (-0.2, 2.2, -0.2, 2.2, -0.2)^\top$. These induce different mixture probabilities $\mathsf{sigmoid}(0.5 X_1)$, which is 0.75 for $v_1$ and 0.48 for $v_2$, yielding distinct bimodal shapes.

\paragraph{Hyperparameters.} For the Deep Feature and Neural-Kernel estimators, the learning rates scale linearly with sample size: $\kappa \times n/100$, where $\kappa = 2 \times 10^{-4}$ (Deep Feature) and $\kappa = 4 \times 10^{-4}$ (Neural-Kernel). Each random forest consists of 100 trees and has a maximum depth of 4. The full configurations are given in Table~\ref{tab:exp1_config}. 

\begin{table}[h!]
    \centering
    \caption{Hyperparameter configurations for Section~\ref{sec:exp1}.}
    \label{tab:exp1_config}
    \vspace{0.2cm}
    \begin{tabular}{lccc}
        \toprule
        \textbf{Hyperparameter} & \textbf{Ridge Regression} & \textbf{Deep Feature} & \textbf{Neural-Kernel} \\
        \midrule
        \multicolumn{4}{l}{\textit{Model Parameters}} \\
        Kernel Bandwidth ($\sigma$) & 2.0 & 2.0 & 2.0 \\
        \begin{tabular}[t]{@{}l@{}}Regularization\\($n\lambda_1 = m\lambda_0$)\end{tabular} & 20.0 & 20.0 & N/A \\
        Grid/Feature Dimension ($M$) & N/A & 20 & 20 \\
        Network Architecture & N/A & \multicolumn{2}{c}{2 layers $\times$ 20 units, ReLU} \\
        \midrule
        \multicolumn{4}{l}{\textit{Training (SGD Optimizer)}} \\
        Learning Rate (Stage 1) & N/A & $2 \times 10^{-4} \times n/100$ & $4 \times 10^{-4} \times n/100$ \\
        Learning Rate (Stage 2) & N/A & $2 \times 10^{-4} \times n/100$ & $4 \times 10^{-4} \times n/100$ \\
		Momentum & N/A & 0.9 & 0.9 \\
        Epochs (Stage 1) & N/A & 6000 & 16000 \\
        Epochs (Stage 2) & N/A & 1000 & 500 \\
        Batch Size & Full Batch & Full Batch & Full Batch \\
        \bottomrule
    \end{tabular}
\end{table}

\paragraph{Evaluation.} We generate 10000 test points $X_{\text{test}} \sim \mathcal{N}(\mathbf{1}_{10}, \mathbf{I}_{10})$ and evaluate $\widehat{\fp}^1(y|v_{\text{test}})$ at $v_{\text{test}} \coloneqq \eta(X_{\text{test}})$ on 1000 grid points spanning $[\min(Y) - 2, \max(Y) + 2]$.

\subsection{MNIST Experiment (Section~\ref{sec:exp3})}

Images are normalized to $[0, 1]$. Propensity scores are clipped to $[0.01, 0.99]$. The propensity score is estimated by logistic regression. We also use early stopping by splitting the data into 80\% train and 20\% validation sets, with patience set to 10 epochs. The hyperparameters are given in Table~\ref{tab:exp3_config}.

\begin{table}[h!]
    \centering
    \caption{Hyperparameter configuration for the MNIST experiment (Section~\ref{sec:exp3}).}
    \label{tab:exp3_config}
    \vspace{0.2cm}
    \begin{tabular}{lc}
        \toprule
        \textbf{Hyperparameter} & \textbf{Deep Feature Estimator} \\
        \midrule
        Feature Dimension ($M$) & 1000 \\
        Network Architecture & 2 layers $\times$ 100 units, ReLU \\
        Kernel Bandwidth ($\sigma$) & 1.0 \\
		Regularization ($n\lambda_1 = m\lambda_0$) & 1.0 \\
        \midrule
        \multicolumn{2}{l}{\textit{Training (SGD Optimizer)}} \\
        Learning Rate (Stage 1) & $0.1$ \\
        Learning Rate (Stage 2) & $0.01$ \\
		Momentum & 0.9 \\
        Epochs (Stage 1) & 100 \\
        Epochs (Stage 2) & 100 \\
        Batch Size & 32 \\
		Patience for Early Stopping & 10 \\
        \bottomrule
    \end{tabular}
\end{table}

\section{Downstream Causal Tasks}\label{sec:downstream}

We apply the doubly robust CCME estimator to three downstream tasks: (i) conditional value-at-risk (CVaR) treatment effects, (ii) a conditional two-sample test, and (iii) risk-sensitive policy learning. 

\paragraph{Data generating process.} In Tasks 1 and 3, we generate the data from the model
\[
V\sim\mathrm{Unif}[-2,2],\quad U\sim \mathcal{N}(0,1),\quad X=(V,U),\quad A\mid X\sim\mathrm{Bernoulli}(\pi(X)),
\]
\[
Y^0=V+0.8U+\epsilon_0,\qquad Y^1=V+1.5+S\,\delta(V)+0.8U+\epsilon_1,\qquad \delta(v)\coloneqq2.2+0.9v,
\]
where $\pi(X)=\mathsf{sigmoid}(0.5V+1.1U)$ is truncated to $[0.05,0.95]$, $S\sim\mathrm{Unif}\{-1,+1\}$, $\epsilon_0,\epsilon_1\sim \mathcal{N}(0,0.7^2)$, all independent, and $Y=AY^1+(1-A)Y^0$. The confounder $U$ affects both $A$ and the potential outcomes but is not a component of $V$. Consequently, the One-Step estimator, which adjusts for $V$ only, is biased. Integrating out $U$, we obtain $P(Y^0\vert V=v)=\mathcal{N}(v, 0.8^2+0.7^2)$. In contrast, $P(Y^1\vert V=v)$ is the Gaussian mixture $\tfrac{1}{2}\mathcal{N}(v+1.5-\delta(v), 0.8^2+0.7^2)+\tfrac{1}{2}\mathcal{N}(v+1.5+\delta(v),0.8^2+0.7^2)$. Thus, $\E[Y^1-Y^0\vert V=v]=1.5$ for every $v$, while the difference between the two conditional distributions varies with $v$ through $\delta(v)$.

\paragraph{Estimators.} We implement the four estimators as follows.
\begin{itemize}
    \item CCME-DR: our proposed two-stage Ridge Regression estimator with a Gaussian kernel on $\gX$, $\gV$, and $\gY$, $n\lambda_1 = m\lambda_0 = 4$, a logistic regression estimate of the propensity score truncated to $[0.05, 0.95]$, and a split of the sample into two subsets.
    \item One-Step (CME/CoDiTE~\cite{Park2021}): the single-stage kernel ridge regression of $\phi(Y)$ on $V$, fitted on each of the treated and control groups, with the same kernel and regularization. This estimator uses neither sample splitting nor a propensity score.
    \item CPME~\cite{Zenati2025}: the augmented inverse propensity weighted embedding, with the same kernel, logistic regression, and regularization, and no sample split. This estimator fits the propensity score and the outcome model as in Stage 1 of CCME-DR. However, in contrast to the regression on $V$ in Stage 2, this estimator takes the average of the pseudo-outcomes \eqref{eq:pseudo_outcome} over the sample.
    \item CATE-DR: the linear regression on $V$ of the difference of the pseudo-outcomes \eqref{eq:pseudo_outcome} for $Y^1$ and $Y^0$ with $\phi(y)$ replaced by $y$, which estimates the conditional average treatment effect (CATE) $\E[Y^1 - Y^0 \mid V = v]$, with linear regressions of $Y$ on $X$ in the first stage, the same logistic regression, and a split of the sample into two subsets.
\end{itemize}
We estimate $\mu_{Y^0\vert V}$ with $A$ replaced by $1 - A$, as in Appendix~\ref{sec:applications}. In Tasks 1 and 3, we draw a sample of size $2000$ in each of $100$ replications and evaluate the estimators at $100$ equally spaced points $v \in [-2, 2]$.

\paragraph{Bandwidths.} In every replication of every task, we select the bandwidths by fitting on $80\%$ of the sample and minimizing an error on the held-out $20\%$, specifically the mean integrated squared error of the density of $Y$ given $V$ for the bandwidths on $\gV$ and $\gY$, and the error of the regression of $\phi(Y)$ on $X$ for the bandwidth on $\gX$. We use the whole sample for this selection in Tasks 1 and 3, and only $\gD_0$ in Task 2. All the kernel methods use the selected bandwidths.

\paragraph{Task 1: CVaR treatment effects.} For $a \in \{0,1\}$, we denote by $F_{Y^a\vert v}$ the cumulative distribution function of $P(Y^a\vert V=v)$. For $\alpha \in (0,1)$, the conditional CVaR
\[
\mathrm{CVaR}_\alpha(Y^a\vert v)\coloneqq\frac{1}{\alpha}\int_0^{\alpha}F^{-1}_{Y^a\vert v}(u)\,\mathrm{d}u
\]
is the mean of the lower $\alpha$-tail of $P(Y^a\vert V=v)$. The CVaR treatment effect is $\mathrm{CVaR}_\alpha(Y^1\vert v)-\mathrm{CVaR}_\alpha(Y^0\vert v)$. To estimate this effect, we evaluate the estimated density $\widehat{\fp}^a(y|v)\coloneqq\langle\widehat\mu_{Y^a\vert V}(v),\phi(y)\rangle_{\gH_\gY}$ at $1000$ equally spaced points $y \in [-20, 20]$. We set the negative values to zero and renormalize the result to integrate to one over the grid. Replacing $F_{Y^a\vert v}$ by the cumulative distribution function $\widehat F_{Y^a\vert v}$ of the renormalized density yields the plug-in estimators $\widehat{\mathrm{CVaR}}_\alpha(Y^a\vert v)$ and $\widehat{\mathrm{CVaR}}_\alpha(Y^1\vert v)-\widehat{\mathrm{CVaR}}_\alpha(Y^0\vert v)$.

We compare the estimates of the $\mathrm{CVaR}_{0.1}$ treatment effect with its true values in Figure~\ref{fig:bench_est}(a). As we can see, the median CCME-DR estimates are close to the true values across all $v$. Because of the confounding by $U$, the median One-Step estimates are larger than the true values at every $v$. Since the CPME embedding does not depend on $v$, the CPME estimates are constant across all values of $v$. We also report the mean absolute errors (MAEs) in Figure~\ref{fig:bench_est}(b). For the CVaR treatment effect, the MAE of CCME-DR is substantially smaller than those of One-Step and CPME. This is not surprising, as One-Step only adjusts for $V$ and not $U$. For the CATE, the MAEs of CCME-DR, CPME, and CATE-DR are all significantly smaller than that of One-Step, which reflects how these three estimators handle the confounder $U$ through the pseudo-outcomes.

\begin{figure}[h!]
    \centering
    \includegraphics[width=\textwidth]{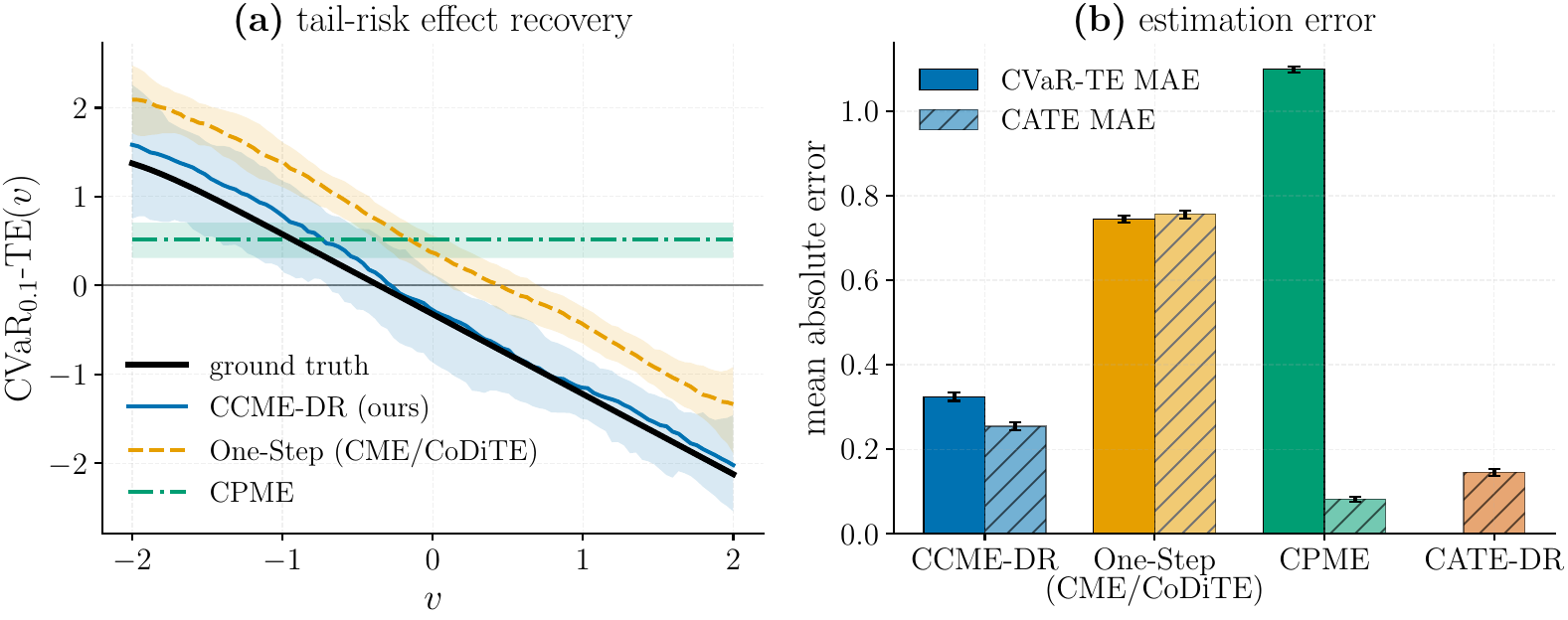}
    \caption{CVaR treatment effects (Task 1). (a) The estimated $\mathrm{CVaR}_{0.1}$ treatment effect against $v$: median over $100$ replications. The shaded regions show pointwise intervals between the $10$th and $90$th percentiles. (b) MAEs of the $\mathrm{CVaR}_{0.1}$ treatment effect and the CATE, with standard errors over the replications. CVaR-TE denotes the CVaR treatment effect. Since CATE-DR estimates only the CATE, this estimator has no CVaR-TE bar.}
    \label{fig:bench_est}
\end{figure}

\paragraph{Task 2: conditional two-sample test.} We test the null hypothesis $H_0$ of Appendix~\ref{sec:applications} with the statistic $\widehat T$ as defined in \eqref{eq:mmd_stat}. With $\widehat{\gV}$ as the set of $100$ equally spaced points in $[-2, 2]$, we denote the average of $\widehat T$ over $\widehat \gV$ by $\overline T\coloneqq\frac{1}{|\widehat{\gV}|}\sum_{v\in\widehat{\gV}}\widehat T(v)$.

For the Ridge Regression estimator, $\widehat\Delta^{(1)}(v)$ has the form
\begin{align*}
\widehat\Delta^{(1)}(v)&\coloneqq\sum_{i\in\mathcal{I}_1}\beta^{(1)}_i(v)\,\widehat{\gC}(Z_{1i}),\qquad \widehat{\gC}(z)\coloneqq\widehat\xi^{\,1}(z)-\widehat\xi^{\,0}(z),\\
\bigl(\beta^{(1)}_i(v)\bigr)_{i\in\mathcal{I}_1}&\coloneqq\bigl(\rmK^{(1)}_V+\tfrac{n}{2}\lambda_1\mathbf{I}_{n/2}\bigr)^{-1}\rmk^{(1)}_V(v),
\end{align*}
where $\rmK^{(1)}_V$ and $\rmk^{(1)}_V(v)$ are $\rmK_V$ and $\rmk_V(v)$ as in \eqref{eq:emp_ccme_form}, computed on $\mathcal{I}_1$, and $\widehat\xi^{\,1}$ and $\widehat\xi^{\,0}$ are the pseudo-outcomes \eqref{eq:pseudo_outcome} for $Y^1$ and $Y^0$, with the nuisance functions estimated on $\gD_0$. The difference $\widehat\Delta^{(2)}(v)$ has the same form with $\mathcal{I}_2$. Thus, we can write: 
\[ \overline T=\sum_{i\in\mathcal{I}_1}f_i, \qquad \text{where }  f_i\coloneqq\frac{1}{|\widehat{\gV}|}\sum_{v\in\widehat{\gV}}\beta^{(1)}_i(v)\inprod{\widehat{\gC}(Z_{1i})}{\widehat\Delta^{(2)}(v)}_{\gH_\gY}.\] 
The statistics for One-Step, CPME, and CATE-DR are calculated similarly with their own versions of $\beta^{(1)}_i(v)$ and $\widehat{\gC}$. Let $\bar f$ and $s_f$ be the mean and the standard deviation of $(f_i)_{i\in\mathcal{I}_1}$. Our test statistic is $T^\dagger\coloneqq\sqrt{n/2}\,\bar f/s_f$, and we reject the null hypothesis at level $\alpha = 0.05$ when $T^\dagger$ is at least the $0.95$ quantile of the standard normal distribution.

We compute the CCME-DR, One-Step, and CPME statistics with the Ridge Regression estimator, $\frac{n}{2}\lambda_1 = m\lambda_0 = 1$ with $n = 1000$. We consider four scenarios, which have the same distributions of $V$, $U$, $X$, $\pi(X)$, $A$, and $S$ as Tasks 1 and 3 and differ only in the potential outcomes:
\begin{align*}
&\text{Scenario 1:}\quad Y^0=V+0.8U+\epsilon_0,\qquad Y^1=V+0.8U+\epsilon_1,\\
&\text{Scenario 2:}\quad Y^0=V-0.8U+\epsilon_0,\qquad Y^1=V+0.8U+\epsilon_1,\\
&\text{Scenario 3:}\quad Y^0=0.8U+3S\,w(V)+\epsilon_0,\qquad Y^1=0.8U+3S\,(1-w(V))+\epsilon_1,\\
&\text{Scenario 4:}\quad Y^0=V+0.8U+\epsilon_0,\qquad Y^1=V+S\,\delta(V)+0.8U+\epsilon_1,
\end{align*}
where $w(v)\coloneqq\mathsf{sigmoid}(2v)$. The noise terms $\epsilon_0,\epsilon_1$ follow $\mathcal{N}(0,0.7^2)$ in Scenarios 1, 2, and 4 and $\mathcal{N}(0,0.6^2)$ in Scenario 3. In all four scenarios, $\E[Y^1-Y^0\vert V]=0$. The null hypothesis $H_0$ holds in Scenarios 1 and 2. In Scenario 2, $P(Y^1\vert V) = P(Y^0\vert V)$, while $P(Y^1\vert X) \neq P(Y^0\vert X)$. In Scenario 3, $P(Y^1\vert V) \neq P(Y^0\vert V)$, while $P(Y^1) = P(Y^0)$. In Scenario 4, $P(Y^1) \neq P(Y^0)$ and $P(Y^1\vert V) \neq P(Y^0\vert V)$.

We report the rejection rates in Table~\ref{tab:bench_test}. In Scenarios 1 and 2, the type I error rates of the CCME-DR, CPME, and CATE-DR tests are close to the nominal level. In Scenarios 3 and 4, the CCME-DR and One-Step tests reject $H_0$ in every replication. However, because of the confounding by $U$, the type I error rate of the One-Step test in Scenarios 1 and 2 is significantly larger than the nominal level. Since the CPME statistic compares the marginal distributions of $Y^1$ and $Y^0$, which coincide in Scenario 3, the CPME test has trivial power in this scenario. Similarly, since $\E[Y^1 - Y^0\vert V] = 0$ in Scenarios 3 and 4, the CATE-DR test has almost no power in these two scenarios. Thus, the CCME-DR test is the only one of the four tests whose type I error rate is close to the nominal level in Scenarios 1 and 2 and that rejects $H_0$ in every replication of Scenarios 3 and 4.

\begin{table}[h!]
    \centering
    \caption{Rejection rates of the conditional two-sample tests at level $\alpha = 0.05$ (Task 2), with standard errors, over $2000$ replications in each scenario, with sample size $2000$. The target level is $\alpha$ in Scenarios 1 and 2, where $H_0$ holds, and $1$ in Scenarios 3 and 4, where $H_0$ fails. In each scenario, the bold entries are those whose $95\%$ confidence intervals cover the target level.}
    \label{tab:bench_test}
    \vspace{0.2cm}
    \setlength{\tabcolsep}{5pt}
    \begin{tabular}{lcccc}
        \toprule
        \textbf{Test} & \textbf{Scenario 1} & \textbf{Scenario 2} & \textbf{Scenario 3} & \textbf{Scenario 4} \\
        \midrule
        CCME-DR & $\bm{0.049 \pm 0.005}$ & $\bm{0.052 \pm 0.005}$ & $\bm{1.000 \pm 0.000}$ & $\bm{1.000 \pm 0.000}$ \\
        One-Step \cite{Park2021} & $0.999 \pm 0.001$ & $0.140 \pm 0.008$ & $\bm{1.000 \pm 0.000}$ & $\bm{1.000 \pm 0.000}$ \\
        CPME \cite{Zenati2025} & $\bm{0.048 \pm 0.005}$ & $\bm{0.056 \pm 0.005}$ & $0.040 \pm 0.004$ & $\bm{1.000 \pm 0.000}$ \\
        CATE-DR & $\bm{0.053 \pm 0.005}$ & $\bm{0.048 \pm 0.005}$ & $0.051 \pm 0.005$ & $0.041 \pm 0.004$ \\
        \bottomrule
    \end{tabular}
\end{table}

\paragraph{Task 3: risk-sensitive policy learning.} A policy is a map $d:\gV\to\{0,1\}$. The tail-sensitive value of a policy $d$ is
\[
\mathrm{Val}_\alpha(d)\coloneqq\E\left[\mathrm{CVaR}_\alpha\left(Y^{d(V)}\vert V\right)\right].
\]
The policy $d^\star_\alpha(v)\coloneqq\mathbb{I}(\mathrm{CVaR}_\alpha(Y^1\vert v)>\mathrm{CVaR}_\alpha(Y^0\vert v))$ maximizes $\mathrm{Val}_\alpha$. This optimal policy depends on the lower tails of the conditional counterfactual distributions, which the CATE does not determine. For each of CCME-DR, One-Step, and CPME, we use the plug-in policy $\widehat d_\alpha(v)\coloneqq\mathbb{I}(\widehat{\mathrm{CVaR}}_\alpha(Y^1\vert v)>\widehat{\mathrm{CVaR}}_\alpha(Y^0\vert v))$, with the two CVaRs computed from the estimated densities as in Task 1. The CATE-DR policy is $\widehat d(v)\coloneqq\mathbb{I}(\widehat{\E}[Y^1-Y^0\vert V=v]>0)$ with the CATE estimated by CATE-DR. This policy does not depend on $\alpha$. We report the regret $\mathrm{Val}_\alpha(d^\star_\alpha)-\mathrm{Val}_\alpha(\widehat d_\alpha)$ (and $\mathrm{Val}_\alpha(d^\star_\alpha)-\mathrm{Val}_\alpha(\widehat d)$ for the CATE-DR policy), evaluated under the true distributions, for $\alpha \in \{0.05, 0.1, 0.15, 0.2, 0.3, 0.4\}$ in Figure~\ref{fig:bench_policy}. The figure shows that, across the values of $\alpha$, the median regret of CCME-DR is close to zero and an order of magnitude smaller than that of One-Step. In contrast, the CPME and CATE-DR policies treat every unit. As a result, they have the largest regret, which increases as $\alpha$ decreases.

\begin{figure}[h!]
    \centering
    \includegraphics[width=0.62\textwidth]{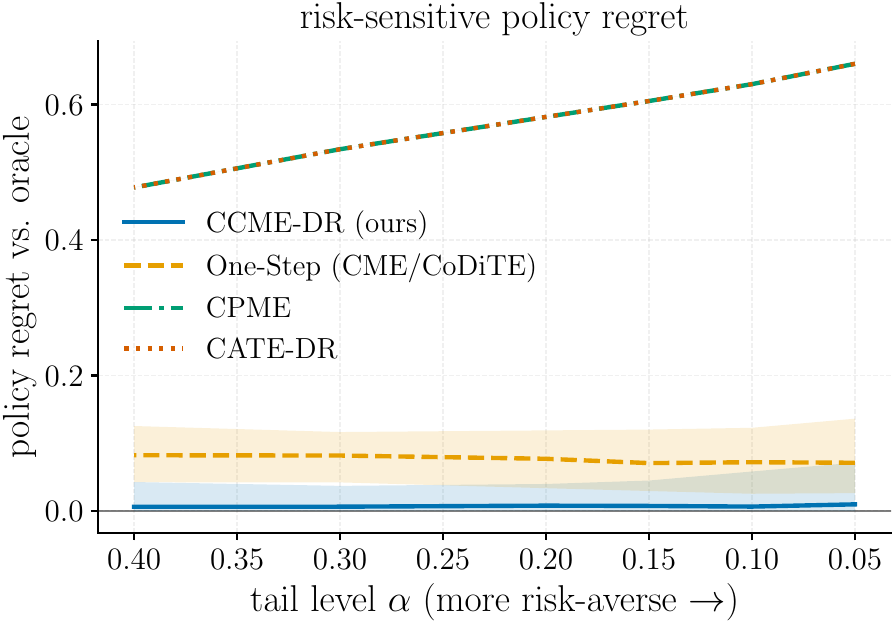}
    \caption{Regret of the estimated risk-sensitive policies relative to the optimal policy $d^\star_\alpha$ (Task 3), as a function of the tail level $\alpha$, which decreases to the right. The lines are the medians over $100$ replications. The shaded regions are pointwise intervals between the $10$th and $90$th percentiles.}
    \label{fig:bench_policy}
\end{figure}

\section{Background on Function Spaces and RKHS}
	\label{app:sobolev}

	We provide a brief review of the Bochner integral, (fractional) Sobolev spaces, and reproducing kernel Hilbert spaces (RKHS). A comprehensive treatment of these topics can be found in \cite[Chapters 1 and 2]{Hytonen2016}.

\subsection{Function Spaces and Fourier Analysis}

    In this subsection, let $\Omega \subseteq \R^d$ denote a generic domain (which will correspond to either $\gY$ or $\gV$ in the main paper).
    
	\begin{definition}[Lebesgue Space $L^{p}(\Omega)$]
		A function $f: \Omega \to \R$ is said to be in the \emph{Lebesgue space} $L^{p}(\Omega)$ for $1 \le p \le \infty$ if it is measurable and its $L^{p}$ norm is finite. That is,
		\begin{align*}
			\norm{f}_{L^p(\Omega)}\coloneqq \left( \int_{\Omega}|f(x)|^{p}\,\mathrm{d}x \right)^{1/p} < \infty, \intertext{and for $p=\infty$,}\norm{f}_{L^{\infty}(\Omega)}\coloneqq \inf \left\{ r \ge 0 : \mathfrak{m}(\{ |f(x)| > r \}) = 0 \right\} < \infty,
		\end{align*}
		where $\mathfrak{m}$ denotes the Lebesgue measure on $\R^{d}$.
	\end{definition}

	\begin{definition}[Weak Derivatives]
		Let $f$ be a function in $L^{1}(\Omega)$. A function $g \in L^{1}(\Omega)$ is the \emph{weak derivative} of $f$ with respect to $x_{i}$ if for all smooth test functions with compact support $\zeta \in C^{\infty}_{c}(\Omega)$, the following identity holds:
		\[
			\int_{\Omega}f(x) \frac{\partial \zeta(x)}{\partial x_{i}}\,\mathrm{d}x = - \int_{\Omega}g
			(x) \zeta(x)\,\mathrm{d}x.
		\]
		For a multi-index $\alpha = (\alpha_{1}, \dots, \alpha_{d}) \in \mathbb{N}_0^{d}$, the $\alpha^{\text{th}}$\emph{-weak derivative}, denoted by $D^{\alpha}f$, must satisfy the following identity for all $\zeta \in C^{\infty}_{c}(\Omega)$:
		\[
			\int_{\Omega}f(x) (D^{\alpha}\zeta(x))\,\mathrm{d}x = (-1)^{|\alpha|}\int_{\Omega}(D^{\alpha}
			f(x)) \zeta(x)\,\mathrm{d}x.
		\]
	\end{definition}

	\begin{definition}[Sobolev Space $W^{k,p}(\Omega)$]\label{def:sobolev}
		For $k \in \mathbb{N}_0$ and $p \in [1,\infty]$, the \emph{Sobolev space} $W^{k,p} (\Omega)$ is the set of all functions $f \in L^{p}(\Omega)$ such that for every multi-index $\alpha \in \mathbb{N}_0^{d}$ with total order $\abs{\alpha} \le k$, the weak derivative $D^{\alpha}f$ exists and is also in $L^{p}(\Omega)$.

		The norm for this space is defined as:
		\[
			\norm{f}_{W^{k,p}(\Omega)}\coloneqq \sum_{\abs{\alpha} \le k}\norm{D^\alpha f}
			_{L^p(\Omega)}.
		\]
	\end{definition}

	\begin{definition}[Fourier Transform]
		Let $f \in L^{1}(\R^{d})$ be an integrable function. The Fourier transform of $f$, denoted $\widehat{f}$ or $\mathcal{F}[f]$, is defined by the integral:
		\begin{equation*}
			\widehat{f}(\omega) \coloneqq \frac{1}{(2\pi)^{d/2}}\int_{\R^{d}}e^{-i \omega^\top x} f(x)
			\,\mathrm{d}x,
		\end{equation*}
		where $\omega \in \R^{d}$ is the frequency variable.
	\end{definition}
	
	One important property of the Fourier transform is that it is a linear isometry on $L^2(\R^{d})$ as the following theorem shows:

	\begin{theorem}[Plancherel's Theorem]
		The Fourier transform, initially defined on $L^{1}(\R^{d}) \cap L^{2}(\R^{d})$, extends uniquely to a bounded linear operator on $L^{2}(\R^{d})$. Moreover, the operator is unitary, meaning that for any $f \in L^2(\R^{d})$:
		\[
			\norm{f}_{L^2(\R^{d})}^{2} = \int_{\R^{d}}|f(x)|^{2}\,\mathrm{d}x = \int_{\R^{d}}|\widehat{f}(\omega)|^{2}\,\mathrm{d}\omega = \norm{\widehat{f}}_{L^2(\R^{d})}^{2}.
		\]
	\end{theorem}

	While $W^{k,p}$ spaces are defined for integer $k$, it is often necessary to consider spaces of functions with fractional smoothness $s \in \R$. The Fourier transform provides a natural characterization of these spaces when $p=2$.

	\begin{definition}[Fractional Sobolev Space $H^{s}(\R^{d})$]
		For any $s \in \R$, the fractional Sobolev space $H^{s}(\R^{d})$ is defined as the space of tempered distributions $f$ such that:
		\[
			\norm{f}_{H^s(\R^{d})}^{2} \coloneqq \int_{\R^{d}} (1 + \norm{\omega}^{2})^{s} |\widehat{f}(\omega)|^{2}\,\mathrm{d}\omega < \infty.
		\]
		When $s=k$ is a non-negative integer, this space is isomorphic to the Sobolev space $W^{k,2}(\R^{d})$ equipped with the standard norm.
	\end{definition}

	For a domain $\Omega \subset \R^{d}$ and a non-integer $s > 0$, we define $W^{s,2}(\Omega) \coloneqq \{F\vert_\Omega : F \in H^{s}(\R^{d})\}$ with the norm $\norm{f}_{W^{s,2}(\Omega)} \coloneqq \inf\{\norm{F}_{H^s(\R^{d})} : F \in H^{s}(\R^{d}),\ F\vert_\Omega = f\}$.

	A crucial property of Sobolev spaces is the embedding theorem, which relates regularity to integrability; it allows us to embed a space with higher derivatives into a space with higher integrability. To state the theorem, we introduce two notions of smooth domains.

	\begin{definition}[Lipschitz Boundary] \label{def:lipschitz_boundary}
    The domain $\Omega \subset \R^{d}$ is said to have a \emph{Lipschitz boundary} if it is a non-empty, bounded set and its boundary $\partial \Omega$ satisfies the following property: For every point $x \in \partial \Omega$, there exists a radius $r > 0$ and an orthogonal coordinate system $(y_1, \dots, y_d)$ (which may depend on $x$) such that if we define the cylinder $C \coloneqq \{y : \norm{(y_1, \dots, y_{d-1})} < r, |y_d| < r\}$, then:
    \begin{equation*}
        \Omega \cap C = \{y \in C : y_d < \gamma(y_1, \dots, y_{d-1}) \},
    \end{equation*}
    where $\gamma: \R^{d-1} \to \R$ is a Lipschitz continuous function.
\end{definition}

For an integer $s > d/2$ and a domain $\Omega$ with a Lipschitz boundary, the space $\{F\vert_\Omega : F \in H^{s}(\R^{d})\}$ with the infimum norm coincides with the space $W^{s,2}(\Omega)$ of Definition~\ref{def:sobolev}, with an equivalent norm~\cite[Section~10.7]{Wendland2004}.

\begin{definition}[Uniform Interior Cone Condition] \label{def:cone_condition}
    The domain $\Omega$ is said to satisfy a \emph{uniform interior cone condition} if there exist an angle $\theta \in (0, \pi/2)$ and a radius $\delta > 0$ such that for every $x \in \bar{\Omega}$, there exists a unit vector $\xi \in \R^{d}$ such that the cone defined by:
    \begin{equation*}
        \mathfrak{C}(x, \xi, \theta, \delta) \coloneqq \left\{ x + z : z \in \R^{d}, \|z\| < \delta, z^\top \xi \ge \|z\| \cos \theta \right\}
    \end{equation*}
    is contained in $\Omega$.
\end{definition}

	\begin{theorem}[Sobolev Embedding]
		Suppose that either $\Omega = \R^d$ or $\Omega$ is a bounded set with a Lipschitz boundary in $\R^d$. Let $(s,p), (t,q) \in \mathbb{N}_0 \times [1,\infty) \cup  [0,\infty) \times \{ 2 \}$ with $p \le q$. If $s - \frac{d}{p} = t - \frac{d}{q}$, then we have the continuous embedding:
				\[
					W^{s,p}(\Omega) \hookrightarrow W^{t,q}(\Omega).
				\]
	\end{theorem}

	\begin{remark}[Boundary Regularity]
		The embeddings stated above do not require any boundary condition when $\Omega = \R^{d}$. However, if the domain is restricted to a proper open subset $\Omega \subset \R^{d}$, these embeddings require $\Omega$ to satisfy specific regularity conditions. Typically, $\Omega$ must have a \emph{Lipschitz boundary} (or satisfy the strong cone property) to ensure the existence of a bounded extension operator from $W^{s,p}(\Omega)$ to $W^{s,p}(\R^{d})$.
	\end{remark}

	\subsection{RKHS and Bochner's Theorem}\label{sec:rkhs}

	\begin{definition}[Reproducing Kernel Hilbert Space]
		A \emph{Reproducing Kernel Hilbert Space} (RKHS) $\gH_{\gY} \subseteq (\gY \to \R)$ is a Hilbert space in which the point evaluation functional is a bounded linear operator. That is, for every $y \in \gY$, there exists a constant $C_{y} > 0$ such that $|f(y)| \le C_{y}\norm{f}_{\gH_\gY}$ for all $f \in \gH_{\gY}$.
	\end{definition}

	By the Riesz representation theorem, the boundedness of the evaluation functional implies the existence of a unique function $k: \gY \times \gY \to \R$, called the \emph{reproducing kernel}. For every $y \in \gY$, the function $k(\cdot, y)$ belongs to $\gH_{\gY}$. We define the \emph{canonical feature map} $\phi: \gY \to \gH_{\gY}$ as $\phi(y) \coloneqq k(\cdot, y)$.
	
	For all $f \in \gH_{\gY}$ and $y \in \gY$, the function value is given by the inner product:
		\[
			f(y) = \langle f, \phi(y) \rangle_{\gH_\gY} = \langle f, k(\cdot, y) \rangle_{\gH_\gY}.
		\]
	\begin{definition}[Positive Definite Kernel]
		A function $k: \R^{d} \times \R^{d} \to \mathbb{C}$ is called a \emph{positive definite kernel} if for any integer $N \ge 1$, any set of points $\{y_{1}, \dots, y_{N}\} \subset \R^{d}$, and any coefficients $\{c_{1}, \dots, c_{N}\} \subset \mathbb{C}$, the following inequality holds:
		\[
			\sum_{i=1}^{N} \sum_{j=1}^{N} c_{i} \overline{c_{j}} k(y_{i}, y_{j}) \ge 0.
		\]
	\end{definition}
	We call $k$ \emph{strictly positive definite} if, in addition, the double sum is positive whenever the points $y_1, \ldots, y_N$ are distinct and $(c_1, \ldots, c_N) \neq 0$.

	For any finite set of distinct points $\{y_{1}, \dots, y_{M}\} \subset \gY$, the \emph{Gram matrix} (or kernel matrix) $\rmK _M\in \R^{M \times M}$ with entries $(\rmK_M)_{ij} = k(y_{i}, y_{j}) = \langle \phi(y_{i}), \phi(y_{j}) \rangle_{\gH_\gY}$ is invertible when $k$ is strictly positive definite, as stated in the following proposition:

	\begin{proposition}[Independence of Feature Maps]\label{prop:feature_indep}
		If the kernel $k$ is strictly positive definite, then for any set of distinct points $\{y_{1}, \dots, y_{M}\} \subset \gY$, the Gram matrix $\rmK_M$ is invertible and the set of feature vectors $\{\phi(y_{1}), \dots, \phi(y_{M})\}$ is linearly independent in $\gH_{\gY}$.
	\end{proposition}

	A function $f: \gV \to \gH_{\gY}$ is called a \emph{simple function} if it can be written as $f(v) = \sum_{i=1}^{N}x_{i}\mathbb{I}_{E_i}(v)$, where $\mathbb{I}_{E_i}(\cdot)$ is the indicator function and $\{E_{i}\} _{i=1}^{N}$ are disjoint, Lebesgue measurable subsets of $\gV$, and $\{x_{i}\}_{i=1} ^{N}$ are vectors in $\gH_{\gY}$. The integral of $f(v)$ is given by $\int_{\gV}f(v)\,\mathrm{d}v \coloneqq \sum_{i=1}^{N}x_{i}\mathfrak{m}(E_{i})$. Note that the result of this integral is a single vector in the Hilbert space $\gH_{\gY}$.

	\begin{definition}[Bochner Integral]\label{def:bochner}
		A function $h: \gV \to \gH_{\gY}$ is said to be \emph{Bochner integrable} if it satisfies two conditions:
		\begin{enumerate}
			\item There exists a sequence of simple functions $\{f_{n}\}_{n=1}^{\infty}$ such that $\lim_{n\to\infty}f_{n}(v) = h(v)$ for almost every $v \in \gV$. In this case, we say that $h$ is \emph{Bochner measurable}.

			\item The real-valued function $v \mapsto \norm{h(v)}_{\gH_\gY}$ is Lebesgue integrable, i.e.,
				\[
					\int_{\gV}\norm{h(v)}_{\gH_\gY}\,\mathrm{d}v < \infty.
				\]
		\end{enumerate}
		If a function $h$ is Bochner integrable, its \emph{Bochner integral} is defined as the limit of the integrals of its approximating sequence of simple functions:
		\[
			\int_{\gV}h(v)\,\mathrm{d}v \coloneqq \lim_{n\to\infty}\int_{\gV}f_{n}(v)\,\mathrm{d}v.
		\]
		The completeness of the Hilbert space $\gH_{\gY}$ guarantees that this limit exists and is unique (up to sets of Lebesgue measure zero), regardless of the choice of the approximating sequence $\{f_{n}\}$.
	\end{definition}

	Of particular interest in this work are kernels that are translation-invariant. Such kernels can be written as $k(y,y') = \varphi(y-y')$ for some continuous function $\varphi: \R^{d} \to \mathbb{C}$. In this setting, Bochner's theorem provides a connection between positive definite functions and Fourier analysis.

	\begin{theorem}[Bochner's Theorem]
		A continuous function $\varphi: \R^{d}\to \mathbb{C}$ is positive definite (i.e., it defines a positive definite kernel $k(y,y')=\varphi(y-y')$) if and only if it is the Fourier transform of a non-negative, finite Borel measure on $\R^{d}$.
	\end{theorem}

	In our case, we assume this measure has a density, which we call $\widehat{\varphi}(\omega )$. This means $\varphi(z) = \mathcal{F}^{-1}[\widehat{\varphi}](z)$, and we require $\widehat {\varphi}(\omega) \ge 0$ for all $\omega$. This spectral perspective allows us to express the RKHS norm via an integral in the frequency domain.

	\begin{proposition}[RKHS Norm in the Fourier Domain]\label{prop:kernelinv}
		For a translation-invariant kernel $k(y,y')=\varphi(y-y')$ on $\R^{d_y}$ with the RKHS $\gH_{\R^{d_y}}$, the RKHS norm of a function $f \in \gH_{\R^{d_y}}$ is given by:
		\[
			\norm{f}_{\gH_{\R^{d_y}}}^{2}= \int_{\R^{d_y}}\frac{|\widehat{f}(\omega)|^{2}}{\widehat{\varphi}(\omega)}
			\,\mathrm{d}\omega.
		\]
		For a set $\gY \subset \R^{d_y}$, the RKHS $\gH_\gY$ of the restricted kernel satisfies $\gH_\gY = \{F\vert_\gY : F \in \gH_{\R^{d_y}}\}$ and $\norm{f}_{\gH_\gY} = \inf\{\norm{F}_{\gH_{\R^{d_y}}} : F \in \gH_{\R^{d_y}},\ F\vert_\gY = f\}$ for every $f \in \gH_\gY$.
	\end{proposition}

\subsection{\texorpdfstring{Operators on Hilbert Spaces}{Operators on Hilbert Spaces}}\label{sec:operators}
Let $\gU$ and $\gW$ be separable Hilbert spaces, and let $(e_k)_{k\ge1}$ be an orthonormal basis of $\gU$. For a bounded linear operator $T: \gU \to \gW$, we denote by $\norm{T}_{\text{op}}$ and $\norm{T}_{\text{HS}}$ the operator norm and the Hilbert-Schmidt norm of $T$:
\[
\norm{T}_{\text{op}} \coloneqq \sup_{\norm{x}_{\gU} \le 1}\norm{Tx}_{\gW}, \qquad \norm{T}^2_{\text{HS}} \coloneqq \sum_{k\ge1}\norm{Te_k}^2_{\gW}.
\]
The Hilbert-Schmidt inner product of two such operators $S$ and $T$ is $\inprod{S}{T}_{\text{HS}} \coloneqq \sum_{k\ge1}\inprod{Se_k}{Te_k}_{\gW}$. For vectors in $\R^M$, $\norm{\cdot}$ is the Euclidean norm. For matrices, $\norm{\cdot}_{\text{op}}$ is the spectral norm. Moreover, the Hilbert-Schmidt norm and inner product of matrices are the Frobenius norm $\norm{\cdot}_F$ and the Frobenius inner product. We denote by $T^*$ the adjoint of $T$: the operator from $\gW$ to $\gU$ with $\inprod{Tx}{y}_{\gW} = \inprod{x}{T^*y}_{\gU}$ for all $x \in \gU$ and $y \in \gW$. An operator $T$ on $\gU$ is self-adjoint if $T^* = T$ and positive if $\inprod{Tx}{x}_{\gU} \ge 0$ for all $x \in \gU$. For a positive self-adjoint operator $T$ on $\gU$, we denote by $\Tr(T)$ the trace of $T$: $\Tr(T) \coloneqq \sum_{k\ge1}\inprod{Te_k}{e_k}_{\gU}$. These sums do not depend on the basis. For $a \in \gW$ and $b \in \gU$, we denote by $a \otimes b$ the operator $x \mapsto \inprod{b}{x}_{\gU}a$ from $\gU$ to $\gW$, which satisfies $\norm{a \otimes b}_{\text{HS}} = \norm{a}_{\gW}\norm{b}_{\gU}$ and $\inprod{T}{a \otimes b}_{\text{HS}} = \inprod{Tb}{a}_{\gW}$.

\section{Proof of the Identification of CCME}\label{sec:iden_proof}

In this section, we provide a proof of the identification of CCME presented in Proposition~\ref{prop:identification}. We begin by formalizing the definition of $V$.

Let $(\Omega, \mathcal{F}_Z, \mathbb{P})$ be the probability space on which the random variables $Z = (X, A, Y)$ are defined. Let $(\gX, \Sigma_\gX)$ and $(\gV, \Sigma_\gV)$ denote the measurable spaces associated with $X$ and $V$, respectively. We assume that the mapping $\eta: \gX \to \gV$ defining $V = \eta(X)$ is $(\Sigma_\gX, \Sigma_\gV)$-measurable. Consequently, the $\sigma$-algebra generated by $V$, denoted $\sigma(V) \coloneqq \{ X^{-1}(\eta^{-1}(E)) : E \in \Sigma_\gV \}$, satisfies $\sigma(V) \subseteq \sigma(X) \subseteq \mathcal{F}_Z$.

\begin{proof}[Proof of Proposition~\ref{prop:identification}]

We first verify that the random variable $\xi(Z)$ is Bochner integrable by showing it satisfies the two conditions of Definition~\ref{def:bochner}.

We begin with the feature map $\phi(Y) = k_\gY(\cdot, Y)$. Since $\gY$ is a subset of a Euclidean space and the kernel $k_\gY$ is continuous, $\phi(Y)$ is continuous as a mapping from $\gY$ to $\gH_\gY$. Continuous functions on separable domains can be approximated by simple functions (e.g., piecewise constant functions on a grid); thus, $\phi(Y)$ is Bochner measurable.

Next, consider the term $\mu_{0}(X)$. By Assumption~\ref{ass:boundk}, the kernel is bounded: $\sup_{y\in\gY} k_\gY(y,y) \le B_k$ for some constant $B_k > 0$. Since $\phi(Y)$ is Bochner measurable and bounded, it is Bochner integrable. Since the conditional expectation of a Bochner integrable random variable is defined to be a Bochner measurable function \cite[Definition 2.5]{Park2020}, $\mu_{0}(X)$ is Bochner measurable. Since $\xi(Z)$ is a linear combination of $\phi(Y)$ and $\mu_{0}(X)$, it is also Bochner measurable.

We verify that $\E[\norm{\xi(Z)}_{\gH_\gY}] < \infty$. By the triangle inequality:
\[
    \norm{\xi(Z)}_{\gH_\gY} \le \frac{|A|}{\pi(X)} \norm{\phi(Y)}_{\gH_\gY} + \left|1 - \frac{A}{\pi(X)}\right| \norm{\mu_{0}(X)}_{\gH_\gY}.
\]
By Jensen's inequality for Bochner integrals, the norm of the conditional mean is bounded by the conditional expectation of the norm:
\[
    \norm{\mu_{0}(X)}_{\gH_\gY} = \norm{\E[\phi(Y)|X, A=1]}_{\gH_\gY} \le \E[\norm{\phi(Y)}_{\gH_\gY}|X, A=1] \le B_k^{1/2}.
\]
Invoking the strong positivity assumption ($\pi(X) \ge \varepsilon > 0$), we verify the bound almost surely:
\[
    \norm{\xi(Z)}_{\gH_\gY} \le \frac{1}{\varepsilon}B_k^{1/2} + \left(1 + \frac{1}{\varepsilon}\right)B_k^{1/2} = B_k^{1/2}(2\varepsilon^{-1} + 1) < \infty.
\]
Since the norm is uniformly bounded, its expectation is finite. Thus, $\xi(Z)$ is Bochner integrable.

    We now prove \eqref{eq:identification}. By the law of iterated expectations, since $\sigma(V) \subseteq \sigma(X)$, we have:
    \begin{equation}
        \label{eq:tower_prop}
        \E[\xi(Z) \vert V] = \E\big[ \E[\xi(Z) \vert X] \mid V \big].
    \end{equation}
    We first compute the inner conditional expectation $\E[\xi(Z) \vert X]$. Since $\pi(X)$ and $\mu_{0}(X)$ are $\sigma(X)$-measurable, they can be factored out:
    \begin{align}
        \E[\xi(Z) \vert X] &= \E\left[ \frac{A}{\pi(X)}\phi(Y) - \frac{A}{\pi(X)}\mu_{0}(X) + \mu_{0}(X) \;\middle|\; X \right] \nonumber \\
        &= \frac{1}{\pi(X)} \E[A \phi(Y) \vert X] - \frac{\mu_{0}(X)}{\pi(X)} \E[A \vert X] + \mu_{0}(X). \label{eq:inner_exp}
    \end{align}
    Recall that $\E[A \vert X] = P(A=1 \vert X) = \pi(X)$. For the term $\E[A \phi(Y) \vert X]$, we apply the conditional ignorability assumption ($Y^1 \perp A \vert X$):
    \begin{align*}
        \E[A \phi(Y) \vert X] &= \E[A \phi(Y^1) \vert X] \\
        &= \E[\phi(Y^1) \vert X, A=1] P(A=1 \vert X) \\
        &= \E[\phi(Y^1) \vert X] \pi(X).
    \end{align*}
    We also note that $\mu_{0}(X) = \E[\phi(Y) \vert X, A=1] = \E[\phi(Y^1) \vert X, A=1] = \E[\phi(Y^1) \vert X]$. Substituting these results back into \eqref{eq:inner_exp} yields:
    \begin{align*}
        \E[\xi(Z) \vert X] &= \frac{1}{\pi(X)} \left( \E[\phi(Y^1) \vert X] \pi(X) \right) - \frac{\E[\phi(Y^1) \vert X]}{\pi(X)} \pi(X) + \E[\phi(Y^1) \vert X] \\
        &= \E[\phi(Y^1) \vert X] - \E[\phi(Y^1) \vert X] + \E[\phi(Y^1) \vert X] \\
        &= \E[\phi(Y^1) \vert X].
    \end{align*}
    Finally, substituting this result into \eqref{eq:tower_prop} and using the law of iterated expectations again yields:
    \begin{align*}
        \E[\xi(Z) \vert V] &= \E\big[ \E[\phi(Y^1) \vert X] \mid V \big] \\
        &= \E[\phi(Y^1) \vert V] \\
        &= \mu_{Y^1|V}(V).
    \end{align*}
    This concludes the proof.
\end{proof}

\section{{Proof of the Upper Bound for the Meta-Estimator (Theorem~\ref{thm:ECME})}}\label{sec:meta_proof}

For brevity, we denote by $\widetilde\mu_{Y^1\vert V}(v)$ the conditional mean embedding of the pseudo-outcome, i.e., $\widetilde\mu_{Y^1\vert V}(v) \coloneqq \E[\widehat{\xi}(Z)\vert V=v, \gD_{0}]$. We split the error into two terms:
\begin{equation}
	\label{eq:total}
	\begin{split}
		\E\left[ \norm{\widehat\mu_{Y^1\vert V}(V) - \mu_{Y^1\vert V}(V)}^{2}_{\gH_\gY}\right]&\le 2 \E\left[ \norm{\widehat\mu_{Y^1\vert V}(V) - \widetilde\mu_{Y^1\vert V}(V)}^{2}_{\gH_\gY}\right] \\
		&\quad + 2 \E\left[ \norm{\widetilde\mu_{Y^1\vert V}(V) - \mu_{Y^1\vert V}(V)}^{2}_{\gH_\gY}\right] \\
		&= 2 \gR^{2}_{\widehat\xi}(\widehat\mu_{Y^1\vert V}) + 2\E\left[ \norm{\widetilde\mu_{Y^1\vert V}(V) - \mu_{Y^1\vert V}(V)}^{2}_{\gH_\gY}\right].
	\end{split}
\end{equation}
To bound the second term on the right-hand side, we use the following lemma, the proof of which is provided in Appendix~\ref{sec:misc_proofs}:

\begin{lemma}[Nuisance Error Bound]\label{lem:nuisance_bdd}
Let Assumptions~\ref{ass:unconfound} and~\ref{ass:boundk} hold. In addition, assume that $\widehat\pi(x), \pi(x) \in [\varepsilon, 1)$ for all $x \in \gX$ and $\sup_{x\in\gX}\norm{\widehat\mu_{0}(x)}_{\gH_{\gY}} \le B_{\mu_0}$ for some constant $B_{\mu_0} > 0$. Then we have the following bound:
\begin{equation}\label{eq:nuisance_bdd}
 \E\left[\norm{\E[\widehat{\xi}(Z)\vert V, \mathcal{D}_0] - \mu_{Y^1\vert V}(V)}^{2}_{\gH_{\gY}}\right] \lesssim \min\left\{\gR^{2}_{\pi}(\widehat\pi), \gR^{2}_{\mu_{0}}(\widehat\mu_{0})\right\}.
\end{equation}
\end{lemma}

With this lemma, we insert \eqref{eq:nuisance_bdd} back in the right-hand side of \eqref{eq:total} to finish the proof. \hfill $\blacksquare$

\section{{Proof of the Upper Bound for the Ridge Regression Estimator (Theorem~\ref{thm:plugin_rate})}}\label{sec:rr_proof}

\subsection{\texorpdfstring{Notations and Setup}{Notations and Setup}}\label{sec:rr_setup}

We denote $\lambda \coloneqq \lambda_1$, $\mu \coloneqq \mu_{Y^1\vert V}$, $\widetilde\mu_{Y^1\vert V}(v) \coloneqq \E[\widehat\xi(Z)\vert V = v, \gD_0]$, and $\norm{h}^2_{L_2(P_V)} \coloneqq \int_\gV \norm{h(v)}^2_{\gH_\gY}\,p_V(v)\,\mathrm{d}v$ for $h \in (\gV \to \gH_\gY)$, where $P_V$ is the distribution of $V$ and $p_V$ is its density in Assumption~\ref{ass:domain}.

For $u = (u_1, \ldots, u_n) \in \gH_\gY^n$, we denote:
\begin{equation}\label{eq:rr_Fu_def}
\widehat F[u](v) \coloneqq \sum_{i=1}^n \beta_i(v)u_i, \qquad \bm{\beta}(v) \coloneqq (\rmK_V + n\lambda\mathbf{I}_n)^{-1}\rmk_V(v), \qquad v \in \gV.
\end{equation}
By~\eqref{eq:emp_ccme_form}, we have $\widehat\mu_{\mathrm{RR}} = \widehat F[u]$ with $u_i = \widehat\xi(Z_{1i})$.

\subsection{\texorpdfstring{Preliminary Lemma}{Preliminary Lemma}}\label{sec:rr_lemma}

The following lemma provides a bound on the norm of the pseudo-outcome:

\begin{lemma}\label{lem:HS_bound}
Suppose that $\sup_{y \in \gY} k_\gY(y,y) \le B_{k}$, that $\widehat{\pi}(x) \ge \varepsilon > 0$ for all $x \in \gX$, and that $\sup_{x \in \gX}\norm{\widehat\mu_{0}(x)}_{\gH_\gY} \le B_{\mu_0}$. Then the pseudo-outcome satisfies $\norm{\widehat\xi(z)}_{\gH_\gY} \le \sqrt{B_{\widehat\xi}}$ for all $z$, where $B_{\widehat\xi} \coloneqq (\sqrt{B_k} + B_{\mu_0})^2/\varepsilon^2$.
\end{lemma}

\begin{proof}
We only write the computation for $z = Z_{1i}$, since the computation is the same for every $z$. Recall that
\[
\widehat{\xi}(Z_{1i}) = \frac{A_{1i}}{\widehat{\pi}(X_{1i})}\phi(Y_{1i}) + \left(1 - \frac{A_{1i}}{\widehat{\pi}(X_{1i})}\right)\widehat\mu_{0}(X_{1i}).
\]
We have:
\begin{align*}
\langle \widehat{\xi}(Z_{1i}), \widehat{\xi}(Z_{1i}) \rangle_{\gH_\gY} &= \frac{A_{1i}^2}{\widehat{\pi}(X_{1i})^2} k_\gY(Y_{1i}, Y_{1i}) \\
&\quad + 2\frac{A_{1i}}{\widehat{\pi}(X_{1i})}\left(1 - \frac{A_{1i}}{\widehat{\pi}(X_{1i})}\right) \langle \phi(Y_{1i}), \widehat\mu_{0}(X_{1i}) \rangle_{\gH_\gY} \\
&\quad + \left(1 - \frac{A_{1i}}{\widehat{\pi}(X_{1i})}\right)^2 \norm{\widehat\mu_{0}(X_{1i})}^2_{\gH_\gY}.
\end{align*}
By assumption, $\norm{\widehat\mu_{0}(X_{1i})}_{\gH_\gY} \le B_{\mu_0}$. Since $A_{1i} \in \{0,1\}$ and $\widehat{\pi}(X_{1i}) \ge \varepsilon$, the coefficients $A_{1i}/\widehat\pi(X_{1i})$ and $1 - A_{1i}/\widehat\pi(X_{1i})$ are at most $1/\varepsilon$ in absolute value. Using $k_\gY(Y_{1i}, Y_{1i}) \le B_k$ and the Cauchy-Schwarz inequality for the cross term, we obtain:
\begin{align*}
\langle \widehat{\xi}(Z_{1i}), \widehat{\xi}(Z_{1i}) \rangle_{\gH_\gY} &\le \frac{B_k}{\varepsilon^2} + \frac{2\sqrt{B_k} B_{\mu_0}}{\varepsilon^2} + \frac{B_{\mu_0}^2}{\varepsilon^2} = \frac{(\sqrt{B_k} + B_{\mu_0})^2}{\varepsilon^2} \eqqcolon B_{\widehat\xi}.
\end{align*}
\end{proof}

\subsection{\texorpdfstring{Proof of the Bound}{Proof of the Bound}}\label{sec:rr_bound}

We prove the bound in several steps. In Step~1, we bound $\norm{\widehat F[u]}^2_{L_2(P_V)}$ in~\eqref{eq:rr_Fu_def} for all $u \in \gH^n_\gY$. In Step~2, we split the error of $\widehat\mu_{\mathrm{RR}}$ into four terms, which we bound in Steps~3--6.

Before Step~1, we introduce several notations and derive a couple of bounds that we will use throughout the proof. By Lemma~\ref{lem:HS_bound}, Jensen's inequality and $k_\gY(y, y) \le B_k$, we have, for every $z$ and $v \in \gV$:
\begin{equation}\label{eq:rr_norms}
\norm{\widehat\xi(z)}_{\gH_\gY} \le \sqrt{B_{\widehat\xi}}, \qquad \norm{\widetilde\mu_{Y^1\vert V}(v)}_{\gH_\gY} \le \sqrt{B_{\widehat\xi}}, \qquad \norm{\mu(v)}_{\gH_\gY} \le \sqrt{B_k}.
\end{equation}
Denote $b \coloneqq \widetilde\mu_{Y^1\vert V} - \mu$. Let $(V, Z)$ be a new copy of $(V_{11}, Z_{11})$, independent of $\gD_0$ and $\gD_1$, and denote $U \coloneqq \widehat\xi(Z) - b(V)$. On the sample, we denote $U_i \coloneqq \widehat\xi(Z_{1i}) - b(V_{1i})$ for $i = 1, \ldots, n$. Since $\widehat\xi$ and $b$ depend only on $\gD_0$, the pairs $(V_{1i}, U_i)$ are, conditionally on $\gD_0$, independent copies of $(V, U)$. Moreover, $\E[U\vert V, \gD_0] = \widetilde\mu_{Y^1\vert V}(V) - b(V) = \mu(V)$. By~\eqref{eq:rr_norms}, we have:
\begin{equation}\label{eq:rr_U}
\begin{gathered}
\sup_{v\in\gV}\norm{b(v)}_{\gH_\gY} \le \sqrt{B_{\widehat\xi}} + \sqrt{B_k}, \qquad \norm{U}_{\gH_\gY} \le B_U \coloneqq 2\sqrt{B_{\widehat\xi}} + \sqrt{B_k}, \\
\norm{U - \mu(V)}_{\gH_\gY} \le 2B_U.
\end{gathered}
\end{equation}
\paragraph{Step~1: A Bound on $\norm{\widehat F[u]}^2_{L_2(P_V)}$.} Since $k_\gV$ is translation invariant, $k_\gV(v,v) = \varphi(0) \eqqcolon \kappa^2$ for every $v \in \gV$. We denote by $\mathrm{Id}$ the identity operator on $\gH_\gV$, by $C$ the covariance operator of $k_\gV$ and by $\widehat C$ its empirical version:
\[
C \coloneqq \E[k_\gV(\cdot, V) \otimes k_\gV(\cdot, V)], \qquad \widehat C \coloneqq \frac{1}{n}\sum_{i=1}^n k_\gV(\cdot, V_{1i}) \otimes k_\gV(\cdot, V_{1i}).
\]
Since $\Tr(C) = k_\gV(v,v) = \kappa^2$, we have $\norm{C}_{\text{op}} \le \kappa^2$. We define the linear maps $S_n: \gH_\gV \to \R^n$ and $T_u: \R^n \to \gH_\gY$ by:
\[
S_n(f) \coloneqq n^{-1/2}(f(V_{11}), \ldots, f(V_{1n})), \quad f \in \gH_\gV; \qquad T_u(w) \coloneqq n^{-1/2}\sum_{i=1}^n w_iu_i,\quad w \in \R^n.
\]
These maps satisfy:
\[
S_n^*S_n = \widehat C, \qquad S_nS_n^* = \rmK_V/n, \qquad \norm{T_u}^2_{\text{HS}} = \frac{1}{n}\sum_{i=1}^n\norm{u_i}^2_{\gH_\gY}.
\]
Since $S_n(\widehat C + \lambda\,\mathrm{Id})^{-1} = (S_nS_n^* + \lambda\mathbf{I}_n)^{-1}S_n$, the definition~\eqref{eq:rr_Fu_def} yields $\widehat F[u](v) = A_uk_\gV(\cdot, v)$ with $A_u \coloneqq T_uS_n(\widehat C + \lambda\,\mathrm{Id})^{-1}$. Consequently, we have:
\begin{equation}\label{eq:rr_Fu_L2}
\begin{split}
\norm{\widehat F[u]}^2_{L_2(P_V)} &= \int_\gV\norm{A_uk_\gV(\cdot, v)}^2_{\gH_\gY}\,p_V(v)\,\mathrm{d}v = \norm{A_uC^{1/2}}^2_{\text{HS}} \\
&\le \norm{T_u}^2_{\text{HS}}\norm{S_n(\widehat C + \lambda\,\mathrm{Id})^{-1}C^{1/2}}^2_{\text{op}}.
\end{split}
\end{equation}
The spectral theorem yields the following operator norm bounds:
\begin{equation}\label{eq:rr_spec}
\norm{S_n(\widehat C + \lambda\,\mathrm{Id})^{-1/2}}_{\text{op}} \le 1, \quad \norm{S_n(\widehat C + \lambda\,\mathrm{Id})^{-1}}_{\text{op}} \le \frac{1}{2\sqrt{\lambda}}, \quad \norm{(C + \lambda\,\mathrm{Id})^{-1/2}C^{1/2}}_{\text{op}} \le 1.
\end{equation}
Thus, we obtain, for every $u \in \gH^n_\gY$:
\begin{equation}\label{eq:rr_Fu_unif}
\norm{\widehat F[u]}^2_{L_2(P_V)} \overset{\text{(a)}}{\le} \norm{T_u}^2_{\text{HS}}\norm{S_n(\widehat C + \lambda\,\mathrm{Id})^{-1}}^2_{\text{op}}\norm{C}_{\text{op}} \overset{\text{(b)}}{\le} \frac{\kappa^2}{4\lambda}\cdot\frac{1}{n}\sum_{i=1}^n\norm{u_i}^2_{\gH_\gY},
\end{equation}
where (a) follows from~\eqref{eq:rr_Fu_L2} and (b) follows from the formula for $\norm{T_u}^2_{\text{HS}}$, the second bound in~\eqref{eq:rr_spec} and $\norm{C}_{\text{op}} \le \kappa^2$.

\paragraph{Step~2: The Decomposition of the Error.}
We denote $\widehat F[U] \coloneqq \widehat F[u]$ with $u_i = U_i$ and $\widehat F[b] \coloneqq \widehat F[u]$ with $u_i = b(V_{1i})$. Since $\widehat\mu_{\mathrm{RR}} = \widehat F[u]$ with $u_i = \widehat\xi(Z_{1i}) = U_i + b(V_{1i})$ and $\widehat F[u]$ is linear in $u$, we have $\widehat\mu_{\mathrm{RR}} = \widehat F[U] + \widehat F[b]$, which yields:
\begin{equation}\label{eq:rr_split}
\norm{\widehat\mu_{\mathrm{RR}} - \mu}^2_{L_2(P_V)} \le 2\norm{\widehat F[U] - \mu}^2_{L_2(P_V)} + 2\norm{\widehat F[b]}^2_{L_2(P_V)}.
\end{equation}
We denote $R \coloneqq \norm{\widehat F[U] - \mu}^2_{L_2(P_V)}$. By~\eqref{eq:rr_Fu_unif} with $u_i = U_i$, \eqref{eq:rr_U} and~\eqref{eq:rr_norms}, we have:
\begin{equation}\label{eq:rr_Rmax}
R \le 2\norm{\widehat F[U]}^2_{L_2(P_V)} + 2\norm{\mu}^2_{L_2(P_V)} \le R_{\max} \coloneqq \frac{\kappa^2B_U^2}{2\lambda} + 2B_k.
\end{equation}
We denote $\Delta \coloneqq (C + \lambda\,\mathrm{Id})^{-1/2}(C - \widehat C)(C + \lambda\,\mathrm{Id})^{-1/2}$ and $E \coloneqq \{\norm{\Delta}_{\text{op}} \le 2/3\}$. Since $\Delta$ depends only on $V_{11}, \ldots, V_{1n}$, which are independent of $\gD_0$, the event $E$ is independent of $\gD_0$.

We write $T_U$ for $T_u$ with $u_i = U_i$. We denote:
\[
A_\lambda \coloneqq \E[T_US_n\vert\gD_0]\big(\E[\widehat C\vert\gD_0] + \lambda\,\mathrm{Id}\big)^{-1}, \qquad F_\lambda(v) \coloneqq A_\lambda k_\gV(\cdot, v) \quad \text{for } v \in \gV.
\]
By the definitions of $T_u$ and $S_n$, the identity $\E[U\vert V, \gD_0] = \mu(V)$ and the fact that the pairs $(V_{1i}, U_i)$ are copies of $(V, U)$ given $\gD_0$, we have:
\begin{equation}\label{eq:rr_TUSn}
T_US_n = \frac{1}{n}\sum_{i=1}^n U_i \otimes k_\gV(\cdot, V_{1i}), \qquad \E[T_US_n\vert\gD_0] = \E[\mu(V) \otimes k_\gV(\cdot, V)], \qquad \E[\widehat C\vert\gD_0] = C.
\end{equation}
On $E$, the triangle inequality in $L_2(P_V)$ yields $R \le 2\norm{F_\lambda - \mu}^2_{L_2(P_V)} + 2\norm{\widehat F[U] - F_\lambda}^2_{L_2(P_V)}$. On $E^c$, we use~\eqref{eq:rr_Rmax}. By~\eqref{eq:rr_TUSn}, $F_\lambda$ does not depend on the data. Thus, we obtain:
\begin{equation}\label{eq:rr_decomp}
\E[R\vert\gD_0] \le 2\norm{F_\lambda - \mu}^2_{L_2(P_V)} + 2\E\left[\norm{\widehat F[U] - F_\lambda}^2_{L_2(P_V)}\mathds{1}_E\,\middle\vert\,\gD_0\right] + R_{\max}\,\mathbb{P}(E^c\vert\gD_0).
\end{equation}
We refer to $\norm{F_\lambda - \mu}^2_{L_2(P_V)}$ as the \emph{approximation error}, $\E[\norm{\widehat F[U] - F_\lambda}^2_{L_2(P_V)}\mathds{1}_E\vert\gD_0]$ as the \emph{statistical error} on $E$, and $R_{\max}\,\mathbb{P}(E^c\vert\gD_0)$ as the \emph{remainder}. The three terms on the right-hand side of~\eqref{eq:rr_decomp}, together with $\E[\norm{\widehat F[b]}^2_{L_2(P_V)}\vert\gD_0]$, are the four terms that we bound in Steps~3--6.

\paragraph{Step~3: The Approximation Error.}
Let $(e_k)_{k\ge1}$ be an orthonormal basis of $\gH_\gY$, and denote $c_k(v) \coloneqq \inprod{\mu(v)}{e_k}_{\gH_\gY} = \int_\gY e_k(y)\fp^1(y\vert v)\,\mathrm{d}y$ for $v \in \gV$. Lemma~\ref{lem:sobolevcoef} with the span of $e_1, \ldots, e_M$ yields $\sum_{k=1}^M \norm{c_k}^2_{W^{r,2}(\gV)} \le B_c$ for every $M$, with a constant $B_c$ that does not depend on $M$ and is finite by Assumption~\ref{ass:smooth_density}. Letting $M \to \infty$ yields:
\begin{equation}\label{eq:rr_Bc}
\sum_{k\ge1}\norm{c_k}^2_{W^{r,2}(\gV)} \le B_c.
\end{equation}
By~\eqref{eq:rr_TUSn}, the adjoint of $A_\lambda = \E[\mu(V) \otimes k_\gV(\cdot, V)](C + \lambda\,\mathrm{Id})^{-1}$ is $A_\lambda^* = (C + \lambda\,\mathrm{Id})^{-1}\E[k_\gV(\cdot, V) \otimes \mu(V)]$, which yields, for every $k$:
\begin{equation}\label{eq:rr_adj}
A_\lambda^*e_k = (C + \lambda\,\mathrm{Id})^{-1}\E[c_k(V)k_\gV(\cdot, V)].
\end{equation}
According to~\cite[Lemma~14]{Fischer2020}, for the function $A_\lambda^*e_k$ in~\eqref{eq:rr_adj}, the bound~\eqref{eq:rr_Bc} with $r \le \tau$ yields, for every $k$:
\begin{equation}\label{eq:rr_fs_bias}
\E\left[\left((A_\lambda^*e_k)(V) - c_k(V)\right)^2\right] \lesssim \lambda^{r/\tau}\norm{c_k}^2_{W^{r,2}(\gV)}.
\end{equation}
Since $\inprod{F_\lambda(v)}{e_k}_{\gH_\gY} = \inprod{k_\gV(\cdot, v)}{A_\lambda^*e_k}_{\gH_\gV} = (A_\lambda^*e_k)(v)$ for every $k$ and $v \in \gV$, it follows from the Parseval identity that:
\begin{align}
\norm{F_\lambda - \mu}^2_{L_2(P_V)} &= \sum_{k\ge1}\E\left[\left(\inprod{F_\lambda(V)}{e_k}_{\gH_\gY} - c_k(V)\right)^2\right] \notag \\
&=\sum_{k\ge1}\E\left[\left((A_\lambda^*e_k)(V) - c_k(V)\right)^2\right] \notag \\
&\overset{\text{(a)}}{\lesssim} \lambda^{r/\tau}\sum_{k\ge1}\norm{c_k}^2_{W^{r,2}(\gV)} \overset{\text{(b)}}{\le} B_c\lambda^{r/\tau}, \label{eq:rr_bias}
\end{align}
where (a) follows from~\eqref{eq:rr_fs_bias} and (b) follows from~\eqref{eq:rr_Bc}. Since $\lambda \asymp n^{-2\tau/(2r + d_v)}$, we have $\lambda^{r/\tau} \asymp n^{-2r/(2r + d_v)}$. Thus, the approximation error satisfies:
\begin{equation}\label{eq:rr_approx}
\norm{F_\lambda - \mu}^2_{L_2(P_V)} \lesssim n^{-2r/(2r + d_v)}.
\end{equation}

\paragraph{Step~4: The Statistical Error on $E$.}
We denote:
\[
\zeta_i \coloneqq (U_i - F_\lambda(V_{1i})) \otimes k_\gV(\cdot, V_{1i}) \quad \text{for } i = 1, \ldots, n, \qquad \Xi \coloneqq \frac{1}{n}\sum_{i=1}^n\zeta_i - \E[\zeta_1\vert\gD_0].
\]
We first show that $\norm{\widehat F[U] - F_\lambda}^2_{L_2(P_V)} \le 9\norm{\Xi(C + \lambda\,\mathrm{Id})^{-1/2}}^2_{\text{HS}}$ on $E$. Then we bound $\E[\norm{\Xi(C + \lambda\,\mathrm{Id})^{-1/2}}^2_{\text{HS}}\vert\gD_0]$.

By~\eqref{eq:rr_TUSn}, the identity $F_\lambda(V_{1i}) \otimes k_\gV(\cdot, V_{1i}) = A_\lambda(k_\gV(\cdot, V_{1i}) \otimes k_\gV(\cdot, V_{1i}))$ and the definition of $A_\lambda$, we have:
\begin{align*}
\frac{1}{n}\sum_{i=1}^n\zeta_i &= T_US_n - A_\lambda\widehat C, \\
\E[\zeta_1\vert\gD_0] &= \E[T_US_n\vert\gD_0] - A_\lambda C = A_\lambda(C + \lambda\,\mathrm{Id}) - A_\lambda C = \lambda A_\lambda.
\end{align*}
Hence, $\Xi = T_US_n - A_\lambda(\widehat C + \lambda\,\mathrm{Id})$. We write $A_U$ for $A_u$ with $u_i = U_i$, so that $\widehat F[U](v) = A_Uk_\gV(\cdot, v)$ with $A_U = T_US_n(\widehat C + \lambda\,\mathrm{Id})^{-1}$. Thus, $A_U - A_\lambda = \Xi(\widehat C + \lambda\,\mathrm{Id})^{-1}$ and $\widehat F[U](v) - F_\lambda(v) = (A_U - A_\lambda)k_\gV(\cdot, v)$. On $E$, the spectrum of the self-adjoint operator $(C + \lambda\,\mathrm{Id})^{-1/2}(\widehat C + \lambda\,\mathrm{Id})(C + \lambda\,\mathrm{Id})^{-1/2} = \mathrm{Id} - \Delta$ lies in $[1/3, 5/3]$, which yields:
\begin{equation}\label{eq:rr_onE}
\norm{(\widehat C + \lambda\,\mathrm{Id})^{-1/2}(C + \lambda\,\mathrm{Id})^{1/2}}^2_{\text{op}} = \norm{(\mathrm{Id} - \Delta)^{-1}}_{\text{op}} \le 3.
\end{equation}
Consequently, we obtain, on $E$:
\begin{align}
&\norm{\widehat F[U] - F_\lambda}^2_{L_2(P_V)} \notag \\
&\quad \overset{\text{(a)}}{=} \norm{\Xi(\widehat C + \lambda\,\mathrm{Id})^{-1}C^{1/2}}^2_{\text{HS}} \notag \\
&\quad = \norm{\Xi(C + \lambda\,\mathrm{Id})^{-1/2}\,(C + \lambda\,\mathrm{Id})^{1/2}(\widehat C + \lambda\,\mathrm{Id})^{-1}(C + \lambda\,\mathrm{Id})^{1/2}\,(C + \lambda\,\mathrm{Id})^{-1/2}C^{1/2}}^2_{\text{HS}} \notag \\
&\quad \overset{\text{(b)}}{\le} \norm{\Xi(C + \lambda\,\mathrm{Id})^{-1/2}}^2_{\text{HS}}\norm{(\widehat C + \lambda\,\mathrm{Id})^{-1/2}(C + \lambda\,\mathrm{Id})^{1/2}}^4_{\text{op}}\norm{(C + \lambda\,\mathrm{Id})^{-1/2}C^{1/2}}^2_{\text{op}} \notag \\
&\quad \overset{\text{(c)}}{\le} 9\norm{\Xi(C + \lambda\,\mathrm{Id})^{-1/2}}^2_{\text{HS}}, \label{eq:rr_varE}
\end{align}
where (a) follows from the equalities in~\eqref{eq:rr_Fu_L2} with $A_u$ replaced by $A_U - A_\lambda = \Xi(\widehat C + \lambda\,\mathrm{Id})^{-1}$; (b) follows from $\norm{(C + \lambda\,\mathrm{Id})^{1/2}(\widehat C + \lambda\,\mathrm{Id})^{-1}(C + \lambda\,\mathrm{Id})^{1/2}}_{\text{op}} = \norm{(\widehat C + \lambda\,\mathrm{Id})^{-1/2}(C + \lambda\,\mathrm{Id})^{1/2}}^2_{\text{op}}$; and (c) follows from~\eqref{eq:rr_onE} and the third bound in~\eqref{eq:rr_spec}.

We now bound $\E[\norm{\Xi(C + \lambda\,\mathrm{Id})^{-1/2}}^2_{\text{HS}}\vert\gD_0]$. Let $\gamma \coloneqq \min\{1, (d_v + r)/(2\tau)\}$. Since $\tau \ge r > d_v/2$, $\gamma$ satisfies $d_v/(2\tau) < \gamma \le 1$. We use two facts from~\cite[Section~4]{Fischer2020}:\footnote{In~\cite[Section~4]{Fischer2020}, the facts hold for domains with $C^\infty$ boundaries and densities that are bounded above and below. Here, $p_V$ is bounded above and below by Assumption~\ref{ass:domain} and the assumption $\inf_{v\in\gV}p_V(v) > 0$ of the theorem. We can extend these facts to the Lipschitz boundary of $\gV$ via Stein's extension operator, which maps $W^{k,2}(\gV)$ continuously into $W^{k,2}(\R^{d_v})$ for every integer $k \ge 0$~\cite[Chapter~VI, Theorem~5]{Stein1970}. We then extend this operator to every real order $t > 0$ by the standard interpolation argument between the integer orders (see, e.g.,~\cite[Lemma~4.2]{ChandlerWilde2015}). Thus, we can extend the facts to Sobolev spaces with fractional orders.}
\begin{enumerate}[(i)]
\item the $i$-th largest eigenvalue of $C$ is $O(i^{-2\tau/d_v})$;
\item $W^{\gamma\tau,2}(\gV) \hookrightarrow L^\infty(\gV)$.
\end{enumerate}
According to~\cite[Lemmas~11 and~13, Theorem~9]{Fischer2020}, (i) and~(ii) yield, respectively, the following bounds for $P_V$-almost every $v$, with a constant $\kappa_\gamma > 0$ that depends only on $k_\gV$, $p_V$ and $\gamma$:
\begin{equation}\label{eq:rr_fs}
\Tr(C(C + \lambda\,\mathrm{Id})^{-1}) \lesssim \lambda^{-d_v/(2\tau)}, \qquad \norm{(C + \lambda\,\mathrm{Id})^{-1/2}k_\gV(\cdot, v)}^2_{\gH_\gV} \le \kappa_\gamma^2\lambda^{-\gamma}.
\end{equation}
Since $\E[k_\gV(\cdot, V) \otimes k_\gV(\cdot, V)] = C$, we have:
\begin{align}
&\E\norm{(C + \lambda\,\mathrm{Id})^{-1/2}k_\gV(\cdot, V)}^2_{\gH_\gV} \notag \\
&\quad = \E\left[\Tr\big((C + \lambda\,\mathrm{Id})^{-1/2}(k_\gV(\cdot, V) \otimes k_\gV(\cdot, V))(C + \lambda\,\mathrm{Id})^{-1/2}\big)\right] \notag \\
&\quad = \Tr(C(C + \lambda\,\mathrm{Id})^{-1}). \label{eq:rr_trace}
\end{align}
Thus, we obtain:
\begin{align}
\E\left[\norm{\Xi(C + \lambda\,\mathrm{Id})^{-1/2}}^2_{\text{HS}}\,\middle\vert\,\gD_0\right] &\overset{\text{(a)}}{\le} \frac{1}{n}\E\left[\norm{\zeta_1(C + \lambda\,\mathrm{Id})^{-1/2}}^2_{\text{HS}}\,\middle\vert\,\gD_0\right] \notag \\
&\overset{\text{(b)}}{=} \frac{1}{n}\E\left[\norm{U - F_\lambda(V)}^2_{\gH_\gY}\norm{(C + \lambda\,\mathrm{Id})^{-1/2}k_\gV(\cdot, V)}^2_{\gH_\gV}\,\middle\vert\,\gD_0\right] \notag \\
&\overset{\text{(c)}}{\le} \frac{2}{n}\E\left[\norm{U - \mu(V)}^2_{\gH_\gY}\norm{(C + \lambda\,\mathrm{Id})^{-1/2}k_\gV(\cdot, V)}^2_{\gH_\gV}\,\middle\vert\,\gD_0\right] \notag \\
&\quad + \frac{2}{n}\E\left[\norm{\mu(V) - F_\lambda(V)}^2_{\gH_\gY}\norm{(C + \lambda\,\mathrm{Id})^{-1/2}k_\gV(\cdot, V)}^2_{\gH_\gV}\right] \notag \\
&\overset{\text{(d)}}{\le} \frac{8B_U^2\Tr(C(C + \lambda\,\mathrm{Id})^{-1}) + 2\kappa_\gamma^2\lambda^{-\gamma}\norm{F_\lambda - \mu}^2_{L_2(P_V)}}{n}, \label{eq:rr_second}
\end{align}
where (a) holds because, given $\gD_0$, $\Xi(C + \lambda\,\mathrm{Id})^{-1/2} = \frac{1}{n}\sum_{i=1}^n\zeta_i(C + \lambda\,\mathrm{Id})^{-1/2} - \E[\zeta_1(C + \lambda\,\mathrm{Id})^{-1/2}\vert\gD_0]$ is the centered average of the i.i.d. Hilbert-Schmidt operators $\zeta_i(C + \lambda\,\mathrm{Id})^{-1/2}$; (b) follows from $\zeta_1(C + \lambda\,\mathrm{Id})^{-1/2} = (U_1 - F_\lambda(V_{11})) \otimes (C + \lambda\,\mathrm{Id})^{-1/2}k_\gV(\cdot, V_{11})$, the identity $\norm{a \otimes b}_{\text{HS}} = \norm{a}\norm{b}$ and the fact that $(V_{11}, U_1)$ is a copy of $(V, U)$ given $\gD_0$; (c) follows from $\norm{U - F_\lambda(V)}^2_{\gH_\gY} \le 2\norm{U - \mu(V)}^2_{\gH_\gY} + 2\norm{\mu(V) - F_\lambda(V)}^2_{\gH_\gY}$; and (d) follows from the third bound in~\eqref{eq:rr_U}, the second bound in~\eqref{eq:rr_fs} and~\eqref{eq:rr_trace}. Since $\lambda \asymp n^{-2\tau/(2r + d_v)}$ and $2\gamma\tau \le d_v + r$, we have:
\begin{equation}\label{eq:rr_nlg}
n\lambda^\gamma \gtrsim n^{r/(2r + d_v)} \ge 1.
\end{equation}
Thus, \eqref{eq:rr_varE}, \eqref{eq:rr_second}, the first bound in~\eqref{eq:rr_fs} and~\eqref{eq:rr_bias} yield:
\begin{align}
\E\left[\norm{\widehat F[U] - F_\lambda}^2_{L_2(P_V)}\mathds{1}_E\,\middle\vert\,\gD_0\right] &\le 9\,\E\left[\norm{\Xi(C + \lambda\,\mathrm{Id})^{-1/2}}^2_{\text{HS}}\,\middle\vert\,\gD_0\right] \notag \\
&\lesssim \frac{\lambda^{-d_v/(2\tau)}}{n} + \frac{\lambda^{r/\tau}}{n\lambda^\gamma} \lesssim n^{-2r/(2r + d_v)}. \label{eq:rr_stat}
\end{align}

\paragraph{Step~5: The Remainder and the Bound on $\E[R\vert\gD_0]$.}
According to~\cite[Lemma~17]{Fischer2020}, the conditions that $\Delta$ depends only on the independent copies $V_{11}, \ldots, V_{1n}$ of $V$ and that the embedding in~(ii) holds yield the following bound for every $s \ge 1$, with probability at least $1 - 2e^{-s}$:
\begin{equation}\label{eq:rr_Delta}
\norm{\Delta}_{\text{op}} \le \frac{4\kappa_\gamma^2 s\,g_\lambda}{3n\lambda^\gamma} + \sqrt{\frac{2\kappa_\gamma^2 s\,g_\lambda}{n\lambda^\gamma}},
\end{equation}
where $\kappa_\gamma$ is the constant in~\eqref{eq:rr_fs} and $g_\lambda \coloneqq \log\big(2e\,\Tr(C(C + \lambda\,\mathrm{Id})^{-1})(\norm{C}_{\text{op}} + \lambda)/\norm{C}_{\text{op}}\big)$. Let $n_1 \in \mathbb{N}$ be sufficiently large so that $\lambda \asymp n^{-2\tau/(2r + d_v)} \le \min\{1, \norm{C}_{\text{op}}\}$ for every $n \ge n_1$. Then, for such $n$, we have:
\begin{align*}
    g_\lambda \le \log\big(4e\,\Tr(C(C + \lambda\,\mathrm{Id})^{-1})\big) \le \log(4e\,\Tr(C)/\lambda) = \log(4e\kappa^2/\lambda).
\end{align*}
Let $\nu \coloneqq r/(2(2r + d_v))$. By~\eqref{eq:rr_nlg} and $\lambda \asymp n^{-2\tau/(2r + d_v)}$, we have $n\lambda^\gamma \gtrsim n^{2\nu}$ and $\log(4e\kappa^2/\lambda) \lesssim \log n$. Consequently, there is $n_2 \ge n_1$ such that $8\kappa_\gamma^2n^\nu\log(4e\kappa^2/\lambda) \le n\lambda^\gamma$ for every $n \ge n_2$. For such $n$ and $s = n^\nu$, we have $\kappa_\gamma^2s\,g_\lambda/(n\lambda^\gamma) \le 1/8$. Thus, the right-hand side of \eqref{eq:rr_Delta} is at most $2/3$, which yields $\mathbb{P}(E) \ge 1 - 2e^{-n^\nu}$. Moreover, $\lambda \le 1$ yields $R_{\max} \lesssim \lambda^{-1} \asymp n^{2\tau/(2r + d_v)}$, while $e^{-n^\nu}$ decays faster than any power of $n$. Since $E$ is independent of $\gD_0$, we obtain, for every $n \ge n_2$:
\begin{equation}\label{eq:rr_PE}
\mathbb{P}(E^c\vert\gD_0) = \mathbb{P}(E^c) \le 2e^{-n^\nu}, \qquad R_{\max}\,\mathbb{P}(E^c\vert\gD_0) \lesssim \lambda^{-1}e^{-n^\nu} \lesssim n^{-2r/(2r + d_v)}.
\end{equation}
Combining~\eqref{eq:rr_decomp}, \eqref{eq:rr_approx}, \eqref{eq:rr_stat} and~\eqref{eq:rr_PE}, we obtain, for every $n \ge n_2$, with a constant that does not depend on $\gD_0$:
\begin{equation}\label{eq:Rbound}
\E[R\vert\gD_0] = \E\left[\norm{\widehat F[U] - \mu}^2_{L_2(P_V)}\,\middle\vert\,\gD_0\right] \lesssim n^{-2r/(2r + d_v)}.
\end{equation}
And for all $n < n_2$, we use $R \le R_{\max}$. Thus, \eqref{eq:Rbound} holds for every $n \ge 1$.

\paragraph{Step~6: The Bound on $\E[\norm{\widehat F[b]}^2_{L_2(P_V)}\vert\gD_0]$.} On $E$, we have, for every $u \in \gH^n_\gY$:
\begin{align}
&\norm{\widehat F[u]}^2_{L_2(P_V)} \notag \\
&\quad \overset{\text{(a)}}{\le} \norm{T_u}^2_{\text{HS}}\norm{S_n(\widehat C + \lambda\,\mathrm{Id})^{-1/2}}^2_{\text{op}}\norm{(\widehat C + \lambda\,\mathrm{Id})^{-1/2}(C + \lambda\,\mathrm{Id})^{1/2}}^2_{\text{op}}\norm{(C + \lambda\,\mathrm{Id})^{-1/2}C^{1/2}}^2_{\text{op}} \notag \\
&\quad \overset{\text{(b)}}{\le} \frac{3}{n}\sum_{i=1}^n\norm{u_i}^2_{\gH_\gY}, \label{eq:rr_Fu_E}
\end{align}
where (a) follows from~\eqref{eq:rr_Fu_L2} and (b) follows from the formula for $\norm{T_u}^2_{\text{HS}}$, the first and third bounds in~\eqref{eq:rr_spec} and~\eqref{eq:rr_onE}. Recall that $b = \widetilde\mu_{Y^1\vert V} - \mu$ depends only on $\gD_0$ and satisfies $\sup_{v\in\gV}\norm{b(v)}_{\gH_\gY} \le \sqrt{B_{\widehat\xi}} + \sqrt{B_k}$. Since $E$ is independent of $\gD_0$, \eqref{eq:rr_Fu_E} on $E$, \eqref{eq:rr_Fu_unif} on $E^c$, both with $u_i = b(V_{1i})$, and~\eqref{eq:rr_PE} yield, for every $n \ge n_2$:
\begin{align}
\E\left[\norm{\widehat F[b]}^2_{L_2(P_V)}\,\middle\vert\,\gD_0\right] &\le 3\,\E\left[\frac{1}{n}\sum_{i=1}^n\norm{b(V_{1i})}^2_{\gH_\gY}\,\middle\vert\,\gD_0\right] + \frac{\kappa^2(\sqrt{B_{\widehat\xi}} + \sqrt{B_k})^2}{4\lambda}\,\mathbb{P}(E^c) \notag \\
&\le 3\norm{b}^2_{L_2(P_V)} + \frac{\kappa^2(\sqrt{B_{\widehat\xi}} + \sqrt{B_k})^2}{2\lambda}\,e^{-n^\nu} \notag \\
&\lesssim \norm{b}^2_{L_2(P_V)} + n^{-2r/(2r + d_v)} \notag \\
&= \norm{\widetilde\mu_{Y^1\vert V} - \mu}^2_{L_2(P_V)} + n^{-2r/(2r + d_v)}. \label{eq:bbound}
\end{align}
And for all $n < n_2$, we use~\eqref{eq:rr_Fu_unif}. Thus, \eqref{eq:bbound} holds for every $n \ge 1$.

\paragraph{Conclusion.}
Finally, taking the expectation of~\eqref{eq:rr_split} and applying the bounds~\eqref{eq:Rbound} and~\eqref{eq:bbound} yield:
\[
\E\big[\norm{\widehat\mu_{\mathrm{RR}}(V) - \mu(V)}^2_{\gH_\gY}\big] \lesssim n^{-\frac{2r}{2r + d_v}} + \E\big[\norm{\widetilde\mu_{Y^1\vert V}(V) - \mu(V)}^2_{\gH_\gY}\big].
\]
Since the assumptions of the theorem include $\pi(x), \widehat\pi(x) \in [\varepsilon, 1)$ and $\sup_{x\in\gX}\norm{\widehat\mu_0(x)}_{\gH_\gY} \le B_{\mu_0}$, we can use Lemma~\ref{lem:nuisance_bdd} to obtain $\E[\norm{\widetilde\mu_{Y^1\vert V}(V) - \mu(V)}^2_{\gH_\gY}] \lesssim \min\left\{\gR^{2}_{\pi}(\widehat\pi), \gR^{2}_{\mu_{0}}(\widehat\mu_{0})\right\}$, which completes the proof. \hfill $\blacksquare$

\section{Tools for Generalization Bounds}\label{sec:gen_tools}
\paragraph{\texorpdfstring{Notations.}{Notations.}} Recall the two datasets $\gD_0 = \{Z_{0i}\}_{i=1}^n$ and $\gD_1 = \{Z_{1i}\}_{i=1}^n$, where $Z_{0i} = (X_{0i}, A_{0i}, Y_{0i})$ and $Z_{1i} = (X_{1i}, A_{1i}, Y_{1i})$ are independent draws from $P$. Let $V_{1i} = \eta(X_{1i})$. Let $Z = (X, A, Y) \sim P$ independent of $\gD_0$ and $\gD_1$, and let $V = \eta(X)$. Recall from Appendix~\ref{sec:rr_setup} that $P_V$ denotes the distribution of $V$ and that $\norm{h}^2_{L_2(P_V)} = \int_\gV \norm{h(v)}^2_{\gH_\gY}\,p_V(v)\,\mathrm{d}v$ for $h \in (\gV \to \gH_\gY)$. The variable $Z$ takes values in $\mathcal{Z} \coloneqq \gX \times \{0,1\} \times \gY$, whose elements we denote by $z = (x, a, y)$. For a bounded function $g\colon \gV \times \mathcal{Z} \to \R$, which may depend on $\gD_0$ and $\gD_1$, we denote $\E_n[g(V, Z)] \coloneqq \frac{1}{n}\sum_{i=1}^n g(V_{1i}, Z_{1i})$. The expectation of $g(V, Z)$ over $Z$ is $\E[g(V, Z) \mid \gD_0, \gD_1]$. The pseudo-outcome $\widehat{\xi}$ is the map defined in~\eqref{eq:pseudo_outcome}, where, by the assumptions of Theorem~\ref{thm:ECME}, the first-stage estimator satisfies $\sup_{x \in \gX}\norm{\widehat\mu_{0}(x)}_{\gH_\gY}\le B_{\mu_0}$. We assume throughout that the kernel satisfies $\sup_{y \in \gY} k_\gY(y,y) \leq B_k$ for some constant $B_k > 0$. For a map $f$ with values in $\R^M$, we denote by $f_k(v)$ the $k$-th component of $f(v)$. When the map itself has a subscript $c$, we denote by $f_{c,k}(v)$ the $k$-th component of $f_c(v)$. We denote the cardinality of a finite set by $\abs{\cdot}$.

We start with some standard tools that are useful for obtaining excess risk bounds for classes of neural networks.

\begin{definition}[Metric Entropy]\label{def:emp_entropy}
	Let $\gS$ be a set. Let $\gF \subseteq (\gS \to \R^M)$ for some $M \ge 1$, and let $Q$ be a probability measure on $\gS$. For any $\varepsilon > 0$ and $p \in [1, \infty)$, an $(\varepsilon,p)$-\emph{cover} of $\gF$ in $L_p(Q)$ is a set $\gF' \subseteq (\gS \to \R^M)$ such that for any $f \in \gF$, there exists $f' \in \gF'$ with
	\[
		\left( \int_\gS \norm{f(s) - f'(s)}^{p}\,\mathrm{d}Q(s)\right)^{1/p}\leq \varepsilon.
	\]
	The $(\varepsilon, p)$-\emph{covering number}, denoted by $\gN_{p}(\gF, \varepsilon, Q)$, is the cardinality of the smallest such $(\varepsilon,p)$-cover. The $(\varepsilon,p)$-\emph{metric entropy} is the logarithm of the covering number:
	\[
		\gH_{p}(\gF, \varepsilon, Q) \coloneqq \log \gN_{p}(\gF, \varepsilon, Q).
	\]
\end{definition}

We consider a class $\gF^M \subseteq (\gV \to \R^M)$ and a linear isometry $\iota: \R^M \to \gH_\gY$. Let $\ell(h, z) \coloneqq \norm{h - \widehat\xi(z)}^2_{\gH_\gY}$, for $h \in \gH_\gY$ and $z \in \mathcal{Z}$, be the squared loss. Denote $\gL_{\mu}(v,z) \coloneqq \ell(\mu(v), z)$ for every map $\mu$ with values in $\gH_\gY$. Let $\widehat{f}$ be an empirical risk minimizer and $f^{\star}$ be a minimizer of the population risk, i.e.,
\begin{equation}\label{eq:minimizers}
	\widehat{f}\in \argmin_{f \in \gF^M}\E_n[\gL_{\iota\circ f}(V, Z)], \qquad f^{\star}\coloneqq \argmin_{f \in \gF^M}\E[\gL_{\iota\circ f}(V, Z) \mid \gD_0]\text{.}
\end{equation}
We shall assume that such minimizers exist.

We now bound the metric entropy of a class of neural networks. For widths $\mathbf{d} = (d_0, \ldots, d_L)$ with $d_0 = d_v$ and a constant $B \in (0, \infty]$, we denote by $\gF_{\mathrm{NN}}(\mathbf{d}, B)$ the class of the fully connected networks~\eqref{eq:nn_class} with outputs bounded by $B$:
\begin{equation}\label{eq:fnn_class}
\gF_{\mathrm{NN}}(\mathbf{d}, B) \coloneqq \Big\{ \psi_\theta : \theta \in \textstyle\prod_{i=1}^L \big(\R^{d_i \times d_{i-1}} \times \R^{d_i}\big),\ \sup_{v\in\gV}\norm{\psi_\theta(v)} \le B \Big\}.
\end{equation}
A common approach to bounding the metric entropy treats a network with $M$ outputs as $M$ scalar-valued functions. However, this approach yields a bound on the metric entropy that is linear in $M$. To reduce this dependence on $M$, we instead view these networks as functions that take the output index as an input. The following lemma provides the resulting bound:
\begin{lemma}[{Metric Entropy of Fully Connected Networks with $M$ Outputs}]\label{lem:dense_entropy}
Let $\mathbf{d} \coloneqq (d_0, \ldots, d_L)$ with $d_0 = d_v$, $d_L = M$ and $W \ge 2$, and let $B > 0$. Then there is a universal constant $c_{\mathrm{FC}} \ge 1$ such that, for every probability measure $Q$ on $\gV$ and every $\varepsilon \in (0,B]$,
\[
	\gH_2(\gF_{\mathrm{NN}}(\mathbf{d}, B), \varepsilon, Q) \le c_{\mathrm{FC}}\, WL\log W\log\left(\frac{4B\sqrt{M}}{\varepsilon}\right).
\]
\end{lemma}
\begin{proof}
We denote by $\mathds{1}_k \in \R^M$ the $k$-th one-hot vector. For $\psi \in \gF_{\mathrm{NN}}(\mathbf{d}, B)$, we denote by $\psi_k(v)$ the $k$-th coordinate of $\psi(v) \in \R^M$. We define $g_\psi \in (\gV \times \{\mathds{1}_1, \ldots, \mathds{1}_M\} \to \R)$ by $g_\psi(v, \mathds{1}_k) \coloneqq \psi_k(v)$.

We first show that a ReLU network with $O(W)$ parameters computes $g_\psi$. We denote by $h(v)$ the output of the last hidden layer, so that $\psi(v) = \rmW_L h(v) + \mathfrak{b}_L$. Moreover, we denote by $w_k^\top$ the $k$-th row of $\rmW_L$. The network that we construct has the input $(v, u) \in \R^{d_v + M}$ and consists of the following three components:
\begin{enumerate}
\item the hidden layers of $\psi$, which take $v$ as input and compute $h(v)$;
\item a layer of $2M$ ReLU units $s_k^\pm \coloneqq \sigma_{\mathrm{relu}}(\pm(w_k^\top h(v) + (\mathfrak{b}_L)_k) + 2Bu_k - 2B)$, $k = 1, \ldots, M$, where $u_k$ is the $k$-th component of $u$;
\item an output unit with value $\sum_{k=1}^M (s_k^+ - s_k^-)$.
\end{enumerate}
At $u = \mathds{1}_k$, the network above computes $g_\psi(v, \mathds{1}_k)$. Indeed, since $\abs{\psi_j(v)} \le \norm{\psi(v)} \le B$ for every $j$, we have:
\begin{align*}
s_j^\pm &= \sigma_{\mathrm{relu}}(\pm\psi_j(v) - 2B) = 0 \quad \text{for } j \ne k, \\
s_k^+ - s_k^- &= \sigma_{\mathrm{relu}}(\psi_k(v)) - \sigma_{\mathrm{relu}}(-\psi_k(v)) = \psi_k(v).
\end{align*}
Components 2 and 3 have $2M(d_{L-1} + 3) \lesssim M(d_{L-1} + 1)$ parameters. Since the output layer of $\psi$ has $M(d_{L-1} + 1) \le W$ parameters, this network has $\lesssim W$ parameters and $L + 1$ layers, with which~\cite[Theorem~7]{Bartlett2019} applies and yields that the class of functions on $\R^{d_v + M}$ computed by these networks has VC-dimension at most $d_{\mathrm{VC}} \coloneqq c\,WL\log W$ for a universal constant $c \ge 1$. From~\cite[Theorem~14.1]{Anthony1999} and its subsequent remark, $d_{\mathrm{VC}}$ is also an upper bound on the pseudo-dimension of this class. Consequently, the class $\gG_{\mathrm{FC}} \coloneqq \{g_\psi : \psi \in \gF_{\mathrm{NN}}(\mathbf{d}, B)\}$ has pseudo-dimension at most $d_{\mathrm{VC}}$.

Next, we convert this bound into a bound on the metric entropy of $\gG_{\mathrm{FC}}$. We use the covering numbers of Definition~\ref{def:emp_entropy} for real-valued functions on $\gS = \gV \times \{\mathds{1}_1, \ldots, \mathds{1}_M\}$ and the product $Q \otimes \mathrm{Unif}\{\mathds{1}_1, \ldots, \mathds{1}_M\}$ of $Q$ and the uniform distribution on $\{\mathds{1}_1, \ldots, \mathds{1}_M\}$. Since $\abs{g} \le B$ for every $g \in \gG_{\mathrm{FC}}$, the class $\{(g + B)/(2B) : g \in \gG_{\mathrm{FC}}\}$ takes values in $[0,1]$ and has pseudo-dimension at most $d_{\mathrm{VC}}$. By~\cite[Corollary~3]{Haussler1995}, for every $\varepsilon' \in (0, B]$, this class has an $((\varepsilon'/(2B))^2, 1)$-cover in $L_1$ of this product measure with at most $e(d_{\mathrm{VC}}+1)(8eB^2/\varepsilon'^2)^{d_{\mathrm{VC}}}$ elements, all of which belong to the class. This cover is also an $(\varepsilon'/(2B), 2)$-cover in $L_2$ of this product measure, since its elements take values in $[0,1]$. This yields:
\[
\log\gN_2(\gG_{\mathrm{FC}}, \varepsilon', Q \otimes \mathrm{Unif}\{\mathds{1}_1, \ldots, \mathds{1}_M\}) \le 2d_{\mathrm{VC}}\log(4eB/\varepsilon'), \qquad 0 < \varepsilon' \le B.
\]
Finally, for $\psi, \psi' \in \gF_{\mathrm{NN}}(\mathbf{d}, B)$, we have the following identity:
\[
\int_\gV\norm{\psi(v) - \psi'(v)}^2\,\mathrm{d}Q(v) = M\cdot\frac{1}{M}\sum_{k=1}^M\int_\gV\left(g_\psi(v, \mathds{1}_k) - g_{\psi'}(v, \mathds{1}_k)\right)^2\,\mathrm{d}Q(v).
\]
Taking $\varepsilon' = \varepsilon/\sqrt{M}$ yields the claim with $c_{\mathrm{FC}} = 2(1 + 1/\log 4)c$.
\end{proof}

\section{{Proof of the Upper Bound for the Deep Feature Estimator (Theorem~\ref{thm:DF_upper_bdd})}}\label{sec:df_proof}

\subsection{Notations and Setup}

We begin by clarifying notation. Recall from Section~\ref{sec:three_estimators} that the Deep Feature estimator uses a neural network $\psi_\theta: \gV \to \R^M$ parameterized by $\theta \in \Theta$. For the theoretical analysis, we consider the class $\gF^M \coloneqq \gF_{\mathrm{NN}}(\mathbf{d}, B_\gF) = \{\psi_\theta : \theta \in \Theta\}$ as defined in~\eqref{eq:fnn_class}, with $\mathbf{d} = (d_v, d_1, \ldots, d_{L-1}, M)$ and $\Theta = \{\theta : \sup_{v\in\gV}\norm{\psi_\theta(v)} \le B_\gF\}$, where $B_\gF > 0$ is a specified bound on the network output to prevent overfitting.

Recall from Appendix~\ref{sec:gen_tools} that $\ell(h, z) = \norm{h - \widehat{\xi}(z)}^2_{\gH_\gY}$ is the squared loss in the RKHS $\gH_\gY$. The estimator of Theorem~\ref{thm:DF_upper_bdd} is $\widehat\mu_{\mathrm{DF}} = \widehat{C}\psi_{\widehat\theta}$ with $\widehat{C} \coloneqq \widehat{C}_{\psi_{\widehat\theta}}$ as defined in~\eqref{eq:C_hat}, where $\psi_\theta \in \gF^M$. We denote by $J_n(\theta, C)$ the objective of~\eqref{eq:first_stage_DF_loss}:
\[
J_n(\theta, C) \coloneqq \frac{1}{n}\sum_{i=1}^{n} \norm{C\psi_\theta(V_{1i}) - \widehat\xi(Z_{1i})}_{\gH_\gY}^{2} + \lambda_1\norm{C}_{\text{HS}}^{2}.
\] 
Hence, by our definitions of $\widehat\theta$ and $\widehat{C}$, we have $J_n(\widehat\theta, \widehat{C}) = \min_{\theta, C} J_n(\theta, C)$.

\subsection{\texorpdfstring{Preliminary Lemma}{Preliminary Lemma}}

The following lemma provides an oracle inequality for the Deep Feature estimator. In this lemma and in the proof of Theorem~\ref{thm:DF_upper_bdd}, we denote $\lambda \coloneqq \lambda_1$ and $\widetilde\mu_{Y^1\vert V}(v) \coloneqq \E[\widehat\xi(Z)\vert V = v, \gD_0]$.

\begin{lemma}[Oracle Inequality for the Deep Feature Estimator]\label{lem:df_oracle}
Let the assumptions of Theorem~\ref{thm:DF_upper_bdd} hold with $\lambda \le 1$ and $n \ge 2$, and let $B_{\widehat\xi}$ be as in Lemma~\ref{lem:HS_bound}. Then, there exists a constant $c_{\mathrm{DF}} > 0$, which depends only on $B_{\widehat\xi}$ and $B_\gF$, such that, for every $\theta_c \in \Theta$ and every $C_c \in \gB\gL(\R^M;\gH_\gY)$ that do not depend on $\gD_1$, the estimator $\widehat\mu_{\mathrm{DF}}$ satisfies:
\begin{align*}
\E[\norm{\widehat\mu_{\mathrm{DF}} - \widetilde\mu_{Y^1\vert V}}^2_{L_2(P_V)} \mid \gD_0] &\le \norm{C_c\psi_{\theta_c} - \widetilde\mu_{Y^1\vert V}}^2_{L_2(P_V)} + \lambda\norm{C_c}^2_{\mathrm{HS}} \\
&\quad + \frac{c_{\mathrm{DF}}}{\lambda}\left(\frac{WL\log W\log(Mn)}{n}\right)^{1/2}.
\end{align*}
\end{lemma}

\begin{proof}
By Lemma~\ref{lem:HS_bound}, $\norm{\widehat\xi(z)}_{\gH_\gY} \le \sqrt{B_{\widehat\xi}}$ for all $z$. We first bound the Hilbert-Schmidt norm of $\widehat{C}$. Since $\widehat{C}$ minimizes $J_n(\widehat\theta, \cdot)$ over $\gB\gL(\R^M;\gH_\gY)$, comparing $\widehat{C}$ with $C = 0$ yields:
\[
\lambda\norm{\widehat{C}}^2_{\text{HS}} \le J_n(\widehat\theta, \widehat{C}) \le J_n(\widehat\theta, 0) = \frac{1}{n}\sum_{i=1}^{n}\norm{\widehat\xi(Z_{1i})}^2_{\gH_\gY} \le B_{\widehat\xi}.
\]
Consequently, $\norm{\widehat{C}}_{\text{HS}} \le R_\lambda \coloneqq \sqrt{B_{\widehat\xi}/\lambda}$. Let $\gF_{\mathrm{DF}}$ be the following function class:
\[
\gF_{\mathrm{DF}} \coloneqq \left\{ C\psi : \psi \in \gF^M,\ C \in \gB\gL(\R^M;\gH_\gY),\ \norm{C}_{\text{HS}} \le R_\lambda \right\}.
\]
Thus, $\widehat\mu_{\mathrm{DF}} \in \gF_{\mathrm{DF}}$ and $\norm{\widehat\mu_{\mathrm{DF}}(v)}_{\gH_\gY} \le R_\lambda B_\gF$ for all $v \in \gV$.

Next, we derive an oracle inequality. For the $\theta_c$ and $C_c$ in the statement of the lemma, we denote $\psi_c \coloneqq \psi_{\theta_c}$ and $f_c \coloneqq C_c\psi_c$. Recall from Appendix~\ref{sec:gen_tools} that $\E_n[g(V, Z)]$ is the average of $g(V_{1i}, Z_{1i})$ over the pairs of $\gD_1$. Since $J_n(\widehat\theta, \widehat{C}) \le J_n(\theta_c, C_c)$, we have:
\begin{equation}\label{eq:df_emp}
\E_n[\gL_{\widehat\mu_{\mathrm{DF}}}(V, Z)] \le \E_n[\gL_{f_c}(V, Z)] + \lambda\norm{C_c}^2_{\text{HS}}.
\end{equation}
Recall also from Appendix~\ref{sec:gen_tools} that $Z$ is independent of $\gD_0$ and $\gD_1$. Thus, we have $\E[\widehat\xi(Z)\vert V, \gD_0, \gD_1] = \E[\widehat\xi(Z)\vert V, \gD_0] = \widetilde\mu_{Y^1\vert V}(V)$. Consequently, expanding the square around $\widetilde\mu_{Y^1\vert V}(V)$ yields the identity $\E[\gL_f(V, Z) \mid \gD_0, \gD_1] = \norm{f - \widetilde\mu_{Y^1\vert V}}^2_{L_2(P_V)} + \E[\norm{\widehat\xi(Z) - \widetilde\mu_{Y^1\vert V}(V)}^2_{\gH_\gY} \mid \gD_0]$ for every bounded map $f \in (\gV \to \gH_\gY)$, which may depend on $\gD_0$ and $\gD_1$. With $f = \widehat\mu_{\mathrm{DF}}$ and $f = f_c$, we obtain:
\begin{align}
\E[\gL_{\widehat\mu_{\mathrm{DF}}}(V, Z) \mid \gD_0, \gD_1] &= \norm{\widehat\mu_{\mathrm{DF}} - \widetilde\mu_{Y^1\vert V}}^2_{L_2(P_V)} + \E[\norm{\widehat\xi(Z) - \widetilde\mu_{Y^1\vert V}(V)}^2_{\gH_\gY} \mid \gD_0], \label{eq:df_pop_hat} \\
\E[\gL_{f_c}(V, Z) \mid \gD_0] &= \norm{f_c - \widetilde\mu_{Y^1\vert V}}^2_{L_2(P_V)} + \E[\norm{\widehat\xi(Z) - \widetilde\mu_{Y^1\vert V}(V)}^2_{\gH_\gY} \mid \gD_0]. \label{eq:df_pop_c}
\end{align}
Combining~\eqref{eq:df_emp}, \eqref{eq:df_pop_hat} and~\eqref{eq:df_pop_c} yields:
\begin{equation}\label{eq:df_basic}
\begin{split}
&\norm{\widehat\mu_{\mathrm{DF}} - \widetilde\mu_{Y^1\vert V}}^2_{L_2(P_V)} - \norm{f_c - \widetilde\mu_{Y^1\vert V}}^2_{L_2(P_V)} \\
&\quad \le \E[\gL_{\widehat\mu_{\mathrm{DF}}}(V, Z) \mid \gD_0, \gD_1] - \E_n[\gL_{\widehat\mu_{\mathrm{DF}}}(V, Z)] \\
&\qquad + \E_n[\gL_{f_c}(V, Z)] - \E[\gL_{f_c}(V, Z) \mid \gD_0] + \lambda\norm{C_c}^2_{\text{HS}}.
\end{split}
\end{equation}
We now take the conditional expectation given $\gD_0$. Since $f_c$ does not depend on $\gD_1$, we have $\E\big[\E_n[\gL_{f_c}(V, Z)] \,\big\vert\, \gD_0\big] = \E[\gL_{f_c}(V, Z) \mid \gD_0]$. To treat the difference of the first two terms on the right-hand side of~\eqref{eq:df_basic}, we define $\bar\gL_f$ for $f \in (\gV \to \gH_\gY)$ by:
\[
\bar\gL_f(v,z) \coloneqq \norm{f(v)}^2_{\gH_\gY} - 2\inprod{f(v)}{\widehat\xi(z)}_{\gH_\gY} = \gL_f(v,z) - \norm{\widehat\xi(z)}^2_{\gH_\gY}.
\]
By this identity, and since $\norm{\widehat\xi(z)}^2_{\gH_\gY}$ does not depend on $\gD_1$, the conditional expectation of this difference given $\gD_0$ does not change when we replace $\gL_{\widehat\mu_{\mathrm{DF}}}$ by $\bar\gL_{\widehat\mu_{\mathrm{DF}}}$.

We now bound the conditional expectation of $\E[\bar\gL_{\widehat\mu_{\mathrm{DF}}}(V, Z) \mid \gD_0, \gD_1] - \E_n[\bar\gL_{\widehat\mu_{\mathrm{DF}}}(V, Z)]$ given $\gD_0$. Since $\widehat\mu_{\mathrm{DF}} \in \gF_{\mathrm{DF}}$, we have:
\begin{equation}\label{eq:df_sup}
\begin{split}
&\E\left[\E[\bar\gL_{\widehat\mu_{\mathrm{DF}}}(V, Z) \mid \gD_0, \gD_1] - \E_n[\bar\gL_{\widehat\mu_{\mathrm{DF}}}(V, Z)] \,\middle\vert\, \gD_0\right] \\
&\quad \le \E\left[\sup_{f\in\gF_{\mathrm{DF}}}\abs{\E[\bar\gL_f(V, Z) \mid \gD_0] - \E_n[\bar\gL_f(V, Z)]} \,\middle\vert\, \gD_0\right].
\end{split}
\end{equation}
On $\gF_{\mathrm{DF}}$, we have $\abs{\bar\gL_f} \le R_\lambda^2B_\gF^2 + 2R_\lambda B_\gF\sqrt{B_{\widehat\xi}}$. Let $\varepsilon_1, \ldots, \varepsilon_n$ be independent Rademacher variables, and denote
\[
\mathrm{Rad}_n \coloneqq \E_\varepsilon\left[\sup_{f\in\gF_{\mathrm{DF}}}\abs{\frac{1}{n}\sum_{i=1}^n\varepsilon_i\,\bar\gL_f(V_{1i}, Z_{1i})}\right].
\]
By the symmetrization inequality~\cite[Chapter~4]{Wainwright2019}, the condition that $\bar\gL_f$ is bounded on $\gF_{\mathrm{DF}}$ yields:
\begin{equation}\label{eq:df_sym}
\E\left[\sup_{f\in\gF_{\mathrm{DF}}}\abs{\E[\bar\gL_f(V, Z) \mid \gD_0] - \E_n[\bar\gL_f(V, Z)]} \,\middle\vert\, \gD_0\right] \le 2\,\E[\mathrm{Rad}_n \mid \gD_0].
\end{equation}
Taking the conditional expectation of~\eqref{eq:df_basic} and inserting~\eqref{eq:df_sup} and~\eqref{eq:df_sym} yield:
\begin{equation}\label{eq:df_oracle}
\E[\norm{\widehat\mu_{\mathrm{DF}} - \widetilde\mu_{Y^1\vert V}}^2_{L_2(P_V)} \mid \gD_0] \le \norm{f_c - \widetilde\mu_{Y^1\vert V}}^2_{L_2(P_V)} + \lambda\norm{C_c}^2_{\text{HS}} + 2\,\E[\mathrm{Rad}_n \mid \gD_0].
\end{equation}
It remains to bound $\mathrm{Rad}_n$, which will be done via a standard covering argument combined with a maximal inequality. For $\psi \in \gF^M$, we write $\psi_i \coloneqq \psi(V_{1i}) \in \R^M$ and $\widehat\xi_i \coloneqq \widehat\xi(Z_{1i})$. Since $\inprod{C\psi_i}{\widehat\xi_i}_{\gH_\gY} = \inprod{C}{\widehat\xi_i \otimes \psi_i}_{\text{HS}}$ and $\norm{C\psi_i}^2_{\gH_\gY} = \psi_i^\top C^*C\,\psi_i$, we have:
\[
\bar\gL_{C\psi}(V_{1i}, Z_{1i}) = \psi_i^\top C^*C\,\psi_i - 2\inprod{C}{\widehat\xi_i \otimes \psi_i}_{\text{HS}}.
\]
By Lemma~\ref{lem:dense_entropy} at the scale $B_\gF/\sqrt{n}$ for the empirical measure $Q$ of $V_{11}, \ldots, V_{1n}$, there is a finite $(B_\gF/\sqrt{n}, 2)$-cover $\gN_0$ of $\gF^M$ in $L_2(Q)$ with $\log\abs{\gN_0} \le c_{\mathrm{FC}}\, WL\log W\log(4\sqrt{Mn})$. We denote by $\Pi_{B_\gF}\colon \R^M \to \{u \in \R^M : \norm{u} \le B_\gF\}$ the projection onto the ball, and by $\gN \coloneqq \{\Pi_{B_\gF}\circ\psi : \psi \in \gN_0\}$ the set of the projections of the elements of $\gN_0$. Since $\norm{\Pi_{B_\gF}(u) - \Pi_{B_\gF}(u')} \le \norm{u - u'}$ for $u, u' \in \R^M$ and $\Pi_{B_\gF}(\psi(v)) = \psi(v)$ for every $\psi \in \gF^M$, the set $\gN$ is a $(B_\gF/\sqrt{n}, 2)$-cover of $\gF^M$ in $L_2(Q)$ with:
\[
\log\abs{\gN} \le \log\abs{\gN_0} \le c_{\mathrm{FC}}\, WL\log W\log\left(4\sqrt{Mn}\right).
\]
For any $\psi \in \gF^M$, we choose $\psi' \in \gN$ with $\frac{1}{n}\sum_{i=1}^n\norm{\psi_i - \psi'_i}^2 \le B_\gF^2/n$, where $\psi'_i \coloneqq \psi'(V_{1i})$. For $\norm{C}_{\text{HS}} \le R_\lambda$, the matrix $C^*C$ is symmetric with $\norm{C^*C}_{\text{op}} \le R_\lambda^2$. Thus, since $\norm{\psi_i} \le B_\gF$ and $\norm{\psi'_i} \le B_\gF$, we have:
\[
\abs{\psi_i^\top C^*C\,\psi_i - \psi_i'^\top C^*C\,\psi'_i} = \abs{(\psi_i - \psi'_i)^\top C^*C(\psi_i + \psi'_i)} \le 2R_\lambda^2B_\gF\norm{\psi_i - \psi'_i}.
\]
By the Cauchy-Schwarz inequality, we also have $\abs{\inprod{C(\psi_i - \psi'_i)}{\widehat\xi_i}_{\gH_\gY}} \le R_\lambda\sqrt{B_{\widehat\xi}}\norm{\psi_i - \psi'_i}$. Since $\frac{1}{n}\sum_{i=1}^n\norm{\psi_i - \psi'_i} \le B_\gF/\sqrt{n}$ by the Cauchy-Schwarz inequality, every $C$ with $\norm{C}_{\text{HS}} \le R_\lambda$ and every choice of $\varepsilon_i$'s satisfy:
\begin{equation}\label{eq:df_cover}
\abs{\frac{1}{n}\sum_{i=1}^n\varepsilon_i\left(\bar\gL_{C\psi} - \bar\gL_{C\psi'}\right)(V_{1i}, Z_{1i})} \le \left(2R_\lambda^2B_\gF + 2R_\lambda\sqrt{B_{\widehat\xi}}\right)\frac{B_\gF}{\sqrt{n}}.
\end{equation}
Next, we fix $\psi' \in \gN$. Denote $S_1(\psi') \coloneqq \frac{1}{n}\sum_{i=1}^n\varepsilon_i\psi'_i\psi_i'^\top$ and $S_2(\psi') \coloneqq \frac{1}{n}\sum_{i=1}^n\varepsilon_i\,\widehat\xi_i \otimes \psi'_i$. By the Cauchy-Schwarz inequality in the Frobenius and Hilbert-Schmidt inner products, and since $\norm{C^*C}_F \le R_\lambda^2$ for $\norm{C}_{\text{HS}} \le R_\lambda$, we have:
\begin{equation}\label{eq:df_cs}
\sup_{\norm{C}_{\text{HS}} \le R_\lambda}\abs{\frac{1}{n}\sum_{i=1}^n\varepsilon_i\,\bar\gL_{C\psi'}(V_{1i}, Z_{1i})} \le R_\lambda^2\norm{S_1(\psi')}_F + 2R_\lambda\norm{S_2(\psi')}_{\text{HS}}.
\end{equation}
We now bound the two terms on the right-hand side. By Jensen's inequality and $\E[\varepsilon_i\varepsilon_j] = 0$ for $i \ne j$, we have $\E_\varepsilon n\norm{S_1(\psi')}_F \le (\sum_{i=1}^n\norm{\psi'_i}^4)^{1/2} \le \sqrt{n}B_\gF^2$. Moreover, denoting by $S_1^{(i)}(\psi')$ the matrix $S_1(\psi')$ with $\varepsilon_i$ replaced by $-\varepsilon_i$, we have $\abs{n\norm{S_1(\psi')}_F - n\norm{S_1^{(i)}(\psi')}_F} \le 2B_\gF^2$. Thus, according to McDiarmid's inequality~\cite[Theorem~6.2]{Boucheron2013}, every $t \in \R$ satisfies:
\[
\E_\varepsilon\left[\exp\left\{t\left(n\norm{S_1(\psi')}_F - \E_\varepsilon\left[n\norm{S_1(\psi')}_F\right]\right)\right\}\right] \le \exp\left(\frac{nB_\gF^4t^2}{2}\right).
\]
Since $\norm{\widehat\xi_i \otimes \psi'_i}_{\text{HS}} \le \sqrt{B_{\widehat\xi}}B_\gF$, we apply the same argument to $n\norm{S_2(\psi')}_{\text{HS}}$. The maximal inequality for $\abs{\gN}$ sub-Gaussian variables, which need not be independent~\cite[Theorem~2.5]{Boucheron2013}, then yields:
\begin{equation}\label{eq:df_max}
\begin{aligned}
\E_\varepsilon\max_{\psi'\in\gN}n\norm{S_1(\psi')}_F &\le \sqrt{n}B_\gF^2\left(1 + \sqrt{2\log\abs{\gN}}\right), \\
\E_\varepsilon\max_{\psi'\in\gN}n\norm{S_2(\psi')}_{\text{HS}} &\le \sqrt{n}\sqrt{B_{\widehat\xi}}B_\gF\left(1 + \sqrt{2\log\abs{\gN}}\right).
\end{aligned}
\end{equation}
Combining~\eqref{eq:df_cover}, \eqref{eq:df_cs} and~\eqref{eq:df_max}, we obtain:
\begin{equation}\label{eq:df_rad}
\mathrm{Rad}_n \le \left(R_\lambda^2B_\gF^2 + 2R_\lambda\sqrt{B_{\widehat\xi}}B_\gF\right)\frac{1 + \sqrt{2\log\abs{\gN}}}{\sqrt{n}} + \left(2R_\lambda^2B_\gF + 2R_\lambda\sqrt{B_{\widehat\xi}}\right)\frac{B_\gF}{\sqrt{n}}.
\end{equation}
Finally, since $R_\lambda^2 = B_{\widehat\xi}/\lambda$ and $R_\lambda\sqrt{B_{\widehat\xi}} = B_{\widehat\xi}/\sqrt{\lambda} \le B_{\widehat\xi}/\lambda$ for $\lambda \le 1$, \eqref{eq:df_rad} and the bound on $\log\abs{\gN}$ yield:
\[
\mathrm{Rad}_n \le \frac{B_{\widehat\xi}(B_\gF^2 + 2B_\gF)}{\lambda}\cdot\frac{3 + \sqrt{2c_{\mathrm{FC}} WL\log W\log(4\sqrt{Mn})}}{\sqrt{n}}.
\]
Since $n \ge 2$, $W \ge 2$ and $c_{\mathrm{FC}}$ is a universal constant, there is a universal constant $c > 0$ such that the numerator of the last fraction is at most $c\,(WL\log W\log(Mn))^{1/2}$. Inserting this bound into the right-hand side of~\eqref{eq:df_oracle} yields the claim with $c_{\mathrm{DF}} = 2cB_{\widehat\xi}(B_\gF^2 + 2B_\gF)$.
\end{proof}

\subsection{Proof of the Bound}

We now begin the proof of the upper bound. We denote $\mu \coloneqq \mu_{Y^1\vert V}$. The quantities $\lambda$ and $\widetilde\mu_{Y^1\vert V}$ are as defined before Lemma~\ref{lem:df_oracle}. The norm $\norm{\cdot}_{L_2(P_V)}$ is as defined in Appendix~\ref{sec:rr_setup}. 

Our main idea of the proof is to choose $M$ auxiliary points to define a subspace on which to approximate the CCME. In particular, we construct $M$ distinct points $\{\widetilde{y}_j\}_{j=1}^M \subset \gY$ whose fill distance $h = \sup_{y\in\gY}\min_j\norm{y - \widetilde{y}_j}$ satisfies $h \asymp M^{-1/d_y}$ as follows: For $s \in (0,1]$, let $\gT_s \subseteq \gY$ be a maximal set with $\norm{y - y'} > s$ for all distinct $y, y' \in \gT_s$. By the maximality of $\gT_s$, $\sup_{y\in\gY}\min_{y'\in \gT_s}\norm{y - y'} \le s$. Since the sets $\{y : \norm{y - y'} < s/2\}$, $y' \in \gT_s$, are disjoint and lie in a bounded neighborhood of $\gY$, we have $\abs{\gT_s} \le C_\gY s^{-d_y}$ for a constant $C_\gY$ that only depends on $\gY$. For any fixed choice of $M \ge C_\gY$, we take $s = (C_\gY/M)^{1/d_y}$ and add distinct points of $\gY$ to $\gT_s$ until there are $M$ points $\widetilde{y}_1, \ldots, \widetilde{y}_M$, which does not increase the fill distance. Thus, $h \le (C_\gY/M)^{1/d_y}$. Conversely, since $\gY \subseteq \bigcup_{j=1}^M \{y : \norm{y - \widetilde{y}_j} \le h\}$, we have $h \gtrsim M^{-1/d_y}$.

By Assumption~\ref{ass:kernel_nk} and Proposition~\ref{prop:feature_indep}, the features $\phi(\widetilde{y}_1), \ldots, \phi(\widetilde{y}_M)$ are linearly independent and their Gram matrix $\rmK_M = (k_\gY(\widetilde{y}_j, \widetilde{y}_l))_{j,l=1}^M$ is invertible. Consequently, the subspace $\gH_M \coloneqq \text{span}\{\phi(\widetilde{y}_j)\}_{j=1}^M$ satisfies $\dim\gH_M = M$. We denote by $\Pi_M\colon \gH_\gY \to \gH_\gY$ the orthogonal projection onto $\gH_M$. Let $\{e_k\}_{k=1}^M$ be an orthonormal basis of $\gH_M$. We use the isometry $C_M: \R^M \to \gH_M$ defined by $C_M(w) \coloneqq \sum_{k=1}^M w_k e_k$, for which $\norm{C_M}^2_{\text{HS}} = M$. Let $\theta_c \in \Theta$ be a parameter, and denote $\psi_c \coloneqq \psi_{\theta_c}$. We define the approximation, nuisance and projection errors:
\begin{gather*}
\rmB \coloneqq \E\big[\norm{C_M\psi_c(V) - \Pi_M\mu(V)}^2_{\gH_\gY}\big], \qquad \mathbf{C} \coloneqq \E\big[\norm{\widetilde\mu_{Y^1\vert V}(V) - \mu(V)}^2_{\gH_\gY}\big], \\
\mathbf{D} \coloneqq \E\big[\norm{\Pi_M\mu(V) - \mu(V)}^2_{\gH_\gY}\big].
\end{gather*}
Then, we have:
\begin{align*}
\norm{C_M\psi_c - \widetilde\mu_{Y^1\vert V}}^2_{L_2(P_V)} &\lesssim \rmB + \mathbf{D} + \norm{\widetilde\mu_{Y^1\vert V} - \mu}^2_{L_2(P_V)}, \\
\norm{\widehat\mu_{\mathrm{DF}} - \mu}^2_{L_2(P_V)} &\lesssim \norm{\widehat\mu_{\mathrm{DF}} - \widetilde\mu_{Y^1\vert V}}^2_{L_2(P_V)} + \norm{\widetilde\mu_{Y^1\vert V} - \mu}^2_{L_2(P_V)}.
\end{align*}
Since $\theta_c$ and $C_M$ do not depend on $\gD_1$, Lemma~\ref{lem:df_oracle} with $C_c = C_M$, the identity $\norm{C_M}^2_{\text{HS}} = M$ and the last display yield, after we take the expectation over $\gD_0$, the following bound for $\lambda \le 1$ and $n \ge 2$:
\begin{equation}\label{eq:df_total}
\E\big[\norm{\widehat\mu_{\mathrm{DF}}(V) - \mu(V)}^2_{\gH_\gY}\big] \lesssim \mathbf{A} + \rmB + \mathbf{C} + \mathbf{D} + \lambda M,
\end{equation}
where $\mathbf{A} \coloneqq \lambda^{-1}(WL\log W\log(Mn)/n)^{1/2}$ is the statistical error. It remains to bound $\rmB$, $\mathbf{C}$ and $\mathbf{D}$.

\paragraph{Term $\rmB$: The Approximation Error.}
Recall that $C_M\psi(v) = \sum_{k=1}^M \psi_k(v)e_k$. Thus, for $C_M\psi(v)$ to approximate $\Pi_M\mu_{Y^1\vert V}(v) = \sum_{k=1}^M c_k(v)e_k$, with $c_k$ as in Lemma~\ref{lem:sobolevcoef}, the neural network $\psi \in \gF^M$ must approximate the coefficient vector $\mathbf{c}(v) \coloneqq (c_1(v), \ldots, c_M(v))$. To achieve this, we use Lemma~\ref{lem:approx}, which provides an approximation of $\Pi_M\mu_{Y^1\vert V}$ by a fully connected network with $M$ bounded outputs. The conditions of Lemma~\ref{lem:approx} are that (i)~Assumptions~\ref{ass:boundk}, \ref{ass:domain} and~\ref{ass:smooth_density} hold, and (ii)~the smallest hidden width $D \coloneqq \min_{1 \le i \le L-1} d_i$ satisfies $D \ge \bar w M$ for a constant $\bar w \ge 2$ defined in Lemma~\ref{lem:approx}.

We now check these conditions. Condition~(i) is immediate by our assumptions. Thus, it remains to show that condition~(ii) holds for all sufficiently large $n$. Our choices of $M$ and $WL$ in Case~1 of the theorem yield $WL/M^2 \asymp n^{a/(4ab+4b+a)}$ and those in Case~2 yield $WL/M^2 \asymp n^{1/(4b+1)}(\log n)^{-4d_y}$. Since $L$ does not depend on $n$, we have $W/M^2 \to \infty$ as $n\to\infty$ in both cases. By the assumption $\min_{1 \le i \le L-1} d_i \asymp \max_{1 \le i \le L-1} d_i$, every hidden width $d_i$ satisfies $d_i \lesssim D$. Thus, $W = \sum_{i=1}^L (d_id_{i-1} + d_i)\lesssim LD^2 + MD$, which implies $W/M^2 \lesssim L(D/M)^2 + D/M$. Since $W/M^2 \to \infty$, we must have $D/M \to \infty$. We therefore conclude that condition~(ii) holds for all sufficiently large $n$. Consequently, Lemma~\ref{lem:approx} yields constants $\bar B > 0$ and $L_0$ such that, for any $B_\gF^2 \ge \bar B$, any $L \ge \max\{L_0, 3\}$ and all sufficiently large $n$, there exists $\psi_c \in \gF^M =\gF_{\mathrm{NN}}(\mathbf{d}, B_\gF)$ such that $\E[\norm{C_M\psi_c(V) - \Pi_M\mu(V)}^2_{\gH_\gY}] \lesssim (M/(DL))^{2b}$. Since $M \lesssim D$, the bound $W \lesssim LD^2 + MD$ gives $W \lesssim LD^2$. Thus, inserting $\psi_c$ into the definition of $\rmB$ yields:
\[
	\rmB = \E\big[\norm{C_M\psi_c(V) - \Pi_M\mu(V)}^2_{\gH_\gY}\big] \lesssim \left(\frac{M}{DL}\right)^{2b} \lesssim \left(\frac{M^2}{WL}\right)^{b}.
\]
\paragraph{Term $\mathbf{C}$: The Nuisance Error.} Since the assumptions of the theorem include $\pi(x), \widehat\pi(x) \in [\varepsilon, 1)$ and $\sup_{x\in\gX}\norm{\widehat\mu_0(x)}_{\gH_\gY} \le B_{\mu_0}$, we can use Lemma~\ref{lem:nuisance_bdd} to obtain:
\[
	\mathbf{C}\lesssim \min\left\{\gR^{2}_{\pi}(\widehat\pi), \gR^{2}_{\mu_{0}}(\widehat\mu_{0})\right\}.
\]
\paragraph{Term $\mathbf{D}$: The Projection Error.} By Assumptions~\ref{ass:domain} and~\ref{ass:smooth_density}, we have $\E[\norm{\fp^1(\cdot\vert V)}^2_{L^2(\gY)}] \le \sup_{v\in\gV}p_V(v)\int_\gY \norm{\fp^1(y\vert\cdot)}^2_{L^2(\gV)}\,\mathrm{d}y < \infty$, so $\fp^1(\cdot\vert V) \in L^2(\gY)$ almost surely. According to Lemma~\ref{lem:projbound}, the conditions that (i) the auxiliary points have fill distance $h \asymp M^{-1/d_y}$, (ii) $\Pi_M$ is the orthogonal projection onto $\text{span}\{\phi(\widetilde{y}_j)\}_{j=1}^M$ and (iii) $\fp^1(\cdot\vert V) \in L^2(\gY)$ almost surely yield, under Assumptions~\ref{ass:domain} and~\ref{ass:kernel}:
\[
	\mathbf{D} = \E\left[\norm{\Pi_M \mu_{Y^1\vert V}(V) - \mu_{Y^1\vert V}(V)}^2_{\gH_\gY}\right] \lesssim M^{-2\tau/d_y}\,\E\left[\norm{\fp^1(\cdot\vert V)}^2_{L^2(\gY)}\right] \lesssim M^{-a}.
\]
In Case~2 of the theorem, the second part of Lemma~\ref{lem:projbound} yields, for some $B_0 > 0$:
\begin{equation}\label{eq:df_gauss}
	\mathbf{D} \lesssim \exp\left( - B_0 M^{\frac{1}{2d_y}} \log M \right).
\end{equation}
\paragraph{Combining the bounds.}

In Case~1 of the theorem, we obtain the following inequality from~\eqref{eq:df_total} and the bounds of Terms~$\rmB$, $\mathbf{C}$ and~$\mathbf{D}$:
\begin{align*}
	\E[\norm{\widehat\mu_{\mathrm{DF}}(V) - \mu_{Y^1\vert V}(V)}^{2}_{\gH_{\gY}}] &\lesssim \left(\frac{M^2}{WL}\right)^{b} + M^{-a} + \lambda M \\
	&\quad + \frac{1}{\lambda}\left(\frac{WL\log W\log(Mn)}{n}\right)^{1/2} + \min\left\{\gR^{2}_{\pi}(\widehat \pi), \gR^{2}_{\mu_{0}}(\widehat\mu_{0}) \right\}.
\end{align*}
Recall our choices of $M \asymp n^{\frac{b}{4ab+4b+a}}$ and $WL \asymp n^{\frac{a+2b}{4ab+4b+a}}$. By choosing $\lambda \asymp n^{-\frac{b(a+1)}{4ab+4b+a}}$, we have $(M^2/(WL))^b \asymp M^{-a} \asymp \lambda M \asymp \lambda^{-1}(WL/n)^{1/2} \asymp n^{-\frac{ab}{4ab+4b+a}}$ and the conditions of~\eqref{eq:df_total} hold for all sufficiently large $n$. Therefore, we obtain the bound:
\[
	\E[\norm{\widehat\mu_{\mathrm{DF}}(V) - \mu_{Y^1\vert V}(V)}^{2}_{\gH_{\gY}}] \lesssim_{\log n} n^{-\frac{ab}{4ab+4b+a}} + \min\left\{\gR^{2}_{\pi}(\widehat \pi), \gR^{2}_{\mu_{0}}(\widehat\mu_{0})\right\}.
\]
In Case~2, we instead obtain from \eqref{eq:df_total}, the bounds of Terms~$\rmB$ and~$\mathbf{C}$, and \eqref{eq:df_gauss} that:
\begin{align*}
	\E[&\norm{\widehat\mu_{\mathrm{DF}}(V) - \mu_{Y^1\vert V}(V)}^{2}_{\gH_{\gY}}]   \\
	& \lesssim \left(\frac{M^2}{WL}\right)^{b} + \exp\left( - B_0 M^{\frac{1}{2d_y}} \log M \right) + \lambda M \\
	&\quad + \frac{1}{\lambda}\left(\frac{WL\log W\log(Mn)}{n}\right)^{1/2} + \min\left\{\gR^{2}_{\pi}(\widehat \pi), \gR^{2}_{\mu_{0}}(\widehat\mu_{0})\right\}.
\end{align*}
Recall our choices of $M \asymp (\log n)^{2d_y}$ and $WL \asymp n^{\frac{1}{4b+1}}$. Since $M^{\frac{1}{2d_y}} \gtrsim \log n$ and $\log M \to \infty$, we have $\exp( - B_0 M^{\frac{1}{2d_y}} \log M ) \le n^{-1}$ for all sufficiently large $n$. By choosing $\lambda \asymp n^{-\frac{b}{4b+1}}$, we have $(M^2/(WL))^b + \lambda M + (1/\lambda)(WL/n)^{1/2} \lesssim_{\log n} n^{-\frac{b}{4b+1}}$. Moreover, the conditions of~\eqref{eq:df_total} hold for all sufficiently large $n$. Therefore, we obtain the bound:
\[
	\E[\norm{\widehat\mu_{\mathrm{DF}}(V) - \mu_{Y^1\vert V}(V)}^{2}_{\gH_{\gY}}] \lesssim_{\log n} n^{-\frac{b}{4b+1}} + \min\left\{\gR^{2}_{\pi}(\widehat \pi), \gR^{2}_{\mu_{0}}(\widehat\mu_{0})\right\}.
\]
\hfill $\blacksquare$

\section{{Proof of the Upper Bound for the Neural-Kernel Estimator (Theorem~\ref{thm:NK_upper_bdd})}}

\subsection{Notations and Setup}

We recall the definition of the Neural-Kernel estimator. We consider the $M$ grid points $\{\widetilde{y}_{j}\}_{j=1}^{M} \subset \gY$ of Assumption~\ref{ass:grid}. 

Let $f_\theta:\gV \to \R^M$ be a neural network parameterized by $\theta \in \Theta$ of Theorem~\ref{thm:NK_upper_bdd} with $M$ outputs. We denote by $\gF_{\mathrm{FC}}$ the class of neural networks $f_\theta$ with the widths $\mathbf{d} = (d_v, d_1, \ldots, d_{L-1}, d_L)$ where $d_L = M$, that is, $\gF_{\mathrm{FC}} \coloneqq \{f_\theta : \theta \in \prod_{i=1}^L (\R^{d_i \times d_{i-1}} \times \R^{d_i}),\ \sup_{v\in\gV}\norm{\sum_{j=1}^M f_\theta(v)_j\phi(\widetilde{y}_j)}_{\gH_\gY} \le B_\gF\}$. We denote by $\gF_{\mathrm{NK}}$ the class of the estimator: $\gF_{\mathrm{NK}} \coloneqq \{v \mapsto \sum_{j=1}^M f_j(v)\phi(\widetilde{y}_j) : f \in \gF_{\mathrm{FC}}\}$. Thus, the elements of $\gF_{\mathrm{NK}}$ take values in $\gH_M \coloneqq \text{span}\{\phi(\widetilde{y}_j)\}_{j=1}^M$.

\subsection{\texorpdfstring{Preliminary Lemmas}{Preliminary Lemmas}}
The following lemma provides an orthonormal basis of $\gH_M$ in which the Neural-Kernel estimator has network coordinates:

\begin{lemma}\label{lem:Llamb}
Let Assumption~\ref{ass:kernel_nk} hold, and denote $e_k \coloneqq \sum_{j=1}^M (\rmK_M^{-1/2})_{jk}\phi(\widetilde{y}_j)$ for $k = 1, \ldots, M$. Then, $\{e_k\}_{k=1}^M$ is an orthonormal basis of $\gH_M$. In addition, every $w \in \R^M$ satisfies
\[
	\sum_{j=1}^M w_j \phi(\widetilde{y}_j) = \sum_{k=1}^M (\rmK_M^{1/2}w)_k\, e_k.
\]
\end{lemma}

\begin{proof}
By Assumption~\ref{ass:kernel_nk}, $\varphi$ is strictly positive definite \cite[Chapter~6]{Wendland2004}, and thus $\rmK_M$ is invertible by Proposition~\ref{prop:feature_indep}. Since the Gram matrix of $\{e_k\}_{k=1}^M$ is $\rmK_M^{-1/2}\rmK_M\rmK_M^{-1/2} = \mathbf{I}_M$, the set $\{e_k\}_{k=1}^M$ is an orthonormal basis of $\gH_M$. For the second claim, we expand the right-hand side:
\[
	\sum_{k=1}^M (\rmK_M^{1/2}w)_k\, e_k = \sum_{j=1}^M (\rmK_M^{-1/2}\rmK_M^{1/2}w)_j\, \phi(\widetilde{y}_j) = \sum_{j=1}^M w_j \phi(\widetilde{y}_j).
\]
\end{proof}

In the Deep Feature estimator, the map $C$ is an element of a ball in the infinite-dimensional space $\gB\gL(\R^M; \gH_\gY)$, for which we have no entropy bound. Thus, we bound the supremum over $C$ in closed form in~\eqref{eq:df_cs}, with which we can employ the Rademacher complexity argument. The class of the Neural-Kernel estimator, on the other hand, consists of networks with finitely many parameters, for which Lemma~\ref{lem:dense_entropy} provides an entropy bound. Thus, we can employ the following oracle inequality to achieve, up to logarithmic factors, an excess risk bound of order $O(n^{-1})$ instead of the order $O(n^{-1/2})$ that we obtained for the Deep Feature estimator. Since~\cite[Theorem~3.5.4]{vanderVaart2023} is stated for multiple models in the context of model selection, we adapt this theorem to a single model:

\begin{lemma}[{Oracle Inequality~\cite[Theorem~3.5.4]{vanderVaart2023}}]\label{lem:oracle}
Let $\gG \subseteq (\gV \times \mathcal{Z} \to \R)$ be a class of functions that do not depend on $\gD_1$. Let $g_0\colon \gV \times \mathcal{Z} \to \R$ be a function that does not depend on $\gD_1$, and let $\widehat g \in \argmin_{g \in \gG}\E_n[g(V, Z)]$. Assume that (i) every $g \in \gG$ satisfies $\abs{g} \le 1$, (ii) every $g \in \gG$ satisfies $\E[(g(V, Z) - g_0(V, Z))^2 \mid \gD_0] \le \E[g(V, Z) - g_0(V, Z) \mid \gD_0]$, and (iii) there exist $\delta_n > 0$ and a function $\vartheta\colon (0, \infty) \to [0, \infty)$ such that $\delta \mapsto \vartheta(\delta)/\delta$ is nonincreasing, $\vartheta(\delta_n) \le \sqrt{n}\,\delta_n^2$ and, for every $\delta \ge \delta_n$,
\begin{align*}
	&\E\Big[\sup\big\{\sqrt{n}\,\big(\E_n[g(V, Z) - g'(V, Z)] - \E[g(V, Z) - g'(V, Z) \mid \gD_0]\big) : \\
	&\qquad\qquad g, g' \in \gG,\ \E[(g(V, Z) - g'(V, Z))^2 \mid \gD_0] < \delta^2\big\} \,\Big\vert\, \gD_0\Big] \le \vartheta(\delta).
\end{align*}
Then, there exists a universal constant $c_5 > 0$ such that
\[
	\E[\widehat g(V, Z) - g_0(V, Z) \mid \gD_0] \le c_5\inf_{g \in \gG}\E[g(V, Z) - g_0(V, Z) \mid \gD_0] + c_5\left(\delta_n^2 + \frac{1}{n}\right).
\]
\end{lemma}

In condition (iii) of Lemma~\ref{lem:oracle}, we need a bound on the expected supremum of the empirical process over the pairs $g, g' \in \gG$ with $\E[(g(V, Z) - g'(V, Z))^2 \mid \gD_0] < \delta^2$. For a class of bounded functions with small second moments, the following local maximal inequality provides such a bound in terms of the uniform entropy integral of the class:

\begin{lemma}[{Local Maximal Inequality~\cite[Theorem~2.14.2]{vanderVaart2023}}]\label{lem:maximal}
With the same notations as Lemma~\ref{lem:oracle}, let $\delta \in (0, 1)$, and assume that (i) every $g \in \gG$ satisfies $\abs{g} \le 1$ and (ii) every $g \in \gG$ satisfies $\E[g(V, Z)^2 \mid \gD_0] < \delta^2$. Denote by $\mathcal{J}_\gG(\delta)$ the following uniform entropy integral of $\gG$, in which the supremum is taken over all probability measures $Q$ on $\gV \times \mathcal{Z}$:
\begin{equation}\label{eq:entropy_integral}
	\mathcal{J}_\gG(\delta) \coloneqq \int_0^\delta \sup_Q\sqrt{1 + \gH_2(\gG, \varepsilon, Q)}\,\mathrm{d}\varepsilon.
\end{equation}
Then, there exists a universal constant $c_6 > 0$ such that
\[
	\E\Big[\sup_{g \in \gG}\sqrt{n}\,\abs{\E_n[g(V, Z)] - \E[g(V, Z) \mid \gD_0]} \,\Big\vert\, \gD_0\Big] \le c_6\mathcal{J}_\gG(\delta)\left(1 + \frac{\mathcal{J}_\gG(\delta)}{\delta^2\sqrt{n}}\right).
\]
\end{lemma}

As a result, we obtain the following excess risk bound for vector-valued functions:

\begin{lemma}[Excess Risk Bound]\label{lem:erb}
  Let $\Pi_M\colon \gH_\gY \to \gH_\gY$ be the orthogonal projection onto an $M$-dimensional subspace $\gH_M \subseteq \gH_\gY$. Let $\{e_k\}_{k=1}^M$ be an orthonormal basis for $\gH_M$. Let $\gF^M \subseteq (\gV \to \R^M)$ be a class of functions, identified with a subset of $(\gV \to \gH_M)$ through the isometry $\iota(w) \coloneqq \sum_{k=1}^M w_ke_k$ such that $\sup_{f \in \gF^M}\sup_{v\in \gV}\norm{f(v)}_{\gH_\gY}\leq B$ for some $B>0$. Let $\widehat f$ and $f^\star$ be the minimizers in~\eqref{eq:minimizers} with this isometry, and denote $\widetilde\mu_{Y^1\vert V}(v) \coloneqq \E[\widehat\xi(Z)\vert V = v, \gD_0]$. In addition, assume that (i) there exist $D_\gF, A_\gF \ge 1$ such that, for every probability measure $Q_\gV$ on $\gV$ and every $\varepsilon \in (0,B]$,
	\begin{equation}\label{eq:entropy}
		\gH_2(\gF^M, \varepsilon, Q_\gV) \le D_\gF\log(A_\gF B/\varepsilon),
	\end{equation}
  and (ii) $\norm{\widehat\xi(z)}_{\gH_\gY} \le B'$ for some constant $B' > 0$ and all $z$. Then, there exists a universal constant $c_3 > 0$ such that, for any $n \ge 2$,
	\begin{equation}
		\label{eq:erm}
		\begin{split}
		\E\left[\norm{\widehat{f} - \Pi_M\widetilde\mu_{Y^1\vert V}}^{2}_{L_2(P_V)} \,\middle\vert\, \gD_0\right] &\le c_3\norm{f^\star - \Pi_M\widetilde\mu_{Y^1\vert V}}^{2}_{L_2(P_V)} \\
		&\quad + \frac{c_3(B + B')^{2}D_\gF\log(A_\gF n)}{n}.
		\end{split}
	\end{equation}
\end{lemma}
\begin{proof}[Proof of Lemma~\ref{lem:erb}]
We denote $f_0 \coloneqq \Pi_M\widetilde\mu_{Y^1\vert V}$ and $\mathfrak{L} \coloneqq B + 3B'$. Our main idea of the proof is to apply Lemma~\ref{lem:oracle} to the class $\gG \coloneqq \{\gL_f/\mathfrak{L}^2 : f \in \gF^M\}$ and the function $g_0 \coloneqq \gL_{f_0}/\mathfrak{L}^2$. We first show that conditions (i) and (ii) of Lemma~\ref{lem:oracle} hold. For every $f \in \gF^M$, both $f(v)$ and $f_0(v)$ belong to $\gH_M$, so we have:
\begin{equation}\label{eq:excess_loss}
	\gL_f(v,z) - \gL_{f_0}(v,z) = \inprod{f(v) - f_0(v)}{f(v) + f_0(v) - 2\Pi_M\widehat\xi(z)}_{\gH_\gY}.
\end{equation}
Since the conditional expectation commutes with the bounded linear operator $\Pi_M$, the definitions of $\widetilde\mu_{Y^1\vert V}$ and $f_0$ yield $\E[\Pi_M\widehat\xi(Z)\vert V, \gD_0] = \Pi_M\widetilde\mu_{Y^1\vert V}(V) = f_0(V)$. Taking the expectation of the last display conditionally on $V$ and $\gD_0$ therefore yields:
\begin{equation}\label{eq:excess_identity}
	\E[\gL_f(V, Z) - \gL_{f_0}(V, Z) \mid \gD_0] = \norm{f - f_0}^2_{L_2(P_V)}.
\end{equation}
Next, we derive a uniform bound on $\gL_f$ and a bound on the second moment of $\gL_f - \gL_{f_0}$. For all $v$ and $z$, we have $\norm{f(v)}_{\gH_\gY} \le B$ and $\norm{\widehat\xi(z)}_{\gH_\gY} \le B'$, which imply $0 \le \gL_f(v,z) \le (B + B')^2 \le \mathfrak{L}^2$. Thus, condition (i) of Lemma~\ref{lem:oracle} holds. Moreover, the bound on $\widehat\xi$ implies $\sup_{v\in\gV}\norm{\widetilde\mu_{Y^1\vert V}(v)}_{\gH_\gY} \le B'$. Since $\norm{\Pi_M}_{\text{op}} \le 1$, the bounds on $\widetilde\mu_{Y^1\vert V}$ and $\widehat\xi$ imply $\norm{f_0(v)}_{\gH_\gY} \le B'$ and $\norm{\Pi_M\widehat\xi(z)}_{\gH_\gY} \le B'$. By~\eqref{eq:excess_loss} and the Cauchy-Schwarz inequality, the bounds on $f$, $f_0$ and $\Pi_M\widehat\xi$ imply $\abs{\gL_f(v,z) - \gL_{f_0}(v,z)} \le \mathfrak{L}\norm{f(v) - f_0(v)}_{\gH_\gY}$. Squaring this bound and using~\eqref{eq:excess_identity}, we obtain:
\[
	\E[(\gL_f(V, Z) - \gL_{f_0}(V, Z))^2 \mid \gD_0] \le \mathfrak{L}^2\,\E[\gL_f(V, Z) - \gL_{f_0}(V, Z) \mid \gD_0].
\]
Thus, condition (ii) of Lemma~\ref{lem:oracle} holds. It remains to check condition (iii) of Lemma~\ref{lem:oracle}, which will be done via a covering argument combined with Lemma~\ref{lem:maximal}.

Let $t \in (0, B]$, and let $Q$ be a probability measure on $\gV \times \mathcal{Z}$. By~\eqref{eq:entropy} with $Q_\gV$ as the marginal of $Q$ on $\gV$, there is a $(t, 2)$-cover $\gN_0$ of $\gF^M$ in $L_2(Q_\gV)$ with $\log\abs{\gN_0} \le D_\gF\log(A_\gF B/t)$. Consider the ball $\gB_M(B) \coloneqq \{u \in \gH_M : \norm{u}_{\gH_\gY} \le B\} \subset \gH_M$ and denote by $\Pi_B\colon \gH_M \to \gB_M(B)$ the projection onto the ball, and let $\gN \coloneqq \{\Pi_B\circ h : h \in \gN_0\}$. Since $\norm{\Pi_B(u) - \Pi_B(u')}_{\gH_\gY} \le \norm{u - u'}_{\gH_\gY}$ for $u, u' \in \gH_M$ and $\Pi_B(f(v)) = f(v)$ for every $f \in \gF^M$, the set $\gN$ is also a $(t, 2)$-cover of $\gF^M$ in $L_2(Q_\gV)$.

For maps $f, f'\colon \gV \to \gH_M$ with values in $\gB_M(B)$, the identity~\eqref{eq:excess_loss} with $f_0$ replaced by $f'$ and the Cauchy-Schwarz inequality yield:
\[
	\abs{\gL_f(v,z) - \gL_{f'}(v,z)} \le 2\mathfrak{L}\norm{f(v) - f'(v)}_{\gH_\gY}.
\]
Therefore, the set $\gN_\gL \coloneqq \{\gL_h/\mathfrak{L}^2 : h \in \gN\} \subseteq (\gV \times \mathcal{Z} \to \R)$ forms a $(2t/\mathfrak{L}, 2)$-cover of $\gG$ in $L_2(Q)$ with at most $\abs{\gN_0}$ elements. In particular, given $\varepsilon \in (0, 1]$, we choose $t = \min\{\mathfrak{L}\varepsilon, B\}$ so that the set $\gN_\gL$ is a $(2\varepsilon, 2)$-cover of $\gG$ in $L_2(Q)$ with $\log\abs{\gN_0} \le D_\gF\log(A_\gF/\varepsilon)$.

Let $\delta' \in (0, 4)$ and denote by $\gK \subseteq (\gV \times \mathcal{Z} \to \R)$ the class of the functions $(g - g')/4$ with $g, g' \in \gG$ and $\E[(g(V, Z) - g'(V, Z))^2 \mid \gD_0] < \delta'^2$. It is easy to verify that the conditions (i) and (ii) of Lemma~\ref{lem:maximal} with $\delta = \delta'/4$ are satisfied for the class $\gK$. In addition, the set $\{(h - h')/4 : h, h' \in \gN_\gL\}$ forms an $(\varepsilon, 2)$-cover of $\gK$ in $L_2(Q)$ with at most $\abs{\gN_0}^2$ elements. Consequently, $\gH_2(\gK, \varepsilon, Q) \le 2\log\abs{\gN_0} \le 2D_\gF\log(A_\gF/\varepsilon)$. Therefore, the uniform entropy integral $\mathcal{J}_\gK(\delta'/4)$ defined in~\eqref{eq:entropy_integral} satisfies:
\[
	\mathcal{J}_\gK(\delta'/4) \le \int_0^{\delta'/4}\sqrt{1 + 2D_\gF\log\frac{A_\gF}{\varepsilon}}\,\mathrm{d}\varepsilon \le \frac{\delta'}{\sqrt{2}}\sqrt{D_\gF\log\frac{4eA_\gF}{\delta'}}.
\]
For $\delta' \ge 1/\sqrt{n}$, the right-hand side of the last display is at most $(\delta'/\sqrt{2})\sqrt{D_\gF\log(4eA_\gF n)}$. Thus, according to Lemma~\ref{lem:maximal} for the class $\gK$ with $\delta = \delta'/4$, this bound and conditions (i) and (ii) of this lemma yield a universal constant $c_4 \ge 1$ such that, for $1/\sqrt{n} \le \delta' < 4$,
\begin{align*}
	&\E\Big[\sup\big\{\sqrt{n}\,\big(\E_n[g(V, Z) - g'(V, Z)] - \E[g(V, Z) - g'(V, Z) \mid \gD_0]\big) : \\
	&\qquad\qquad g, g' \in \gG,\ \E[(g(V, Z) - g'(V, Z))^2 \mid \gD_0] < \delta'^2\big\} \,\Big\vert\, \gD_0\Big] \\
	&\qquad \le c_4\left(\delta'\sqrt{D_\gF\log(4eA_\gF n)} + \frac{D_\gF\log(4eA_\gF n)}{\sqrt{n}}\right).
\end{align*}
Since $0 \le g \le 1$ for every $g \in \gG$, the supremum on the left-hand side of the last display for $\delta' > 2$ is the same as the supremum for $\delta' = 2$. Moreover, the right-hand side of the last display is nondecreasing in $\delta'$. Thus, condition (iii) in Lemma~\ref{lem:oracle} holds with $\vartheta(\delta') = c_4(\delta'\sqrt{D_\gF\log(4eA_\gF n)} + D_\gF\log(4eA_\gF n)/\sqrt{n})$ and $\delta_n = (2c_4 + 1)\sqrt{D_\gF\log(4eA_\gF n)/n}$, which satisfies $\delta_n \ge 1/\sqrt{n}$ and $\vartheta(\delta_n) \le \sqrt{n}\,\delta_n^2$.

By~\eqref{eq:excess_identity}, the law of iterated expectations, conditionally on $\gD_1$, and Lemma~\ref{lem:oracle}, we obtain:
\begin{align*}
	\E[\norm{\widehat f - f_0}^2_{L_2(P_V)} \mid \gD_0] &= \E[\gL_{\widehat f}(V, Z) - \gL_{f_0}(V, Z) \mid \gD_0] \\
    &\le c_5\inf_{f \in \gF^M}\E[\gL_f(V, Z) - \gL_{f_0}(V, Z) \mid \gD_0] + c_5\mathfrak{L}^2\left(\delta_n^2 + \frac{1}{n}\right) \\
    &= c_5\norm{f^\star - f_0}^2_{L_2(P_V)} + c_5\mathfrak{L}^2\left(\delta_n^2 + \frac{1}{n}\right).
\end{align*}
Since $\mathfrak{L} \le 3(B + B')$, $D_\gF, A_\gF \ge 1$ and $n \ge 2$, we obtain~\eqref{eq:erm} for a universal constant $c_3$.
\end{proof}

\subsection{Proof of the Bound}

We now set up for the proof of the upper bound. Let $\mu^{\star}_{\widehat{\xi}}\coloneqq \argmin_{\mu \in \gF_{\mathrm{NK}}}\E[\norm{\mu(V) - \widehat{\xi}(Z)}^{2}_{\gH_{\gY}}\vert \mathcal{D}_{0}]$ be the best-in-class estimator. Let $\Pi_{M}: \gH_\gY \to \gH_\gY$ be the orthogonal projection operator onto the subspace $\gH_{M}$ and $\widetilde\mu_{Y^1\vert V}(v) \coloneqq \E[\widehat{\xi}(Z)\vert V = v, \mathcal{D}_0]$. By the triangle inequality:
\begin{equation}\label{eq:nk_triangle}
\begin{aligned}
\E[\norm{\widehat{\mu}_{\mathrm{NK}}(V) - \mu_{Y^1\vert V}(V)}^{2}_{\gH_{\gY}}] &\lesssim \E\left[\norm{\widehat{\mu}_{\mathrm{NK}}(V) - \Pi_M\widetilde\mu_{Y^1\vert V}(V)}^{2}_{\gH_{\gY}}\right] \\
&\quad + \E\left[\norm{\Pi_M\widetilde\mu_{Y^1\vert V}(V) - \Pi_M\mu_{Y^1\vert V}(V)}^{2}_{\gH_{\gY}}\right] \\
&\quad + \E\left[\norm{\Pi_M \mu_{Y^1\vert V}(V) - \mu_{Y^1\vert V}(V)}^{2}_{\gH_{\gY}}\right].
\end{aligned}
\end{equation}
Let $c_3 > 0$ be the universal constant in the bound~\eqref{eq:erm} of Lemma~\ref{lem:erb}. We define the statistical error:
\[
\mathbf{A} \coloneqq \E\left[\norm{\widehat{\mu}_{\mathrm{NK}} - \Pi_M\widetilde\mu_{Y^1\vert V}}^2_{L_2(P_V)} - c_3\norm{\mu^\star_{\widehat{\xi}} - \Pi_M\widetilde\mu_{Y^1\vert V}}^2_{L_2(P_V)}\right].
\]
Let $\mu \in \gF_{\mathrm{NK}}$. Since the conditional expectation commutes with $\Pi_M$, we have $\E[\Pi_M\widehat\xi(Z)\vert V, \gD_0] = \Pi_M\widetilde\mu_{Y^1\vert V}(V)$. This identity and the fact that $\mu - \Pi_M\widetilde\mu_{Y^1\vert V}$ takes values in $\gH_M$ imply that the cross term in the following expansion vanishes:
\begin{align*}
\E[\norm{\mu(V) - \widehat{\xi}(Z)}^{2}_{\gH_{\gY}}\vert \mathcal{D}_{0}] &= \E[\norm{\mu(V) - \Pi_M\widetilde\mu_{Y^1\vert V}(V) +   \Pi_M\widetilde\mu_{Y^1\vert V}(V) - \widehat\xi(Z)}^{2}_{\gH_{\gY}}\vert\mathcal{D}_0] \\ 
&= \norm{\mu - \Pi_M\widetilde\mu_{Y^1\vert V}}^{2}_{L_2(P_V)} + \E[\norm{ \Pi_M\widetilde\mu_{Y^1\vert V}(V) - \widehat\xi(Z)}^{2}_{\gH_{\gY}}\vert\mathcal{D}_0].
\end{align*}
Since the last term on the right-hand side does not depend on $\mu$, the definition of $\mu^\star_{\widehat\xi}$ yields $\mu^\star_{\widehat{\xi}} \in \argmin_{\mu \in \gF_{\mathrm{NK}}}\norm{\mu - \Pi_M\widetilde\mu_{Y^1\vert V}}^{2}_{L_2(P_V)}$. Consequently,
\begin{align*}
\norm{\mu^\star_{\widehat{\xi}} - \Pi_M\widetilde\mu_{Y^1\vert V}}^{2}_{L_2(P_V)} &= \inf_{\mu \in \gF_{\mathrm{NK}}}\norm{\mu - \Pi_M\widetilde\mu_{Y^1\vert V}}^{2}_{L_2(P_V)} \\
&\lesssim \inf_{\mu \in \gF_{\mathrm{NK}}} \norm{\mu - \Pi_M\mu_{Y^1\vert V}}^{2}_{L_2(P_V)} + \norm{\Pi_M\mu_{Y^1\vert V} - \Pi_M\widetilde\mu_{Y^1\vert V}}^{2}_{L_2(P_V)}.
\end{align*}
We insert the definition of $\mathbf{A}$ and the last display into the right-hand side of~\eqref{eq:nk_triangle}. Taking the expectation over $\mathcal{D}_0$ then yields the following decomposition, with a constant $c > 0$ that only depends on $c_3$:
\begin{equation}\label{eq:MSE_decom}
\begin{split}
	&\E[\norm{\widehat{\mu}_{\mathrm{NK}}(V) - \mu_{Y^1\vert V}(V)}^{2}_{\gH_{\gY}}] \\
	&\quad \le 3\underbrace{{\E\left[\norm{\widehat{\mu}_{\mathrm{NK}} - \Pi_M\widetilde\mu_{Y^1\vert V}}^2_{L_2(P_V)} - c_3\norm{\mu^\star_{\widehat{\xi}} - \Pi_M\widetilde\mu_{Y^1\vert V}}^2_{L_2(P_V)}\right]}}_{\text{$\mathbf{A}$: Statistical Error}} \\
	&\qquad + c\,\underbrace{\inf_{\mu \in \gF_{\mathrm{NK}}}\E\left[\norm{\mu(V) - \Pi_M\mu_{Y^1\vert V}(V)}^2_{\gH_{\gY}}\right]}_{\text{$\rmB$: Approximation Error}} \\
	                                                       &\qquad + c\,\underbrace{\E\left[\norm{\Pi_M\widetilde\mu_{Y^1\vert V}(V) - \Pi_M\mu_{Y^1\vert V}(V)}^2_{\gH_{\gY}}\right]}_{\text{$\mathbf{C}$: Nuisance Error}} \\
	                                                       &\qquad + c\,\underbrace{\E\left[\norm{\Pi_M \mu_{Y^1\vert V}(V) - \mu_{Y^1\vert V}(V)}^2_{\gH_{\gY}}\right]}_{\text{$\mathbf{D}$: Projection Error}}.
\end{split}
\end{equation}
We now bound each of these terms.

\paragraph{Term $\mathbf{A}$: The Statistical Error.}

The excess risk bound of Lemma~\ref{lem:erb} requires an isometry from $\R^M$ onto $\gH_M$. Since the map $w \mapsto \sum_{j=1}^M w_j\phi(\widetilde{y}_j)$ is not an isometry, we first change the coordinates of the networks. We write
\[
	\widehat\mu_{\mathrm{NK}}(v) = \sum_{j=1}^{M} \widehat{f}_{j}(v) \phi(\widetilde{y}_{j}), \qquad \mu^{\star}_{\widehat\xi}(v) = \sum_{j=1}^{M} f^{\star}_{j}(v) \phi(\widetilde{y}_{j}),
\]
where $f^\star \in \gF_{\mathrm{FC}}$ and $\widehat f \coloneqq f_{\widehat\theta}$ for the minimizer $\widehat\theta$ of~\eqref{eq:NK_loss} over $\Theta$. By Lemma~\ref{lem:Llamb}, every $f \in \gF_{\mathrm{NN}}(\mathbf{d}, \infty)$ satisfies $\sum_{j=1}^M f_j(v)\phi(\widetilde{y}_j) = \sum_{k=1}^M \psi_k(v)e_k$ with $\psi \coloneqq \rmK_M^{1/2}f$.

For $f \in \gF_{\mathrm{NN}}(\mathbf{d}, \infty)$ with the output layer $(\rmW_L, \mathfrak{b}_L)$, the map $\psi = \rmK_M^{1/2}f$ is the network with the same hidden layers as $f$ and the output layer $(\rmK_M^{1/2}\rmW_L, \rmK_M^{1/2}\mathfrak{b}_L)$. Since $\rmK_M$ is invertible, the map $(\rmW_L, \mathfrak{b}_L) \mapsto (\rmK_M^{1/2}\rmW_L, \rmK_M^{1/2}\mathfrak{b}_L)$ is a bijection. By the identity of Lemma~\ref{lem:Llamb} and the orthonormality of $\{e_k\}_{k=1}^M$, we have $\norm{\sum_{j=1}^M f_j(v)\phi(\widetilde{y}_j)}_{\gH_\gY} = \norm{\psi(v)}$ for all $v \in \gV$. The bijection between the output layers, the identity of Lemma~\ref{lem:Llamb} and this norm identity yield:
\begin{equation}\label{eq:nk_classes}
\{\rmK_M^{1/2}f : f \in \gF_{\mathrm{FC}}\} = \gF_{\mathrm{NN}}(\mathbf{d}, B_\gF), \qquad \gF_{\mathrm{NK}} = \Big\{\sum_{k=1}^M \psi_k(\cdot)e_k : \psi \in \gF_{\mathrm{NN}}(\mathbf{d}, B_\gF)\Big\}.
\end{equation}
By Lemma~\ref{lem:Llamb}, the coordinates of $\widehat\mu_{\mathrm{NK}}(v)$ and $\mu^\star_{\widehat\xi}(v)$ in the orthonormal basis $\{e_k\}_{k=1}^M$ are $\widehat\psi(v) \coloneqq \rmK_M^{1/2}\widehat{f}(v)$ and $\psi^\star(v) \coloneqq \rmK_M^{1/2}f^\star(v)$. In these coordinates, that is, with the isometry $\iota(w) = \sum_{k=1}^M w_ke_k$, the loss~\eqref{eq:NK_loss} satisfies:
\[
\widehat\gL_{\mathrm{NK}}(\theta) = \E_n[\gL_{\iota\circ\rmK_M^{1/2}f_\theta}(V, Z)] - \E_n[\norm{\widehat\xi(Z)}^2_{\gH_\gY}].
\]
The last term on the right-hand side does not depend on $\theta$. Consequently, by~\eqref{eq:nk_classes}, $\widehat\psi \in \argmin_{\psi \in \gF_{\mathrm{NN}}(\mathbf{d}, B_\gF)}\E_n[\gL_{\iota\circ\psi}(V, Z)]$. Similarly, by the definition of $\mu^\star_{\widehat\xi}$ and~\eqref{eq:nk_classes}, we have $\psi^\star \in \argmin_{\psi \in \gF_{\mathrm{NN}}(\mathbf{d}, B_\gF)}\E[\gL_{\iota\circ\psi}(V, Z) \mid \gD_0]$. Thus, $\widehat\psi$ and $\psi^\star$ are the minimizers in~\eqref{eq:minimizers} for the class $\gF_{\mathrm{NN}}(\mathbf{d}, B_\gF)$.

We now verify the two conditions of Lemma~\ref{lem:erb}. Lemma~\ref{lem:dense_entropy} with $B = B_\gF$ yields the metric entropy bound~\eqref{eq:entropy} for the class $\gF_{\mathrm{NN}}(\mathbf{d}, B_\gF)$ with $D_\gF = c_{\mathrm{FC}} WL\log W$ and $A_\gF = 4\sqrt{M}$. Thus, condition (i) is satisfied. Since Lemma~\ref{lem:HS_bound} yields $\norm{\widehat\xi(z)}_{\gH_\gY} \le \sqrt{B_{\widehat\xi}}$ for all $z \in \mathcal{Z}$, condition (ii) is also satisfied.

We are now in a position to apply Lemma~\ref{lem:erb}. With $B = B_\gF$, $B' = \sqrt{B_{\widehat\xi}}$ and $\gF^M = \gF_{\mathrm{NN}}(\mathbf{d}, B_\gF)$, we obtain:
\[
	\E\left[\norm{\widehat\mu_{\mathrm{NK}} - \Pi_M\widetilde\mu_{Y^1\vert V}}^{2}_{L_2(P_V)} \,\middle\vert\, \gD_0\right] - c_3\norm{\mu^\star_{\widehat{\xi}} - \Pi_M\widetilde\mu_{Y^1\vert V}}^{2}_{L_2(P_V)} \lesssim \frac{WL \log W \log(Mn)}{n}.
\]
Thus, after taking the expectation over $\gD_0$, we obtain a bound for $\mathbf{A}$:
\[
	\mathbf{A} \lesssim \frac{WL\log W \log(Mn)}{n}.
\]
\paragraph{Term $\rmB$: The Approximation Error.}

As in the proof for Term~$\rmB$ of the Deep Feature estimator in Appendix~\ref{sec:df_proof}, we apply Lemma~\ref{lem:approx}, whose only nontrivial condition is $D \ge \bar w M$ where $D \coloneqq \min_{1 \le i \le L-1} d_i$. Our choices of $M$ and $WL$ in Case 1 yield $WL/M^2 \asymp n^{a/(ab+a+2b)}$ and those in Case 2 yield $WL/M^2 \asymp n^{1/(b+1)}(\log n)^{-4d_y}$. Since $L$ does not depend on $n$, we have $W/M^2 \to \infty$ as $n \to \infty$ in both cases. We then proceed as with the Deep Feature estimator to obtain $W \lesssim LD^2$ and $D/M \to \infty$ as $n \to \infty$. Consequently, Lemma~\ref{lem:approx} yields a network $\psi_c \in \gF_{\mathrm{NN}}(\mathbf{d}, B_\gF)$ such that, as in Appendix~\ref{sec:df_proof},
\[
	\E\left[\Bigl\lVert \sum_{k=1}^M \psi_{c,k}(V)e_k - \Pi_M\mu_{Y^1\vert V}(V)\Bigr\rVert^2_{\gH_\gY}\right] \lesssim \left(\frac{M}{DL}\right)^{2b} \lesssim \left(\frac{M^2}{WL}\right)^{b}.
\]
By~\eqref{eq:nk_classes}, the map $\sum_{k=1}^M \psi_{c,k}(\cdot)e_k$ belongs to $\gF_{\mathrm{NK}}$. Consequently,
\begin{align*}
	\rmB = \inf_{\mu \in \gF_{\mathrm{NK}}}\E\left[\norm{\mu(V) - \Pi_M\mu_{Y^1\vert V}(V)}^2_{\gH_\gY}\right] &\le \E\left[\Bigl\lVert \sum_{k=1}^M \psi_{c,k}(V)e_k - \Pi_M\mu_{Y^1\vert V}(V)\Bigr\rVert^2_{\gH_\gY}\right] \\
	&\lesssim \left(\frac{M^2}{WL}\right)^{b}.
\end{align*}
\paragraph{Term $\mathbf{C}$: The Nuisance Error.}

By the non-expansiveness of the projection operator, we have
\begin{align*}
	\mathbf{C} &= \E\left[\norm{\Pi_M\widetilde\mu_{Y^1\vert V}(V) - \Pi_M\mu_{Y^1\vert V}(V)}^{2}_{\gH_{\gY}}\right] \\ &\le \E\left[\norm{\widetilde\mu_{Y^1\vert V}(V) - \mu_{Y^1\vert V}(V)}^{2}_{\gH_{\gY}}\right].
\end{align*}
Since the assumptions of the theorem include $\pi(x), \widehat\pi(x) \in [\varepsilon, 1)$ and $\sup_{x\in\gX}\norm{\widehat\mu_0(x)}_{\gH_\gY} \le B_{\mu_0}$, we can use Lemma~\ref{lem:nuisance_bdd} to bound the right-hand side of the last display:
\[
	\mathbf{C}\lesssim \min\left\{\gR^{2}_{\pi}(\widehat\pi), \gR^{2}_{\mu_{0}}(\widehat\mu_{0})\right\}.
\]
\paragraph{Term $\mathbf{D}$: The Projection Error.}

As in the proof for Term~$\mathbf{D}$ of the Deep Feature estimator in Appendix~\ref{sec:df_proof}, we apply Lemma~\ref{lem:projbound}. The only difference is that the grid points $\{\widetilde{y}_j\}_{j=1}^M$ of Assumption~\ref{ass:grid} directly satisfy the fill distance requirement $h \asymp M^{-1/d_y}$ of Lemma~\ref{lem:projbound}. Consequently, in Case~1 of the theorem,
\[
    \mathbf{D}= \E \left[\norm{\Pi_M \mu_{Y^1\vert V}(V) - \mu_{Y^1\vert V}(V)}^2_{\gH_{\gY}}\right] \lesssim M^{-2\tau/d_y}\,\E\left[\norm{\fp^1(\cdot\vert V)}^2_{L^2(\gY)}\right] \lesssim M^{-a}\text{.}
\]
In Case~2 of the theorem, Lemma~\ref{lem:projbound} yields the following bound for some $B_0>0$:
\begin{equation}\label{eq:nk_gauss}
    \mathbf{D} \lesssim \exp\left( - B_0 M^{\frac{1}{2d_y}} \log M \right).
\end{equation}
\paragraph{Combining the bounds.}

In Case~1 of the theorem, we obtain the following inequality from~\eqref{eq:MSE_decom} and the bounds of Terms~$\mathbf{A}$, $\rmB$, $\mathbf{C}$ and~$\mathbf{D}$:
\begin{align*}
	\E[\norm{\widehat\mu_{\mathrm{NK}}(V) - \mu_{Y^1\vert V}(V)}^{2}_{\gH_{\gY}}] & \le 3\mathbf{A} + c(\rmB + \mathbf{C} + \mathbf{D}) \\
	& \lesssim \frac{WL \log W \log(Mn)}{n}+ \left(\frac{M^2}{WL}\right)^{b}+M^{-a} \\
	&\quad +  \min\left\{\gR^{2}_{\pi}(\widehat \pi), \gR^{2}_{\mu_{0}}(\widehat\mu_{0}) \right\}.
\end{align*}
Recall our choices of $M \asymp n^{\frac{b}{ab+a+2b}}$ and $WL \asymp n^{\frac{a+2b}{ab+a+2b}}$. These choices yield $WL/n \asymp (M^2/(WL))^b \asymp M^{-a} \asymp n^{-\frac{ab}{ab+a+2b}}$. Therefore, we obtain the bound:
\[
	\E[\norm{\widehat\mu_{\mathrm{NK}}(V) - \mu_{Y^1\vert V}(V)}^{2}_{\gH_{\gY}}] \lesssim_{\log n} n^{-\frac{ab}{ab+a+2b}}+ \min\left\{\gR^{2}_{\pi}(\widehat \pi), \gR^{2}_{\mu_{0}}(\widehat\mu_{0})\right\}.
\]
In Case~2, we instead obtain from \eqref{eq:MSE_decom}, the bounds of Terms~$\mathbf{A}$, $\rmB$ and~$\mathbf{C}$, and \eqref{eq:nk_gauss} that:
\begin{align*}
	\E[&\norm{\widehat\mu_{\mathrm{NK}}(V) - \mu_{Y^1\vert V}(V)}^{2}_{\gH_{\gY}}]   \\
	& \lesssim \frac{WL \log W \log(Mn)}{n}+ \left(\frac{M^2}{WL}\right)^{b}+\exp\left( - B_0 M^{\frac{1}{2d_y}} \log M \right) + \min\left\{\gR^{2}_{\pi}(\widehat \pi), \gR^{2}_{\mu_{0}}(\widehat\mu_{0}) \right\}.
\end{align*}
Recall our choices of $M \asymp (\log n)^{2d_y}$ and $WL \asymp n^{\frac{1}{b+1}}$. As in the proof of Theorem~\ref{thm:DF_upper_bdd}, we have $\exp( - B_0 M^{\frac{1}{2d_y}} \log M ) \le n^{-1}$ for all sufficiently large $n$. These choices also yield $\frac{WL\log W\log(Mn)}{n} + (M^2/(WL))^b \lesssim_{\log n} n^{-\frac{b}{b+1}}$. Therefore, we obtain the bound:
\[
	\E[\norm{\widehat\mu_{\mathrm{NK}}(V) - \mu_{Y^1\vert V}(V)}^{2}_{\gH_{\gY}}] \lesssim_{\log n} n^{-\frac{b}{b+1}} + \min\left\{\gR^{2}_{\pi}(\widehat \pi), \gR^{2}_{\mu_{0}}(\widehat\mu_{0})\right\}.
\]
\hfill $\blacksquare$

\section{Miscellaneous Lemmas} \label{sec:misc_proofs}

\begin{lemma}[{Lemma~\ref{lem:nuisance_bdd}, restated}]\label{lem:nuisance_bdd_restate}
Let Assumptions~\ref{ass:unconfound} and~\ref{ass:boundk} hold. In addition, assume that $\widehat\pi(x), \pi(x) \in [\varepsilon, 1)$ for all $x \in \gX$, and that $\sup_{x\in\gX}\norm{\widehat\mu_{0}(x)}_{\gH_{\gY}} \le B_{\mu_0}$ for some constant $B_{\mu_0} > 0$. Then we have the following bound:
\begin{equation*}
 \E\left[\norm{\E[\widehat{\xi}(Z)\vert V, \mathcal{D}_0] - \mu_{Y^1\vert V}(V)}^{2}_{\gH_{\gY}}\right] \lesssim \min\left\{\gR^{2}_{\pi}(\widehat\pi), \gR^{2}_{\mu_{0}}(\widehat\mu_{0})\right\}.
\end{equation*}
\end{lemma}

\begin{proof}
By the identification result~\eqref{eq:identification}, the independence of $Z$ and $\gD_{0}$ and the definitions of $\widehat\xi$ and $\xi$,
\begin{align*}
	\E[\widehat{\xi}(Z)\vert V, \mathcal{D}_{0}] - \mu_{Y^1\vert V}(V) & = \E[\widehat{\xi}(Z)\vert V, \mathcal{D}_{0}] - \E[\xi(Z) \vert V]                                                           \\
	& =\E[\widehat{\xi}(Z) - \xi(Z) \vert V, \mathcal{D}_{0}]                                                                           \\
	& = \E\left[\left(\frac{A}{\widehat{\pi}(X)}- \frac{A}{\pi(X)}\right)(\phi(Y) - \mu_{0}(X)) \Big\vert V, \mathcal{D}_{0}\right]   \\
	& \quad - \E\left[\left(\frac{A}{\widehat{\pi}(X)}-1\right)(\widehat\mu_{0}(X) - \mu_{0}(X)) \Big\vert V, \mathcal{D}_{0}\right].
\end{align*}
Since $V = \eta(X)$ is a function of $X$ and $\E[A\phi(Y)\vert X] = \pi(X)\E[\phi(Y)\vert X, A=1] = \pi(X)\mu_{0}(X) = \E[A\mu_0(X)\vert X]$, the law of iterated expectations, conditionally on $(X, \gD_0)$, yields that the first term on the right-hand side is zero. Thus, using $\E[A \vert X] = \pi(X)$, $\widehat{\pi}(X) \ge \varepsilon$, Jensen's inequality and the Cauchy-Schwarz inequality, we obtain:
\begin{align}
	\E&\left[\norm{\E[\widehat{\xi}(Z)\vert V, \mathcal{D}_0] - \mu_{Y^1\vert V}(V)}^{2}_{\gH_{\gY}}\right] \nonumber \\
	& \lesssim \E\left[\E\left[\abs{\widehat\pi(X) - \pi(X)}\cdot \left\lVert \widehat\mu_{0}(X) - \mu_{0}(X)\right\rVert_{\gH_\gY}\Big\vert V, \gD_{0}\right]^{2}\right] \nonumber \\
	           & \le \E\left[\E\left[ \abs{\widehat\pi(X) - \pi(X)}^{2}\Big\vert V, \gD_{0}\right] \E\left[\norm{\widehat\mu_{0}(X) - \mu_{0}(X)}^{2}_{\gH_\gY}\Big\vert V, \gD_{0}\right]\right]. \label{eq:nuisance_cs}
\end{align}
We now bound the first factor on the right-hand side of~\eqref{eq:nuisance_cs} by a constant. Since both $\pi(X) < 1$ and $\widehat\pi(X) < 1$, this factor is at most $1$. Thus, \eqref{eq:nuisance_cs} yields:
\begin{equation}
	\label{eq:pibdd1}\E\left[\norm{\E[\widehat{\xi}(Z)\vert V, \mathcal{D}_0] - \mu_{Y^1\vert V}(V)}^{2}_{\gH_{\gY}}\right] \lesssim \E \left[ \E\left[\norm{\widehat\mu_{0}(X) - \mu_{0}(X)}^{2}_{\gH_\gY}\Big\vert V, \gD_{0}\right]\right] = \gR^{2}_{\mu_{0}}(\widehat\mu_{0}).
\end{equation}
Alternatively, we bound the second factor on the right-hand side of~\eqref{eq:nuisance_cs} by a constant. Let $Y'$ be a copy of $Y$ that is independent of $Y$ given $X$ and $A = 1$. Then,
\begin{align*}
	\lVert \mu_{0}(X) \rVert^{2}_{\gH_{\gY}} & = \left\lan \mu_{0}(X), \mu_{0}(X) \right\ran_{\gH_\gY} \\
	&= \left\lan \E[\phi(Y) \mid X, A = 1],\E[\phi(Y') \mid X, A = 1]\right\ran_{\gH_\gY} \\
	                                                & = \E[\lan \phi(Y), \phi(Y') \ran_{\gH_\gY} \mid X, A = 1] = \E[k_\gY(Y,Y') \mid X, A = 1] \le B_k\text{.}
\end{align*}
Moreover, by assumption, $\norm{\widehat\mu_{0}(x)}_{\gH_{\gY}} \le B_{\mu_0}$ for all $x$. Consequently, $\E[\norm{\widehat\mu_{0}(X) - \mu_{0}(X)}^{2}_{\gH_\gY}\vert V, \gD_0] \le (B_{\mu_0} + \sqrt{B_k})^2$ almost surely. Therefore, \eqref{eq:nuisance_cs} yields:
\begin{equation}
	\label{eq:mubdd1}\E\left[\norm{\E[\widehat{\xi}(Z)\vert V, \mathcal{D}_0] - \mu_{Y^1\vert V}(V)}^{2}_{\gH_{\gY}}\right] \lesssim \E \left[ \E\left[ \abs{\widehat\pi(X) - \pi(X)}^{2}\Big\vert V, \gD_{0}\right] \right] = \gR^{2}_{\pi}(\widehat\pi).
\end{equation}
Combining \eqref{eq:pibdd1} and \eqref{eq:mubdd1} yields:
\[
	\E\left[\norm{\E[\widehat{\xi}(Z)\vert V, \mathcal{D}_0] - \mu_{Y^1\vert V}(V)}^{2}_{\gH_{\gY}}\right]\lesssim \min\left\{\gR^{2}_{\pi}(\widehat\pi), \gR^{2}_{\mu_{0}}(\widehat\mu_{0})\right\}.
\]
\end{proof}
\begin{proof}[{Proof of \eqref{eq:l4bound}}] We also derive the product bound \eqref{eq:l4bound} in Remark~\ref{rem:product}. We apply the Cauchy-Schwarz inequality for the outer expectation on the right-hand side of~\eqref{eq:nuisance_cs}, followed by Jensen's inequality:
\[
\E\left[\norm{\E[\widehat{\xi}(Z)\vert V, \mathcal{D}_0] - \mu_{Y^1\vert V}(V)}^{2}_{\gH_{\gY}}\right] \lesssim \E\left[(\widehat\pi(X) - \pi(X))^{4}\right]^{1/2}\E\left[\norm{\widehat\mu_{0}(X) - \mu_{0}(X)}^{4}_{\gH_\gY}\right]^{1/2}.
\]
\end{proof}

\begin{lemma}[Sobolev Regularity of Projection Coefficients] \label{lem:sobolevcoef}
Let Assumptions~\ref{ass:boundk}, \ref{ass:domain} and~\ref{ass:smooth_density} hold. Let $\Pi_M\colon \gH_\gY \to \gH_\gY$ be the orthogonal projection onto a subspace $\gH_M \subseteq \gH_\gY$ with an orthonormal basis $\{e_k\}_{k=1}^M$, and denote $c_k(v) \coloneqq \int_\gY e_k(y)\fp^1(y\vert v)\,\mathrm{d}y$ for $v \in \gV$, so that $\Pi_M\mu_{Y^1\vert V}(v) = \sum_{k=1}^M c_k(v)e_k$. Denote by $\abs{\gY}$ the Lebesgue measure of $\gY$ and by $N_r$ the number of multi-indices $\alpha \in \mathbb{N}_0^{d_v}$ with $\abs{\alpha} \le r$. Then, the following inequality holds:
\[
	\sum_{k=1}^M \norm{c_k}^2_{W^{r,2}(\gV)} \le N_rB_k\abs{\gY} \int_\gY \norm{\fp^1(y\vert \cdot)}^2_{W^{r,2}(\gV)}\,\mathrm{d}y.
\]
\end{lemma}

\begin{proof}
Let $q \in L^2(\gY)$, and denote $\mu_q \coloneqq \int_\gY \phi(y)q(y)\,\mathrm{d}y \in \gH_\gY$. By the reproducing property, $\inprod{\mu_q}{h}_{\gH_\gY} = \int_\gY h(y)q(y)\,\mathrm{d}y$ for all $h \in \gH_\gY$. By Assumption~\ref{ass:boundk} and the Cauchy-Schwarz inequality, we have:
\[
	\norm{\mu_q}^2_{\gH_\gY} = \int_\gY\int_\gY k_\gY(y,y')q(y)q(y')\,\mathrm{d}y\,\mathrm{d}y' \le B_k \left(\int_\gY \abs{q(y)}\,\mathrm{d}y\right)^2 \le B_k\abs{\gY}\norm{q}^2_{L^2(\gY)}.
\]
Thus, Bessel's inequality for the orthonormal set $\{e_k\}_{k=1}^M$ and the last display yield:
\begin{equation}\label{eq:bessel_coef}
	\sum_{k=1}^M \left(\int_\gY e_k(y)q(y)\,\mathrm{d}y\right)^2 = \sum_{k=1}^M \inprod{\mu_q}{e_k}^2_{\gH_\gY} \le \norm{\mu_q}^2_{\gH_\gY} \le B_k\abs{\gY}\norm{q}^2_{L^2(\gY)}.
\end{equation}
Next, we apply~\eqref{eq:bessel_coef} to the coefficients $c_k$. Since $\fp^1(\cdot\vert v)$ is the conditional density of $Y^1$ given $V = v$, the CCME can be written as the following Bochner integral:
\begin{equation}\label{eq:ccme_density}
	\mu_{Y^1\vert V}(v) = \E[\phi(Y^1)\vert V = v] = \int_\gY \phi(y)\fp^1(y\vert v)\,\mathrm{d}y.
\end{equation}
Thus, $\mu_{Y^1\vert V}(v) = \mu_{\fp^1(\cdot\vert v)}$. By the reproducing property, the coefficients of $\Pi_M\mu_{Y^1\vert V}(v)$ are then $\inprod{\mu_{Y^1\vert V}(v)}{e_k}_{\gH_\gY} = \int_\gY e_k(y)\fp^1(y\vert v)\,\mathrm{d}y = c_k(v)$. For a multi-index $\alpha$ with $\abs{\alpha} \le r$, we apply the derivative operator $D^\alpha_v$. Under Assumption~\ref{ass:smooth_density}, and since $\abs{e_k(y)} \le \sqrt{B_k}$, we can differentiate under the integral sign:
\[
	D^\alpha_v c_k(v) = \int_\gY e_k(y)\,D^\alpha_v\fp^1(y\vert v)\,\mathrm{d}y.
\]
Thus, applying~\eqref{eq:bessel_coef} with $q = D^\alpha_v\fp^1(\cdot\vert v)$, integrating over $v \in \gV$, summing over $\abs{\alpha} \le r$, and applying Fubini's theorem and the Cauchy-Schwarz inequality to the sum over $\alpha$ yields:
\begin{align*}
	\sum_{k=1}^M \norm{c_k}^2_{W^{r,2}(\gV)} &\le N_rB_k\abs{\gY}\int_\gV\sum_{\abs{\alpha}\le r}\int_\gY \abs{D^\alpha_v\fp^1(y\vert v)}^2\,\mathrm{d}y\,\mathrm{d}v \\
	&\le N_rB_k\abs{\gY}\int_\gY \norm{\fp^1(y\vert\cdot)}^2_{W^{r,2}(\gV)}\,\mathrm{d}y.
\end{align*}
\end{proof}

\begin{lemma}[Approximation by a Fully Connected Network]\label{lem:approx}
Let Assumptions~\ref{ass:boundk}, \ref{ass:domain} and~\ref{ass:smooth_density} hold, and denote $b \coloneqq 2r/d_v$. Let $\Pi_M\colon \gH_\gY \to \gH_\gY$ be the orthogonal projection onto a subspace $\gH_M \subseteq \gH_\gY$ with an orthonormal basis $\{e_k\}_{k=1}^M$. Then, there exist constants $L_0$ and $\bar w \ge 2$ that only depend on $r$ and $d_v$, a constant $\bar B > 0$ that only depends on $\gV$, $\gY$, $r$, $B_k$ and $\fp^1$, and a constant $C > 0$ that does not depend on $M$ or $\gH_M$, such that, for any $B_\gF^2 \ge \bar B$, any $L \ge \max\{L_0, 3\}$ and any widths $\mathbf{d} = (d_v, d_1, \ldots, d_{L-1}, M)$ of the networks~\eqref{eq:nn_class} whose smallest hidden width $D \coloneqq \min_{1 \le i \le L-1} d_i$ satisfies $D \ge \bar w M$, there is $\psi_c \in \gF_{\mathrm{NN}}(\mathbf{d}, B_\gF)$ with
\[
	\E\left[\Bigl\lVert \sum_{k=1}^M \psi_{c,k}(V)e_k - \Pi_M\mu_{Y^1\vert V}(V)\Bigr\rVert^2_{\gH_\gY}\right] \le C\left(\frac{M}{DL}\right)^{2b}.
\]
\end{lemma}

\begin{proof}
With $c_k$ as in Lemma~\ref{lem:sobolevcoef}, we have $\Pi_M\mu_{Y^1\vert V}(v) = \sum_{k=1}^M c_k(v)e_k$. Our main idea of the proof is to take the networks of~\cite[Theorem~2.1]{Yang2025} that approximate the coefficients $c_k$, to clip their outputs, and to place these networks in parallel in one network of $\gF_{\mathrm{NN}}(\mathbf{d}, B_\gF)$. In this proof, the constants in $\lesssim$ do not depend on $M$, $\gH_M$ or the widths $d_1, \ldots, d_{L-1}$.

We first extend the coefficients to a cube, on which~\cite[Theorem~2.1]{Yang2025} holds. Fix a cube $\gQ \supseteq \gV$. According to~\cite[Chapter~VI, Theorem~5]{Stein1970}, the condition that $\gV$ is bounded with a Lipschitz boundary (Assumption~\ref{ass:domain}) yields a linear operator that extends every function on $\gV$ to a function on $\R^{d_v}$ and maps $W^{r,2}(\gV)$ continuously into $W^{r,2}(\R^{d_v})$. We denote by $\bar c_k$ the image of $c_k$ under this operator. Thus, $\bar c_k = c_k$ on $\gV$. Moreover, there is a constant $C_\gV$ that only depends on $\gV$ and $r$ such that $\norm{\bar c_k}_{W^{r,2}(\gQ)} \le C_\gV\norm{c_k}_{W^{r,2}(\gV)}$ for every $k$. This bound, Lemma~\ref{lem:sobolevcoef} and Assumption~\ref{ass:smooth_density} then yield:
\begin{equation}\label{eq:coef_ext}
	\sum_{k=1}^M \norm{\bar c_k}^2_{W^{r,2}(\gQ)} \le C_\gV^2 N_rB_k \abs{\gY} \int_\gY \norm{\fp^1(y\vert \cdot)}^2_{W^{r,2}(\gV)}\,\mathrm{d}y < \infty.
\end{equation}
Since $r > d_v/2$ by Assumption~\ref{ass:smooth_density}, the space $W^{r,2}(\gQ)$ embeds continuously into the continuous functions on $\overline{\gQ}$. Thus, there is a constant $C_\gQ$ such that $s_k \coloneqq \sup_\gQ\abs{\bar c_k} \le C_\gQ\norm{\bar c_k}_{W^{r,2}(\gQ)}$ for every $k$.

We now construct $\psi_c$. Let $L_0$ and $\bar w \ge 2$ be integers that only depend on $r$ and $d_v$ such that~\cite[Theorem~2.1]{Yang2025} provides its approximation bound for the networks~\eqref{eq:nn_class} with one output, depth at least $L_0 - 1$ and width at least $\bar w$ in every hidden layer, and denote $w \coloneqq \lfloor D/M\rfloor$. Since $\bar w \le D/M$ and $\bar w$ is an integer, we have $\bar w \le \lfloor D/M \rfloor = w$. In addition, we have from our assumption that $L \ge L_0$. Therefore, according to~\cite[Theorem~2.1]{Yang2025}, for every $k$, there exists a network $\widetilde\psi_k$ of the form~\eqref{eq:nn_class} with one output, depth $L-1$ and width $w$ in every hidden layer such that:
\begin{equation}\label{eq:yang_approx}
	\norm{\widetilde\psi_k - \bar c_k}^2_{L^2(\gQ)} \lesssim (wL)^{-2b}\norm{\bar c_k}^2_{W^{r,2}(\gQ)}.
\end{equation}
Since \eqref{eq:yang_approx} only bounds the $L^2(\gQ)$ error, the values of $\widetilde\psi_k$ need not lie in $[-s_k, s_k]$, while the class $\gF_{\mathrm{NN}}(\mathbf{d}, B_\gF)$ requires a bound on $\sup_{v\in\gV}\norm{\psi_c(v)}$. Thus, we replace the linear output unit in $\widetilde\psi_k$ by the two ReLU units $\sigma_{\mathrm{relu}}(\widetilde\psi_k(v) + s_k)$ and $\sigma_{\mathrm{relu}}(\widetilde\psi_k(v) - s_k)$. Then, we add a linear output unit with the following value:
\[
	\sigma_{\mathrm{relu}}(\widetilde\psi_k(v) + s_k) - \sigma_{\mathrm{relu}}(\widetilde\psi_k(v) - s_k) - s_k = \min\{\max\{\widetilde\psi_k(v), -s_k\}, s_k\}.
\]
The resulting network has $L-2$ hidden layers of width $w$, a hidden layer of width $2$ and a linear output. Since $D \ge \bar w M \ge 2M$, we can place the resulting networks for $k = 1, \ldots, M$ in parallel in one network $\psi_c$ with the widths $\mathbf{d}$.

Finally, we bound $\psi_c$. By construction, $\psi_{c,k}(v) = \min\{\max\{\widetilde\psi_k(v), -s_k\}, s_k\}$ for every $k$. Since $\abs{\psi_{c,k}} \le s_k$, the bound on $s_k$ and~\eqref{eq:coef_ext} yield:
\begin{equation}\label{eq:BF}
	\sup_{v\in\gV}\norm{\psi_c(v)}^2 \le \sum_{k=1}^M s_k^2 \le C_\gQ^2 C_\gV^2 N_rB_k\abs{\gY}\int_\gY \norm{\fp^1(y\vert \cdot)}^2_{W^{r,2}(\gV)}\,\mathrm{d}y \eqqcolon \bar B.
\end{equation}
Since $\gQ$ depends only on $\gV$, the constant $\bar B$ depends only on $\gV$, $\gY$, $r$, $B_k$ and $\fp^1$. Thus, $\psi_c \in \gF_{\mathrm{NN}}(\mathbf{d}, B_\gF)$ whenever $B_\gF^2 \ge \bar B$. Moreover, for every $v \in \gV$, we have $c_k(v) = \bar c_k(v) \in [-s_k, s_k]$. Since the map $t \mapsto \min\{\max\{t, -s_k\}, s_k\}$ is $1$-Lipschitz and fixes every point of $[-s_k, s_k]$, we obtain $\abs{\psi_{c,k}(v) - c_k(v)} \le \abs{\widetilde\psi_k(v) - c_k(v)}$ for $v \in \gV$. Thus, the orthonormality of $\{e_k\}_{k=1}^M$, the bound $\sup_{v\in\gV}p_V(v) < \infty$, \eqref{eq:yang_approx}, \eqref{eq:coef_ext} and $w = \lfloor D/M\rfloor \ge D/(2M)$, respectively, yield the bound in the lemma:
\begin{align*}
	\E\left[\Bigl\lVert \sum_{k=1}^M \psi_{c,k}(V)e_k - \Pi_M\mu_{Y^1\vert V}(V)\Bigr\rVert^2_{\gH_\gY}\right] &= \sum_{k=1}^{M} \E [(\psi_{c,k}(V) - c_k(V))^2] \\
	&\le \sum_{k=1}^{M} \E [(\widetilde\psi_k(V) - c_k(V))^2] \\
	&\lesssim (wL)^{-2b} \sum_{k=1}^M \norm{\bar c_k}^2_{W^{r,2}(\gQ)} \lesssim \left(\frac{M}{DL}\right)^{2b}.
\end{align*}
\end{proof}

\begin{lemma}[Projection Error Bound]\label{lem:projbound} 
    Let $\gY \subset \R^{d_y}$ be a bounded domain with a Lipschitz boundary. Let $\{\widetilde{y}_j\}_{j=1}^M \subset \gY$ be a set of grid points with fill distance $h \coloneqq \sup_{y \in \gY} \min_{j} \|y - \widetilde{y}_j\| \asymp M^{-1/d_y}$. Let $\Pi_M: \gH_\gY \to \gH_\gY$ denote the orthogonal projection from $\gH_\gY$ onto the subspace spanned by $\{k_\gY(\cdot, \widetilde{y}_j)\}_{j=1}^M = \{\phi(\widetilde{y}_j)\}_{j=1}^M$. In the following bounds, the constants in $\lesssim$ do not depend on $M$ or the grid points.
    \begin{enumerate}
        \item \textbf{(Sobolev Kernels)} Assume $k_\gY(y,y')=\varphi(y-y')$ for all $y,y'\in \gY$, where $\widehat{\varphi}(\omega) \lesssim (1 + \|\omega\|^2)^{-\tau}$ for some $\tau > d_y/2$. Suppose that $\fp^1(\cdot\vert V) \in L^2(\gY)$ almost surely. Then, the expected approximation error satisfies:
        \begin{equation}\label{eq:proj_bdd_1}
            \E  \left[ \| \mu_{Y^1\vert V}(V) - \Pi_M \mu_{Y^1\vert V}(V) \|_{\gH_\gY}^2 \right] \lesssim M^{-\frac{2\tau}{d_y}} \E  \left[ \| \fp^1(\cdot\vert V) \|_{L^2(\gY)}^2 \right].
        \end{equation}
        \item \textbf{(Gaussian Kernels)} Assume that $\gY$ satisfies a uniform interior cone condition (Definition~\ref{def:cone_condition}) and that $\gH_\gY$ is the RKHS generated by a Gaussian kernel. Suppose that $\fp^1(\cdot\vert V) \in L^2(\gY)$ almost surely. Then, there exists a constant $B_0 > 0$ such that:
        \begin{equation}\label{eq:proj_bdd_2}
            \E  \left[ \| \mu_{Y^1\vert V}(V) - \Pi_M \mu_{Y^1\vert V}(V) \|_{\gH_\gY}^2 \right] \lesssim \exp\left( - B_0 M^{\frac{1}{2d_y}} \log M \right) \E  \left[ \| \fp^1(\cdot\vert V) \|_{L^2(\gY)}^2 \right].
        \end{equation}
    \end{enumerate}
\end{lemma}

\begin{proof}
    Given any $v \in \gV$, we denote $\mu_v \coloneqq \mu_{Y^1\vert V}(v)$ and $e_v \coloneqq \mu_v - \Pi_M \mu_v$. We first show, for both kernels, that $e_v$ vanishes at the grid points and that $\norm{e_v}^2_{\gH_\gY} \le \norm{e_v}_{L^2(\gY)}\norm{\fp^1(\cdot\vert v)}_{L^2(\gY)}$. 

    Since $e_v$ is orthogonal to $\text{span}\{\phi(\widetilde{y}_j)\}_{j=1}^M$, which contains $\Pi_M\mu_v$, we have:
    \begin{equation*}
        \| e_v \|_{\gH_\gY}^2 = \langle e_v, \mu_v - \Pi_M \mu_v \rangle_{\gH_\gY} = \langle e_v, \mu_v \rangle_{\gH_\gY}\text{.}
    \end{equation*}
    Since $e_v$ is orthogonal to $\phi(\widetilde{y}_j)$, the reproducing property also yields $e_v(\widetilde{y}_j) = \inprod{e_v}{\phi(\widetilde{y}_j)}_{\gH_\gY} = 0$ for all $j = 1, \ldots, M$. Moreover, by~\eqref{eq:ccme_density} and the reproducing property, we obtain:
    \begin{equation*}
        \| e_v \|_{\gH_\gY}^2 = \left\langle e_v, \int_{\gY} \phi(y) \fp^1(y\vert v) \, \mathrm{d}y \right\rangle_{\gH_\gY} = \int_{\gY} e_v(y) \fp^1(y\vert v) \, \mathrm{d}y = \langle e_v, \fp^1(\cdot\vert v) \rangle_{L^2(\gY)}.
    \end{equation*}
    Thus, applying the Cauchy-Schwarz inequality in $L^2(\gY)$, we obtain:
    \begin{equation} \label{eq:duality_main}
        \| e_v \|_{\gH_\gY}^2 \le \| e_v \|_{L^2(\gY)} \| \fp^1(\cdot\vert v) \|_{L^2(\gY)}.
    \end{equation}
We now prove \eqref{eq:proj_bdd_1}. To bound $\| e_v \|_{L^2(\gY)}$, we invoke the sampling inequality from \cite[Theorem~4.1]{Arcangeli2007} for Sobolev spaces of order $\tau > d_y/2$ on bounded connected open sets with a Lipschitz boundary. Since $\gY$ is bounded with a Lipschitz boundary, $\gY$ has finitely many connected components $\gY_1, \ldots, \gY_J$ with Lipschitz boundaries, whose closures are pairwise disjoint. Thus, we can apply the sampling inequality on each $\gY_l$. More precisely, let $h_l$ be the fill distance of the grid points in $\gY_l$. According to~\cite[Theorem~4.1]{Arcangeli2007}, for each $l$, there is a threshold $h_{0,l} > 0$, which depends only on $\gY_l$ and $\tau$, such that, whenever $h_l \le h_{0,l}$, every $f \in W^{\tau,2}(\gY_l)$ that vanishes at the grid points in $\gY_l$ satisfies:
    \begin{equation}\label{eq:bare} \norm{f}_{L^2(\gY_l)} \lesssim h_l^\tau\norm{f}_{W^{\tau,2}(\gY_l)}. \end{equation}
    Since the sampling inequality applies to each connected component with only the grid points in that component, we need to make sure that the fill distance $h$ is smaller than the distances between the connected components. To this end, let $\delta > 0$ be smaller than the distance between the closures of any two different components, and let $h_1 \coloneqq \min\{\delta, h_{0,1}, \ldots, h_{0,J}\}$. Choose a sufficiently large $M$ so that $h \le h_1$. Since $h \le \delta$, every grid point within distance $h$ of a point $y \in \gY_l$ lies in $\gY_l$. Thus, by the definition of the fill distance, $h_l \le h$ for every $l$. Consequently, $h_l \le h_{0,l}$ for every $l$. Thus, since $e_v$ vanishes at the grid points, for such $h$, the sampling inequality \eqref{eq:bare} with $f = e_v$ yields, for every $l$:
    \[\norm{e_v}_{L^2(\gY_l)} \lesssim h_l^\tau\norm{e_v}_{W^{\tau,2}(\gY_l)} \le h^\tau\norm{e_v}_{W^{\tau,2}(\gY_l)}.\]
    Moreover, by Proposition~\ref{prop:kernelinv} and the bound $\widehat\varphi(\omega) \lesssim (1 + \norm{\omega}^2)^{-\tau}$, every $F \in \gH_{\R^{d_y}}$ satisfies $\norm{F}^2_{W^{\tau,2}(\R^{d_y})} \lesssim \int_{\R^{d_y}} \abs{\widehat{F}(\omega)}^2/\widehat\varphi(\omega)\,\mathrm{d}\omega = \norm{F}^2_{\gH_{\R^{d_y}}}$. Taking the infimum over the extensions $F$ of $e_v$ (Proposition~\ref{prop:kernelinv}) yields $\norm{e_v}_{W^{\tau,2}(\gY)} \lesssim \norm{e_v}_{\gH_\gY}$. Summing the squares of the inequalities $\norm{e_v}_{L^2(\gY_l)} \lesssim h^\tau\norm{e_v}_{W^{\tau,2}(\gY_l)}$ over $l = 1, \ldots, J$, and using $\norm{e_v}_{W^{\tau,2}(\gY_l)} \le \norm{e_v}_{W^{\tau,2}(\gY)}$ and this bound, we obtain:
    \begin{equation*}
        \| e_v \|_{L^2(\gY)} \mathrel{=} \Big(\sum_{l=1}^J \| e_v \|^2_{L^2(\gY_l)}\Big)^{1/2} \mathrel{\lesssim} h^\tau \Big(\sum_{l=1}^J \| e_v \|^2_{W^{\tau,2}(\gY_l)}\Big)^{1/2} \lesssim h^\tau \| e_v \|_{W^{\tau,2}(\gY)} \mathrel{\lesssim} h^\tau \| e_v \|_{\gH_\gY}.
    \end{equation*}
    Plugging this bound back into the right-hand side of \eqref{eq:duality_main} yields:
    \begin{equation*}
        \| e_v \|_{\gH_\gY}^2 \lesssim h^\tau \| e_v \|_{\gH_\gY} \| \fp^1(\cdot\vert v) \|_{L^2(\gY)}.
    \end{equation*}
    Dividing both sides by $\| e_v \|_{\gH_\gY}$, we get:
    \begin{equation*}
        \| e_v \|_{\gH_\gY} \lesssim h^\tau \| \fp^1(\cdot\vert v) \|_{L^2(\gY)}.
    \end{equation*}
    Squaring both sides and taking the expectation with respect to $V$ yields:
    \begin{align*}
       \E  \left[ \| e_V \|_{\gH_\gY}^2 \right] \lesssim h^{2\tau}\, \E\left[ \norm{\fp^1(\cdot\vert V)}_{L^2(\gY)}^2 \right].
    \end{align*}
    Thus, we obtain \eqref{eq:proj_bdd_1} by substituting the fill distance $h \asymp M^{-1/d_y}$ for every $M$ with $h \le h_1$.

    We now prove \eqref{eq:proj_bdd_2}. In this case, $\gH_\gY$ is an RKHS with a Gaussian kernel. To bound the $L^2$-norm of $e_v$, we invoke the sampling inequality of~\cite[Theorem~3.5]{rieger2010} for bounded domains with a Lipschitz boundary that satisfy an interior cone condition, which states that there exist constants $B_1, h_0 > 0$, where $B_1$ depends only on $\gY$, such that, whenever the fill distance satisfies $h \le h_0$, every $u \in \gH_\gY$ satisfies:
    \begin{equation} \label{eq:rieger_thm35}
        \| u \|_{L^2(\gY)} \leq \exp\left( \frac{B_1 \log h}{\sqrt{h}} \right) \| u \|_{\gH_\gY} + \sup_{1 \le j \le M} \lvert u(\widetilde{y}_j) \rvert.
    \end{equation}
    With $u = e_v$ in \eqref{eq:rieger_thm35} and $e_v(\widetilde{y}_j) = 0$ for all $j$, we obtain:
    \begin{equation} \label{eq:L2_exp_bound}
        \| e_v \|_{L^2(\gY)} \leq \exp\left( \frac{B_1 \log h}{\sqrt{h}} \right) \| e_v \|_{\gH_\gY}.
    \end{equation}
    We then substitute \eqref{eq:L2_exp_bound} back into the right-hand side of \eqref{eq:duality_main}:
    \[ \| e_v \|_{\gH_\gY}^2 \leq \left[ \exp\left( \frac{B_1 \log h}{\sqrt{h}} \right) \| e_v \|_{\gH_\gY} \right] \| \fp^1(\cdot\vert v) \|_{L^2(\gY)}. \]
    Dividing both sides by $\| e_v \|_{\gH_\gY}$ yields:
    \[ \| e_v \|_{\gH_\gY} \leq \exp\left( \frac{B_1 \log h}{\sqrt{h}} \right) \| \fp^1(\cdot\vert v) \|_{L^2(\gY)}. \]
    Squaring both sides and taking the expectation with respect to $V$, we obtain:
    \[
        \E\big[\norm{e_V}^2_{\gH_\gY}\big] \le \exp\Big(\frac{2B_1\log h}{\sqrt{h}}\Big)\E\big[\norm{\fp^1(\cdot\vert V)}^2_{L^2(\gY)}\big].
    \]
    Let $c_2 > 0$ be a constant with $h \le c_2M^{-1/d_y}$. For every $M \ge c_2^{2d_y}$, we have $c_2M^{-1/d_y} \le 1$ and $\log c_2 \le \log M/(2d_y)$. Since $h \mapsto \log h/\sqrt{h}$ is increasing on $(0, 1]$, these two inequalities yield:
    \[
        \frac{2B_1\log h}{\sqrt{h}} \le \frac{2B_1\log(c_2M^{-1/d_y})}{\sqrt{c_2M^{-1/d_y}}} = \frac{2B_1}{\sqrt{c_2}}\Big(\log c_2 - \frac{\log M}{d_y}\Big)M^{\frac{1}{2d_y}} \le -\frac{B_1}{d_y\sqrt{c_2}}M^{\frac{1}{2d_y}}\log M.
    \]
    Thus, \eqref{eq:proj_bdd_2} holds with $B_0 \coloneqq B_1/(d_y\sqrt{c_2})$.
\end{proof}

\end{document}